\documentclass[pdftex,dvipsnames,twoside,11pt]{article}

\usepackage{jair, theapa, rawfonts, natbib}

\usepackage[utf8]{inputenc} 
\usepackage[T1]{fontenc}    
\usepackage{url}            
\usepackage{booktabs}       
\usepackage{amsfonts}       
\usepackage{nicefrac}       
\usepackage{microtype}      

\usepackage{nicefrac}       
\usepackage{amsmath, amsthm, amssymb, amsfonts}
\usepackage{bm}
\usepackage{bbm} 
\usepackage[table]{xcolor}
\definecolor{darkblue}{HTML}{091c70}
\definecolor{highlightcolor}{HTML}{ffdada}
\usepackage{hyperref}
\hypersetup{
    colorlinks=True,
    linkcolor=darkblue,
    filecolor=darkblue,      
    urlcolor=darkblue,
    citecolor=darkblue,
}
\usepackage{subcaption} 
\usepackage{mathtools}
\usepackage{relsize}
\usepackage{multirow}

\usepackage{algorithm}
\usepackage[noend]{algpseudocode}
\usepackage{wrapfig}
\usepackage{graphicx}

\newcommand*{\todo}[2][]{\textcolor{red}{[\textbf{\ifthenelse{\equal{#1}{}}{TODO}{TODO(#1)}}: #2]}}

\newtheorem{theorem}{Theorem}
\newtheorem{lemma}{Lemma}
\newtheorem{corollary}{Corollary}[theorem]
\newtheorem{definition}{Definition}
\newtheorem{condition}{Condition}

\DeclareMathOperator*{\argmax}{arg\,max}

\newcommand{\smalleq}{{\scriptstyle =}}
\newcommand{\smallplus}{{\scriptstyle +}}
\newcommand{\smallminus}{{\scriptstyle -}}

\newcommand{\smallra}{{\scriptstyle \rightarrow}}
\newcommand{\tinyeq}{{\scriptscriptstyle =}}
\newcommand{\smallin}{{\scriptstyle \in}}

\newcommand{\nm}{\bm{Q}_{\tilde{y} \vert y^*}}
\newcommand{\inv}{\bm{Q}_{y^* \vert \tilde{y}}}
\newcommand{\joint}{\bm{Q}_{\tilde{y}, y^*}}
\newcommand{\prior}{\bm{Q}_{y^*}}

\newcommand{\estnm}{\hat{\bm{Q}}_{\tilde{y} \vert y^*}}
\newcommand{\estinv}{\hat{\bm{Q}}_{y^* \vert \tilde{y}}}
\newcommand{\estjoint}{\hat{\bm{Q}}_{\tilde{y}, y^*}}
\newcommand{\estprior}{\hat{\bm{Q}}_{y^*}}

\newcommand{\noisypriorlong}{\bm{Q}_{\tilde{y} = i}}

\newcommand{\estnmlong}{\hat{\bm{Q}}_{\tilde{y} = i \vert y^* = j}}
\newcommand{\estinvlong}{\hat{\bm{Q}}_{y^* = j \vert \tilde{y} = i}}
\newcommand{\estjointlong}{\hat{\bm{Q}}_{\tilde{y} = i, y^* = j}}
\newcommand{\estpriorlong}{\hat{\bm{Q}}_{y^* = j}}

\newcommand{\model}{\bm{\theta}} 
\newcommand{\pyx}{\hat{p}(\tilde{y} ; \bm{x}, \bm{\theta})} 
\newcommand{\probmatrix}{\bm{\hat{P}}_{k,i}}
\newcommand{\predprobshorti}{\hat{p}_{\bm{x}, \tilde{y} \smalleq i}}
\newcommand{\predprobshortj}{\hat{p}_{\bm{x}, \tilde{y} \smalleq j}}

\newcommand{\perfprobshorti}{p^*_{\bm{x}, \tilde{y} \smalleq i}}
\newcommand{\perfprobshortj}{p^*_{\bm{x}, \tilde{y} \smalleq j}}

\newcommand{\errorxj}{\epsilon_{\bm{x}, \tilde{y} \smalleq j}}
\newcommand{\cj}{\bm{C}_{\tilde{y}, y^*}}

\newcommand{\estpartition}{\hat{\bm{X}}_{ \tilde{y} \smalleq i,y^* \smalleq j}}
\newcommand{\partition}{\bm{X}_{ \tilde{y} \smalleq i,y^* \smalleq j}}

\DeclareMathOperator*{\EXlimits}{\mathbb{E}}
\DeclareMathOperator{\EX}{\mathbb{E}}

\newcommand{\beginsupplement}{ 
\setcounter{section}{0}
\renewcommand{\thesection}{S\arabic{section}} %
\renewcommand{\thesubsection}{\thesection.\arabic{subsection}}
\setcounter{table}{0}
\renewcommand{\thetable}{S\arabic{table}} %
\setcounter{figure}{0}
\renewcommand{\thefigure}{S\arabic{figure}} %
}

\newcommand{\possessivecite}[1]{\citeauthor{#1}'s (\citeyear{#1})}

\jairheading{70}{2021}{1373-1411}{09/2020}{04/2021}
\ShortHeadings{Confident Learning: Estimating Uncertainty in Dataset Labels}
{Northcutt, Jiang, \& Chuang}
\firstpageno{1373}

\begin{document}

\title{Confident Learning:\\Estimating Uncertainty in Dataset Labels}

\author{\name Curtis G. Northcutt \email curtis@cleanlab.ai  \\
      \addr Cleanlab \email cgn@mit.edu \\
      \addr Massachusetts Institute of Technology, \\
      Department of EECS, Cambridge, MA, USA
      \AND
      \name Lu Jiang \email lujiang@google.com \\
      \addr Google Research, Mountain View, CA, USA
      \AND
      \name Isaac L. Chuang \email ichuang@mit.edu \\
      \addr Massachusetts Institute of Technology, \\
      Department of EECS, Department of Physics, Cambridge, MA, USA}



\maketitle

\begin{abstract}
    
    


Learning exists in the context of data, yet notions of \emph{confidence} typically focus on model predictions, not label quality. Confident learning (CL) is a data-centric approach which focuses instead on label quality by characterizing and identifying label errors in datasets, based on the principles of pruning noisy data, counting with probabilistic thresholds to estimate noise, and ranking examples to train with confidence. Whereas numerous studies have developed these principles independently, here, we combine them, building on the assumption of a class-conditional noise process to directly estimate the joint distribution between noisy (given) labels and uncorrupted (unknown) labels. This results in a generalized CL which is provably consistent and experimentally performant. We present sufficient conditions where CL exactly finds label errors, and show CL performance exceeding seven recent competitive approaches for learning with noisy labels on the CIFAR dataset. Uniquely, the CL framework is \emph{not} coupled to a specific data modality or model (e.g., we use CL to find several label errors in the presumed error-free MNIST dataset and improve sentiment classification on text data in Amazon Reviews). We also employ CL on ImageNet to quantify ontological class overlap (e.g., estimating 645 \emph{missile} images are mislabeled as their parent class \emph{projectile}), and moderately increase model accuracy (e.g., for ResNet) by cleaning data prior to training. These results are replicable using the open-source \href{https://github.com/cleanlab/cleanlab}{\texttt{cleanlab}} framework.

\end{abstract}

\section{Introduction}


Advances in learning with noisy labels and weak supervision usually introduce a new model or loss function. Often this model-centric approach band-aids the real question: which data is mislabeled? Yet, large datasets with noisy labels have become increasingly common. Examples span prominent benchmark datasets like ImageNet \citep{ILSVRC15_imagenet} and MS-COCO \citep{MS-COCO-IMAGE-DATASET} to human-centric datasets like electronic health records \citep{yoni_sontag_anchors_2016} and educational data \citep{NORTHCUTT_cameo_2016}. The presence of noisy labels in these datasets introduces two problems. How can we identify examples with label errors and how can we learn well despite noisy labels, irrespective of the data modality or model employed? Here, we follow a data-centric approach to theoretically and experimentally investigate the premise that the key to learning with noisy labels lies in accurately and directly characterizing the uncertainty of label noise in the data.

A large body of work, which may be termed ``confident learning,'' has arisen to address the uncertainty in dataset labels, from which two aspects stand out.  First, \possessivecite{angluin1988learning} classification noise process (CNP) provides a starting assumption that label noise is class-conditional, depending only on the latent true class, not the data. While there are exceptions, this assumption is commonly used \citep{DBLP:conf/iclr/GoldbergerB17_smodel,Sukhbaatar_fergus_iclr_2015} because it is reasonable for many datasets. For example, in ImageNet, a \emph{leopard} is more likely to be mislabeled \emph{jaguar} than \emph{bathtub}.  Second, direct estimation of the joint distribution between noisy (given) labels and true (unknown) labels (see Fig. \ref{fig_flow_examples}) can be pursued effectively based on three principled approaches used in many related studies: (a) {\bf Prune}, to search for label errors, e.g. following the example of \cite{chen2019confusion, patrini2017making, rooyen_menon_unhinged_nips15}, using \emph{soft-pruning} via loss-reweighting, to avoid the convergence pitfalls of iterative re-labeling -- (b) {\bf Count}, to train on clean data, avoiding error-propagation in learned model weights from reweighting the loss \citep{natarajan2017cost} with imperfect predicted probabilities, generalizing seminal work \cite{forman2005counting, forman2008quantifying, icml_lipton_label_shift_confusion_matrix} -- and (c) {\bf Rank} which examples to use during training, to allow learning with unnormalized probabilities or decision boundary distances, building on well-known robustness findings \citep{page1997pagerank} and ideas of curriculum learning \citep{jiang2018mentornet}.

To our knowledge, no prior work has thoroughly analyzed the direct estimation of the joint distribution between noisy and uncorrupted labels. Here, we assemble these principled approaches to generalize confident learning (CL) for this purpose. Estimating the joint distribution is challenging as it requires disambiguation of epistemic uncertainty (model predicted probabilities) from aleatoric uncertainty (noisy labels) \citep{aleatoric_epistemic_uncertainty_2013}, but useful because its marginals yield important statistics used in the literature, including latent noise transition rates \citep{Sukhbaatar_fergus_iclr_2015, noisy_boostrapping_google_reed_iclr_2015}, latent prior of uncorrupted labels \citep{lawrence_bayesian_prior_py_noisy_2001, graepel_kernel_gibbs_sampler_prior_py_bayesian_nips_2001}, and inverse noise rates \citep{scott19_jmlr_mutual_decontamination}. While noise rates are useful for loss-reweighting \citep{NIPS2013_5073}, only the joint can directly estimate the number of label errors for each pair of true and noisy classes. Removal of these errors prior to training is an effective approach for learning with noisy labels \citep{chen2019confusion}. The joint is also useful to discover ontological issues in datasets for dataset curation, e.g. ImageNet includes two classes for the same \emph{maillot} class (c.f. Table \ref{table_imagenet_characterization_top10} in Sec. \ref{sec:experiments}). 


The generalized CL assembled in this paper upon the principles of pruning, counting, and ranking, is a model-agnostic family of theories and algorithms for characterizing, finding, and learning with label errors. It uses predicted probabilities and noisy labels to count examples in the unnormalized \emph{confident joint}, estimate the joint distribution, and prune noisy data, producing clean data as output.

This paper makes two key contributions to prior work on finding, understanding, and learning with noisy labels. First, a proof is presented giving realistic sufficient conditions under which CL exactly finds label errors and exactly estimates the joint distribution of noisy and true labels. Second, experimental data are shared, showing that this CL algorithm is empirically performant on three tasks (a) label noise estimation, (b) label error finding, and (c) learning with noisy labels, increasing ResNet accuracy on a cleaned-ImageNet and outperforming seven recent highly competitive methods for learning with noisy labels on the CIFAR dataset. The results presented are reproducible with the implementation of CL algorithms, open-sourced as the \textbf{\texttt{\href{https://github.com/cleanlab/cleanlab/?utm_source=arxiv&utm_medium=publication&utm_campaign=cleanlab}{cleanlab}}}\footnote{To foster future research in data cleaning and learning with noisy labels and to improve accessibility for newcomers, \textbf{\texttt{\href{https://github.com/cleanlab/cleanlab/?utm_source=arxiv&utm_medium=publication&utm_campaign=cleanlab}{cleanlab}}} is open-source and well-documented: \href{https://github.com/cleanlab/cleanlab/?utm_source=arxiv&utm_medium=publication&utm_campaign=cleanlab}{https://github.com/cleanlab/cleanlab/}} Python package.  

These contributions are presented beginning with the formal problem specification and notation (Section~\ref{sec:setup}), then defining the algorithmic methods employed for CL (Section~\ref{sec:methods}) and theoretically bounding expected behavior under ideal and noisy conditions (Section~\ref{sec:theory}). Experimental benchmarks on the CIFAR, ImageNet, WebVision, and MNIST datasets, cross-comparing CL performance with that from a wide range of highly competitive approaches, including
\emph{INCV} \citep{chen2019confusion}, \emph{Mixup} \citep{icml2018mixup}, \emph{MentorNet} \citep{jiang2018mentornet}, and \emph{Co-Teaching} \citep{han2018coteaching}, are then presented in Section~\ref{sec:experiments}. Related work (Section~\ref{sec:related}) and concluding observations (Section~\ref{sec:conclusion}) wrap up the presentation. Extended proofs of the main theorems, algorithm details, and comprehensive performance comparison data are presented in the appendices.

\section{CL Framework and Problem Set-up}
\label{sec:setup}

In the context of multiclass data with possibly noisy labels, let $[m]$ denote $\{1, 2, ..., m\}$, the set of $m$ unique class labels, and $\bm{X} \coloneqq (\bm{x}, \tilde{y})^n \in {(\mathbb{R}^d, [m])}^n$ denote the dataset of $n$ examples $\bm{x} \in \mathbb{R}^d$ with associated observed noisy labels $\tilde{y} \in [m]$. $\bm{x}$ and $\tilde{y}$ are coupled in $\bm{X}$ to signify that \emph{cleaning} removes data and label. While a number of relevant works address the setting where annotator labels are available \citep{chi2021data, bouguelia2018agreeing, tanno2019learning, Tanno_2019_CVPR, khetan2018learning}, this paper addresses the general setting where no annotation information is available except the observed noisy labels.

\paragraph{Assumptions} We assume there exists, for every example, a latent, true label $y^*$. Prior to observing $\tilde{y}$, a class-conditional classification noise process \citep{angluin1988learning} maps $y^* {\scriptstyle \rightarrow} \; \tilde{y}$ such that every label in class $j \in [m]$ may be independently mislabeled as class $i \in [m]$ with probability $p(\tilde{y} \smalleq i \vert y^* \smalleq j)$. This assumption is reasonable and has been used in prior work~\citep{DBLP:conf/iclr/GoldbergerB17_smodel,Sukhbaatar_fergus_iclr_2015}.

\paragraph{Notation} Notation is summarized in Table \ref{t:notation}. The discrete random variable $\tilde{y} $ takes an observed, noisy label (potentially flipped to an incorrect class), and $y^* $ takes a latent, uncorrupted label. The subset of examples in $\bm{X}$ with noisy class label $i$ is denoted $\bm{X}_{\tilde{y}\smalleq i}$, \emph{i.e.} $\bm{X}_{\tilde{y}\smalleq \text{cow}}$ is read, ``examples with class label \emph{cow}.'' The notation $p(\tilde{y} ; \bm{x})$, as opposed to $p(\tilde{y} \vert \bm{x})$, expresses our assumption that input $\bm{x}$ is observed and error-free. We denote the discrete joint probability of the noisy and latent labels as $p(\tilde{y},y^*)$, where conditionals $p(\tilde{y} \vert y^*)$ and $p(y^* \vert \tilde{y})$ denote probabilities of label flipping. We use $\hat{p}$ for predicted probabilities. In matrix notation, the $n \times m$ matrix of out-of-sample predicted probabilities is $\bm{\hat{P}}_{k,i} \coloneqq \hat{p}(\tilde{y} = i; \bm{x}_k, \bm{\theta})$, the prior of the latent labels is $\prior \coloneqq p(y^* \smalleq i)$; the $m \times m$ joint distribution matrix is $\joint \coloneqq p(\tilde{y} \smalleq i,y^* \smalleq j)$; the $m \times m$ noise transition matrix (noisy channel) of flipping rates is $\nm \coloneqq p(\tilde{y} \smalleq i \vert y^* \smalleq j)$; and the $m \times m$ mixing matrix is $\inv \coloneqq p(y^* \smalleq i \vert \tilde{y} \smalleq j)$. At times, we abbreviate $\hat{p}(\tilde{y} = i ; \bm{x}, \bm{\theta})$ as $\predprobshorti$, where $\model$ denotes the model parameters. CL assumes no specific loss function associated with $\model$: the CL framework is model-agnostic.

\begin{table*}[hbt!] 
\caption{Notation used in confident learning.}
\label{t:notation}
\begin{center}
\resizebox{\textwidth}{!}{
\begin{tabular}{rl} 
\toprule

\textbf{Notation} & \textbf{Definition} \\
\midrule
$m$ & The number of unique class labels \\
$[m]$ & The set of $m$ unique class labels \\
$\tilde{y}$  & Discrete random variable $\tilde{y} \in [m]$ takes an observed, noisy label \\
$y^*$  & Discrete random variable  $y^* \in [m]$  takes the unknown, true, uncorrupted label \\
$\bm{X}$ & The dataset $(\bm{x}, \tilde{y})^n \in {(\mathbb{R}^d, [m])}^n$ of $n$ examples $\bm{x} \in \mathbb{R}^d$ with noisy labels \\
$\bm{x}_k$ & The $k^{th}$ training data example \\
$\tilde{y}_k$ & The observed, noisy label corresponding to $\bm{x}_k$ \\
$y_k^*$ & The unknown, true label corresponding to $\bm{x}_k$ \\
$n$ & The cardinality of $\bm{X} \coloneqq (\bm{x}, \tilde{y})^n$, i.e. the number of examples in the dataset \\
$\bm{\theta}$ & Model parameters \\
$\bm{X}_{\tilde{y}\smalleq i}$ & Subset of examples in $\bm{X}$ with noisy label $i$, \emph{i.e.} $\bm{X}_{\tilde{y}\smalleq \text{cat}}$ is ``examples labeled cat'' \\
$\bm{X}_{\tilde{y}\smalleq i, y^*\smalleq j}$ & Subset of examples in $\bm{X}$ with noisy label $i$ and true label $j$ \\
$\hat{\bm{X}}_{\tilde{y}\smalleq i, y^*\smalleq j}$ & Estimate of subset of examples in $\bm{X}$ with noisy label $i$ and true label $j$ \\
$p(\tilde{y} \smalleq i, y^* \smalleq j)$ & Discrete joint probability of noisy label $i$ and true label $j$. \\
$p(\tilde{y} \smalleq i \vert y^* \smalleq j)$ & Discrete conditional probability of true label flipping, called the noise rate \\
$p( y^* \smalleq j \vert \tilde{y} \smalleq i )$ & Discrete conditional probability of noisy label flipping, called the inverse noise rate \\
$\hat{p}(\cdot)$ & Estimated or predicted probability (may replace $p(\cdot)$ in any context) \\
$\prior$ & The prior of the latent labels \\
$\estprior$ & Estimate of the prior of the latent labels \\
$\joint$ & The $m \times m$ joint distribution matrix for $p(\tilde{y},y^*)$ \\
$\estjoint$ & Estimate of the $m \times m$ joint distribution matrix for $p(\tilde{y},y^*)$ \\
$\nm$ & The $m \times m$ noise transition matrix (noisy channel) of flipping rates for $p(\tilde{y} \vert y^*)$ \\
$\estnm$ & Estimate of the $m \times m$ noise transition matrix of flipping rates for $p(\tilde{y} \vert y^*)$ \\
$\inv$ & The inverse noise matrix for $p(y^* \vert \tilde{y})$ \\
$\estinv$ & Estimate of the inverse noise matrix for $p(y^* \vert \tilde{y})$ \\
$\hat{p}(\tilde{y} = i ; \bm{x}, \bm{\theta})$ & Predicted probability of label $\tilde{y} = i$ for example $\bm{x}$ and model parameters $\bm{\theta}$ \\
$\predprobshorti$ & Shorthand abbreviation for predicted probability $\hat{p}(\tilde{y} = i ; \bm{x}, \bm{\theta})$ \\
$\hat{p}(\tilde{y}\smalleq i ; \bm{x} \smallin \bm{X}_{\tilde{y}=i}, \bm{\theta} )$ & The \emph{self-confidence} of example $\bm{x}$ belonging to its given label $\tilde{y}\smalleq i$ \\
$\bm{\hat{P}}_{k,i}$ & $n \times m$ matrix of out-of-sample predicted probabilities $\hat{p}(\tilde{y} = i; \bm{x}_k, \bm{\theta})$ \\
$\cj$ & The \emph{confident joint} ${\cj \in \mathbb{N}_{\geq 0}}^{m \times m}$, an unnormalized estimate of $\joint$ \\
$\bm{C}_{\text{confusion}}$ & Confusion matrix of given labels $\tilde{y}_k$ and predictions $\argmax_{i \in [m]} \hat{p}(\tilde{y}\smalleq i;\bm{x}_k, \bm{\theta}) $ \\
$t_j$ & The expected (average) self-confidence for class $j$ used as a threshold in $\cj$ \\
$p^*(\tilde{y} \smalleq i \vert y^* \smalleq y_k^*)$ & \emph{Ideal} probability for some example $\bm{x}_k$, equivalent to noise rate $p^*(\tilde{y} \smalleq i \vert y^* \smalleq j)$ \\
$\perfprobshorti$ & Shorthand abbreviation for ideal probability $p^*(\tilde{y} \smalleq i \vert y^* \smalleq y_k^*)$ \\

\bottomrule
\end{tabular}
}
\end{center}
\end{table*}

\paragraph{Goal} Our assumption of a class-conditional noise process implies the label noise transitions are data-independent, i.e., 
$p(\tilde{y} \vert y^* ; \bm{x}) = p(\tilde{y} \vert y^*)$. To characterize class-conditional label uncertainty, one must estimate $p(\tilde{y} \vert y^*)$ and $p(y^*)$, the latent prior distribution of uncorrupted labels. Unlike prior works which estimate $p(\tilde{y} \vert y^*)$ and $p(y^*)$ independently, we estimate both jointly by directly estimating the joint distribution of label noise, $p(\tilde{y} , y^*)$. \textbf{Our goal} is to estimate every $p(\tilde{y} , y^*)$ as a matrix $\joint$ and use $\joint$ to find all mislabeled examples $\bm{x}$ in dataset $\bm{X}$ where $y^* \neq \tilde{y}$. This is hard because it requires disambiguation of model error (epistemic uncertainty) from the intrinsic label noise (aleatoric uncertainty), while simultaneously estimating the joint distribution of label noise ($\joint$) without prior knowledge of the latent noise transition matrix ($\nm$), the latent prior distribution of true labels ($\prior$), or any latent, true labels ($y*$).

\begin{definition} [Sparsity]
    A statistic to quantify the characteristic shape of the label noise defined by fraction of zeros in the off-diagonals of $\joint$. \emph{High sparsity quantifies non-uniformity of label noise, common to real-world datasets. For example, in ImageNet, \emph{missile} may have high probability of being mislabeled as \emph{projectile}, but near-zero probability of being mislabeled as most other classes like \emph{wool} or \emph{wine}. Zero sparsity implies every noise rate in $\joint$ is non-zero. A sparsity of 1 implies no label noise because the off-diagonals of $\joint$, which encapsulate the class-conditional noise rates, must all be zero if sparsity = 1.}
\end{definition}

\begin{definition}[Self-Confidence]
    The predicted probability for some model $\bm{\theta}$ that an example $\bm{x}$ belongs to its given label $\tilde{y} $, expressed as $\hat{p}(\tilde{y}\smalleq i ; \bm{x} \smallin \bm{X}_{\tilde{y}=i}, \bm{\theta} )$. \emph{Low self-confidence is a heuristic-likelihood of being a label error.}
\end{definition}

\section{CL Methods}
\label{sec:methods}

Confident learning (CL) estimates the joint distribution between the (noisy) observed labels and the (true) latent labels. CL requires two inputs: (1) the out-of-sample predicted probabilities $\bm{\hat{P}}_{k,i}$ and (2) the vector of noisy labels $\tilde{y}_k$. The two inputs are linked via index $k$ for all $\bm{x}_k \in \bm{X}$. None of the true labels $y^*$ are available, except when $\tilde{y} = y^*$, and we do not know when that is the case.

The out-of-sample predicted probabilities $\bm{\hat{P}}_{k,i}$ used as input to CL are computed beforehand (e.g. cross-validation) using a model $\model$: so, how does $\model$ fit into the CL framework? Prior works typically learn with noisy labels by directly modifying the model or training loss function, restricting the class of models. Instead, CL decouples the model and data cleaning procedure by working with model outputs $\bm{\hat{P}}_{k,i}$, so that any model that produces a mapping $\model: \bm{x} \rightarrow \hat{p}(\tilde{y}\smalleq i;\bm{x}_k, \bm{\theta})$ can be used (e.g. neural nets with a softmax output, naive Bayes, logistic regression, etc.). However, $\model$ affects the predicted probabilities $\hat{p}(\tilde{y}\smalleq i;\bm{x}_k, \bm{\theta})$ which in turn affect the performance of CL. Hence, in Section \ref{sec:theory}, we examine sufficient conditions where CL finds label errors exactly, even when $\hat{p}(\tilde{y}\smalleq i;\bm{x}_k, \bm{\theta})$ is erroneous. Any model $\model$ may be used for final training on clean data provided by CL.


CL identifies noisy labels in existing datasets to improve learning with noisy labels. The main procedure (see Fig. \ref{fig_flow_examples}) comprises three steps: (1) estimate $\estjoint$ to characterize class-conditional label noise (Sec. \ref{sec:label_noise_characterization}), (2) filter out noisy examples (Sec. \ref{sec:rank_prune_find_errors}), and (3) train with errors removed, reweighting the examples by class weights $\frac{\estprior[i]}{\estjoint[i][i]}$ for each class $i \in [m]$. In this section, we define these three steps and discuss their expected outcomes.

\begin{figure*}[t]
\centerline{\includegraphics[width=.7\linewidth]{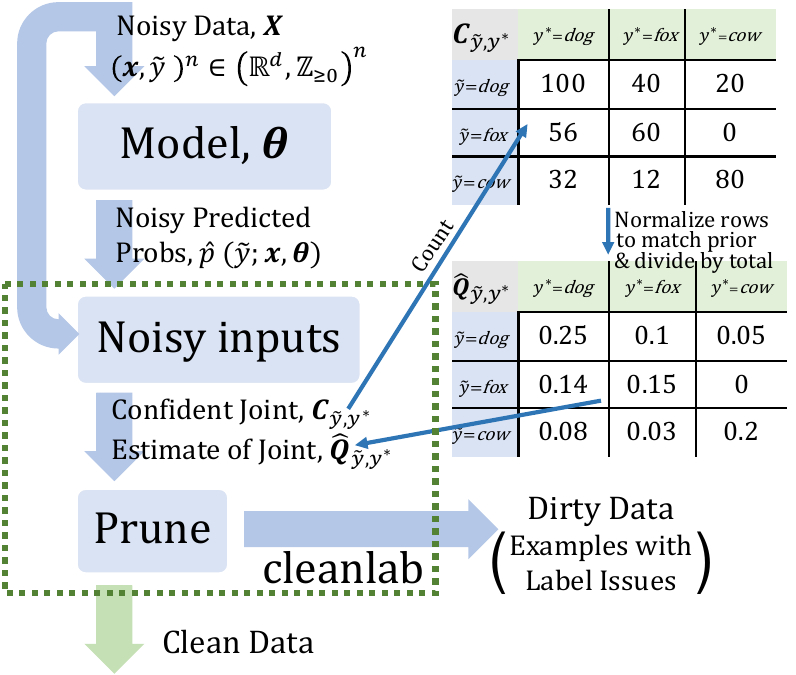}}
\caption{An example of the confident learning (CL) process. CL uses the confident joint, $\cj$, and $\estjoint$, an estimate of $\joint$, the joint distribution of noisy observed labels $\tilde{y}$ and unknown true labels $y^*$, to find examples with label errors and produce clean data for training.}
    \label{fig_flow_examples}
\end{figure*}

\subsection{Count: Characterize and Find Label Errors using the Confident Joint}
\label{sec:label_noise_characterization}

To estimate the joint distribution of noisy labels $\tilde{y}$ and true labels, $\joint$, we count examples that are likely to belong to another class and calibrate those counts so that they sum to the given count of noisy labels in each class, $|\bm{X}_{\tilde{y}=i}|$. Counts are captured in the \emph{confident joint} ${\cj \in \mathbb{Z}_{\geq 0}}^{m \times m}$, a statistical data structure in CL to directly find label errors. Diagonal entries of $\cj$ count correct labels and non-diagonals capture asymmetric label error counts. As an example, $C_{\tilde{y}\smalleq 3,y^* \smalleq 1} \smalleq 10$ is read, ``Ten examples are labeled \emph{3} but should be labeled \emph{1}.''

In this section, we first introduce the \emph{confident joint} $\cj$ to partition and count label errors. Second, we show how $\cj$ is used to estimate $\joint$ and characterize label noise in a dataset $\bm{X}$. Finally, we provide a related baseline $\bm{C}_{\text{confusion}}$ and consider its assumptions and short-comings (e.g. class-imbalance) in comparison with $\cj$ and CL. CL overcomes these shortcomings using thresholding and collision handling to enable robustness to class imbalance and heterogeneity in predicted probability distributions across classes.

\paragraph{The confident joint $\cj$}

$\cj$ estimates $\partition$, the set of examples with noisy label $i$ that actually have true label $j$, by partitioning $\bm{X}$ into estimate bins $\estpartition$. When $\estpartition = \partition$, then $\cj$ exactly finds label errors (proof in Sec. \ref{sec:theory}). $\estpartition$ (note the hat above $\hat{\bm{X}}$ to indicate $\estpartition$ is an estimate of $\partition$) is the set of examples $\bm{x}$ labeled $\tilde{y}\smalleq i$ with \emph{large enough} $\hat{p} (\tilde{y} = j ;\bm{x}, \bm{\theta})$ to likely belong to class $y^*\smalleq j$, determined by a per-class threshold, $t_j$. Formally, the definition of the \emph{confident joint} is

\begin{equation} \label{eqn_paper_confident_joint}
\begin{split}
    \cj[i][j] \coloneqq & \lvert \estpartition  \rvert  \quad \text{where} \\
    \estpartition \coloneqq & \left\{ \bm{x} \in \bm{X}_{\tilde{y} = i} : \; \hat{p} (\tilde{y} = j ;\bm{x}, \bm{\theta})  \ge t_j, \,\, j = \hspace{-18pt} \argmax_{l \in [m] : \hat{p} (\tilde{y} = l ;\bm{x}, \bm{\theta})  \ge t_l}  \hat{p} (\tilde{y} = l ;\bm{x}, \bm{\theta})  \right\}
\end{split}
\end{equation}

and the threshold $t_j$ is the expected (average) self-confidence for each class
\vspace{-2pt}
\begin{equation} \label{eqn_paper_threshold}
     t_j = \frac{1}{|\bm{X}_{\tilde{y}=j}|} \sum_{\bm{x} \in \bm{X}_{\tilde{y}=j}} \hat{p}(\tilde{y}=j; \bm{x}, \bm{\theta})
\end{equation}

Unlike prior art, which estimates label errors under the assumption that the true labels are $\tilde{y}^*_k = \argmax_{i \in [m]} \hat{p}(\tilde{y}\smalleq i;\bm{x}_k, \bm{\theta})$~\citep{chen2019confusion}, the thresholds in this formulation improve CL uncertainty quantification robustness to (1) heterogeneous class probability distributions and (2) class-imbalance. For example, if examples labeled $i$ tend to have higher probabilities because the model is over-confident about class $i$, then $t_i$ will be proportionally larger; if some other class $j$ tends toward low probabilities, $t_j$ will be smaller. These thresholds allow us to guess $y^*$ in spite of class-imbalance, unlike prior art which may guess over-confident classes for $y^*$ because $\argmax$ is used \citep{kilian_weinberger_calibration_Guo2017_icml}. We examine ``how good'' the probabilities produced by model $\model$ need to be for this approach to work in Section \ref{sec:theory}.

To disentangle Eqn. \ref{eqn_paper_confident_joint}, consider a simplified formulation:
\begin{align}
    \hat{\bm{X}}^{\text{(simple)}}_{ \tilde{y} \smalleq i,y^* \smalleq j} &  = \{\bm{x} \in \bm{X}_{\tilde{y} = i} : \,\, \hat{p} (\tilde{y} = j ;\bm{x}, \bm{\theta})  \ge t_j \} \nonumber
\end{align}

The simplified formulation, however, introduces \emph{label collisions} when an example $\bm{x}$ is confidently counted into more than one $\estpartition$ bin. Collisions only occur along the $y^*$ dimension of $\cj$ because $\tilde{y}$ is given. We handle collisions in the right-hand side of Eqn. \ref{eqn_paper_confident_joint} by selecting $\hat{y}^* \gets \argmax_{j \in [m]} \; \hat{p} (\tilde{y} = j ;\bm{x}, \bm{\theta})$ whenever $\lvert \left\{ k \smallin [m] : \hat{p} (\tilde{y} \smalleq k ;\bm{x} \smallin \bm{X}_{\tilde{y} = i}, \bm{\theta})  \ge t_k \right\} \rvert > 1$ (collision). In practice with softmax, collisions sometimes occur for softmax outputs with higher temperature (more uniform probabilities), few collisions occur with lower temperature, and no collisions occur with a temperature of zero (one-hot prediction probabilities).


The definition of $\cj$ in Eqn. \ref{eqn_paper_confident_joint} has some nice properties in certain circumstances.
First, if an example has low (near-uniform) predicted probabilities across classes, then it will not be counted for any class in $\cj$ so that $\cj$ may be robust to pure noise or examples from an alien class not in the dataset.
Second, $\cj$ is intuitive -- $t_j$ embodies the intuition that examples with higher probability of belonging to class $j$ than the expected probability of examples in class $j$ probably belong to class $j$.
Third, thresholding allows flexibility -- for example, the $90^{th}$ percentile may be used in $t_j$ instead of the mean to find errors with higher confidence; despite the flexibility, we use the mean because we show (in Sec. \ref{sec:theory}) that this formulation exactly finds label errors in various settings, and we leave the study of other formulations, like a percentile-based threshold, as future work.

\paragraph{Complexity} We provide algorithmic implementations of Eqns. \ref{eqn_paper_threshold}, \ref{eqn_paper_confident_joint}, and \ref{eqn_calibration} in the Appendix. Given predicted probabilities $\bm{\hat{P}}_{k,i}$ and noisy labels $\tilde{y}$, these require $\mathcal{O}(m^2 + nm)$ storage and arithmetic operations to compute $\cj$, for $n$ training examples over $m$ classes.

\paragraph{Estimate the joint $\estjoint$.} Given the confident joint $\cj$, we estimate $\joint$ as
\begin{equation} \label{eqn_calibration}
    \estjointlong = \frac{\frac{\bm{C}_{\tilde{y}=i, y^*=j}}{\sum_{j \in [m]} \bm{C}_{\tilde{y}=i, y^*=j}} \cdot \lvert \bm{X}_{\tilde{y}=i} \rvert}{\sum\limits_{i \in [m], j \in [m]} \left( \frac{\bm{C}_{\tilde{y}=i, y^*=j}}{\sum_{j^\prime \in [m]} \bm{C}_{\tilde{y}=i, y^*=j^\prime}} \cdot \lvert \bm{X}_{\tilde{y}=i} \rvert \right)}
\end{equation}
The numerator calibrates $\sum_j \estjointlong = \lvert \bm{X}_i \rvert / \sum_{i \in [m]}  \lvert \bm{X}_i \rvert, \forall i \smallin [m]$ so that row-sums match the observed marginals. The denominator calibrates $\sum_{i,j} \estjointlong = 1$ so that the distribution sums to 1.

\paragraph{Label noise characterization} Using the observed prior $\noisypriorlong = \lvert \bm{X}_i \rvert \, / \, \sum_{i \in [m]}  \lvert \bm{X}_i \rvert $ and marginals of $\joint$, we estimate the latent prior as $\estpriorlong  \coloneqq  \sum_i \estjointlong, \forall j \smallin [m]$; the noise transition matrix (noisy channel) as $\estnmlong \coloneqq \estjointlong / \estpriorlong, \forall i \smallin [m]$; and the mixing matrix \citep{scott19_jmlr_mutual_decontamination} as $\estinvlong \coloneqq \hat{\bm{Q}}^\top_{\tilde{y} \smalleq j,y^* = i} / \noisypriorlong, \forall i \smallin [m]$. As long as $\estjoint \approxeq \joint$, each of these estimators is similarly consistent (we prove this is the case under practical conditions in Sec. \ref{sec:theory}).  Whereas prior approaches compute the noise transition matrices by directly averaging error-prone predicted probabilities \citep{noisy_boostrapping_google_reed_iclr_2015, DBLP:conf/iclr/GoldbergerB17_smodel}, CL is one step removed from the predicted probabilities by estimating noise rates based on counts from $\cj$ -- these counts are computed based on whether the predicted probability is greater than a threshold, relying only on the \emph{relative ranking} of the predicted probability, not its exact value. This feature lends itself to the robustness of confident learning to imperfect probability estimation.

\paragraph{Baseline approach $\bm{C}_{\text{confusion}}$}
To situate our understanding of $\cj$ performance in the context of prior work, we compare $\cj$ with $\bm{C}_{\text{confusion}}$, a baseline based on a single-iteration of the performant INCV method \citep{chen2019confusion}. $\bm{C}_{\text{confusion}}$ forms an $m \times m$ confusion matrix of counts $\lvert \tilde{y}_k = i, y^*_k = j \rvert$ across all examples $\bm{x}_k$, assuming that model predictions, trained from noisy labels, uncover the true labels, i.e. $\bm{C}_{\text{confusion}}$ simply assumes $y^*_k = \argmax_{i \in [m]} \hat{p}(\tilde{y}\smalleq i;\bm{x}_k, \bm{\theta})$. This baseline approach performs reasonably empirically (Sec. \ref{sec:experiments}) and is a consistent estimator for noiseless predicted probabilities (Thm. \ref{thm:exact_label_errors}), but fails when the distributions of probabilities are not similar for each class (Thm. \ref{thm:robustness}), e.g. class-imbalance, or when predicted probabilities are overconfident \citep{kilian_weinberger_calibration_Guo2017_icml}.

\paragraph{Comparison of $\cj$ (confident joint) with $\bm{C}_{\text{confusion}}$ (baseline)}
To overcome the sensitivity of $\bm{C}_{\text{confusion}}$ to class-imbalance and distribution heterogeneity, the \emph{confident joint}, $\cj$, uses per-class thresholding \citep{richard1991neural_thresholding, elkan_cost_sensitive_learning_thresholding} as a form of calibration \citep{hendrycks17baseline}. Moreover, we prove that unlike $\bm{C}_{\text{confusion}}$, the confident joint (Eqn. \ref{eqn_paper_confident_joint}) exactly finds label errors and consistently estimates $\joint$ in more realistic settings with noisy predicted probabilities (see Sec. \ref{sec:theory}, Thm. \ref{thm:robustness}).

\subsection{Rank and Prune: Data Cleaning} \label{sec:rank_prune_find_errors}

Following the estimation of $\cj$ and $\joint$ (Section \ref{sec:label_noise_characterization}), any rank and prune approach can be used to clean data. This \emph{modularity} property allows CL to find label errors using interpretable and explainable ranking methods, whereas prior works typically couple estimation of the noise transition matrix with training loss \citep{DBLP:conf/iclr/GoldbergerB17_smodel} or couple the label confidence of each example with the training loss using loss reweighting \citep{NIPS2013_5073, jiang2018mentornet}. In this paper, we investigate and evaluate five rank and prune methods for finding label errors, grouped into two approaches. We provide a theoretical analysis for Method 2: $\cj$ in Sec. \ref{sec:theory} and evaluate all methods empirically in Sec. \ref{sec:experiments}.

\paragraph{Approach 1: Use off-diagonals of $\cj$ to estimate $\estpartition$} We directly use the sets of examples counted in the off-diagonals of $\cj$ to estimate label errors. 

\textbf{\emph{CL baseline 1: $\bm{C}_{\text{confusion}}$.}} \;\; Estimate label errors as the Boolean vector ${\tilde{y}}_{k} \neq \argmax_{j \in [m]} \hat{p}(\tilde{y} = j; \bm{x}_k, \bm{\theta})$, for all $\bm{x}_k \smallin \bm{X}$, where \emph{true} implies label error and \emph{false} implies clean data. This is identical to using the off-diagonals of $\bm{C}_{\text{confusion}}$ and similar to a single iteration of INCV \citep{chen2019confusion}.

\textbf{\emph{CL method 2: $\cj$.}} \;\; Estimate label errors as $\{\bm{x} \in \estpartition : i \neq j \}$ from the off-diagonals of $\cj$.

\paragraph{Approach 2: Use $n \cdot \estjoint$ to estimate $\lvert \estpartition \rvert$, prune by probability ranking} These approaches calculate $n \cdot \estjoint$ to estimate $\lvert \estpartition \rvert$, the count of label errors in each partition. They either sum over the $y^*$ dimension of $\lvert \estpartition \rvert$ to estimate and remove the number of errors in each class (prune by class), or prune for every off-diagonal partition (\emph{prune by noise rate}). The choice of which examples to remove is made by ranking the examples based on predicted probabilities.

\textbf{\emph{CL method 3: Prune by Class (PBC).}} \;\; For each class $i \in [m]$, select the \\ $n \cdot\sum_{j \in [m] : j \neq i} \left( \hat{\bm{Q}}_{\tilde{y} = i, y^* = j}[i] \right)$ examples with lowest self-confidence $\hat{p}(\tilde{y}=i;\bm{x} \in \bm{X}_i)$ .

\textbf{\emph{CL method 4: Prune by Noise Rate (PBNR).}} \;\; For each off-diagonal entry in $\estjointlong, i \neq j$, select the $n \cdot \estjointlong$ examples $\bm{x} \smallin \bm{X}_{\tilde{y}\smalleq i}$ with max margin $\predprobshortj - \predprobshorti$.
This margin is adapted from \possessivecite{wei2018nomralizedmaxmargin} normalized margin.

\textbf{\emph{CL method 5: C+NR.}} \;\; Combine the previous two methods via element-wise `\emph{and}', i.e. set intersection. Prune an example if both methods PBC and PBNR prune that example.

\paragraph{Learning with Noisy Labels} To train with errors removed, we account for missing data by reweighting the loss by $\frac{1}{\hat{p}(\tilde{y}=i \vert y^* = i)} \smalleq \frac{\estprior[i]}{\estjoint[i][i]}$ for each class $i \smallin [m]$, where dividing by $\estjoint[i][i]$ normalizes out the count of clean training data and $\estprior[i]$ re-normalizes to the latent number of examples in class $i$. CL finds errors, but does not prescribe a specific training procedure using the clean data. Theoretically, CL requires no hyper-parameters to find label errors. In practice, cross-validation might introduce a hyper-parameter: $k$-fold. However, in our paper $k=4$ is fixed in the experiments using cross-validation.

\paragraph{Which CL method to use?} Five methods are presented to clean data. By default we use CL: $\cj$ because it matches the conditions of Thm. \ref{thm:robustness} exactly and is experimentally performant (see Table \ref{table:cifar10_label_error_measures}). Once label errors are found, we observe ordering label errors by the normalized margin: $\hat{p}(\tilde{y} \smalleq i; \bm{x}, \bm{\theta}) - \max_{j \neq i} \hat{p}(\tilde{y} \smalleq j; \bm{x}, \bm{\theta})$ \citep{wei2018nomralizedmaxmargin} works well. 

\section{Theory} \label{sec:theory}

In this section, we examine sufficient conditions when (1) the confident joint exactly finds label errors and (2) $\estjoint$ is a consistent estimator for $\joint$. We first analyze CL for noiseless $\predprobshortj$, then evaluate more realistic conditions, culminating in Thm. \ref{thm:robustness} where we prove (1) and (2) with noise in the predicted probabilities for every example. Proofs are in the Appendix (see Sec. \ref{sec:proofs}). As a notation reminder, $\predprobshorti$ is shorthand for $\hat{p}(\tilde{y} \smalleq i; \bm{x}, \bm{\theta})$. 

In the statement of each theorem, we use $\estjoint \approxeq \joint$, i.e. \emph{approximately equals}, to account for precision error of using discrete count-based $\cj$ to estimate real-valued $\joint$. For example, if a noise rate is $0.39$, but the dataset has only 5 examples in that class, the nearest possible estimate by removing errors is $2/5 = 0.4  \approxeq 0.39$. So, $\estjoint$ is technically a \emph{consistent estimator} for $\joint$ only because of discretization error, otherwise all equalities are exact. Throughout, we assume $\bm{X}$ includes at least one example from every class.

\subsection{Noiseless Predicted Probabilities}

We start with the \emph{ideal} condition and a non-obvious lemma that yields a closed-form expression for threshold $t_i$ when $\predprobshorti$ is ideal. Without some condition on $\predprobshorti$, one cannot disambiguate label noise from model noise. 

\begin{condition}[Ideal]
The predicted probabilities $\pyx$ for a model $\theta$ are \emph{ideal} if $\forall \bm{x}_k \smallin \bm{X}_{y^* \tinyeq j}, i \smallin [m], j \smallin [m]$,
we have that
$\hat{p}(\tilde{y} \smalleq i ; \bm{x}_k \in \bm{X}_{y^* \tinyeq j}, \bm{\theta}) = p^*(\tilde{y} \smalleq i \vert y^* \smalleq y_k^*) = p^*(\tilde{y} \smalleq i \vert y^* \smalleq j)$. \emph{The final equality follows from the class-conditional noise process assumption. The \emph{ideal} condition implies error-free predicted probabilities: they match the noise rates corresponding to the $y^*$ label of $\bm{x}$. We use $\perfprobshorti$ as a shorthand.}
\end{condition}

\begin{lemma}[Ideal Thresholds] \label{lemma:ideal_threshold} 
For a noisy dataset $\bm{X} \coloneqq (\bm{x}, \tilde{y})^n \in {(\mathbb{R}^d, [m])}^n$ and model $\model$, 
if $\hat{p}(\tilde{y}; \bm{x}, \model)$ is \emph{ideal}, then 
$\forall i \smallin [m], t_i = \sum_{j \smallin [m]} p(\tilde{y} = i \vert y^* \smalleq j) p(y^* \smalleq j \vert \tilde{y} = i)$.
\end{lemma}

This form of the threshold is intuitively reasonable: the contributions to the sum when $i=j$ represents the probabilities of correct labeling, whereas when $i\neq j$, the terms give the probabilities of mislabeling $p(\tilde{y} = i \vert y^* = j)$, weighted by the probability $p(y^* = j \vert \tilde{y} = i)$ that the mislabeling is corrected. Using Lemma \ref{lemma:ideal_threshold} under the ideal condition, we prove in Thm. \ref{thm:exact_label_errors} confident learning exactly finds label errors and $\estjoint$ is a consistent estimator for $\joint$ when each diagonal entry of $\nm$ maximizes its row and column. The proof hinges on the fact that the construction of $\cj$ eliminates collisions.

\begin{theorem}[Exact Label Errors] \label{thm:exact_label_errors} 
For a noisy dataset, $\bm{X} \coloneqq (\bm{x}, \tilde{y})^n \smallin {(\mathbb{R}^d, [m])}^n$ and model $\bm{\theta}{\scriptstyle:} \bm{x} \smallra \hat{p}(\tilde{y})$, 
if $\hat{p}(\tilde{y} ; \bm{x}, \bm{\theta})$ is \emph{ideal} and each diagonal entry of $\nm$ maximizes its row and column, 
then $\hat{\bm{X}}_{ \tilde{y} \smalleq i,y^* \smalleq j} = \bm{X}_{\tilde{y} \smalleq i, y^* \smalleq j}$ and $\estjoint \approxeq \joint$ (consistent estimator for $\joint$).
\end{theorem}




While Thm. \ref{thm:exact_label_errors} is a reasonable sanity check, observe that $ y^* \gets \argmax_j \; \hat{p}(\tilde{y} \smalleq i \vert \tilde{y}^* \smalleq i; \bm{x}) $, used by $\bm{C}_{\text{confusion}}$, trivially satisfies Thm. \ref{thm:exact_label_errors} if the diagonal of $\nm$ maximizes its row and column. We highlight this because $\bm{C}_{\text{confusion}}$ is the variant of CL most-related to prior work (e.g., \citet{chen2019confusion}). We next consider relaxed conditions \emph{motivated by real-world settings}~(e.g., \citet{jiang2020beyond}) where $\cj$ exactly finds label errors ($\hat{\bm{X}}_{ \tilde{y} \smalleq i,y^* \smalleq j} = \bm{X}_{\tilde{y} \smalleq i, y^* \smalleq j}$) and consistently estimates the joint distribution of noisy and true labels ($\estjoint \approxeq \joint$), but $\bm{C}_{\text{confusion}}$ does not.

\subsection{Noisy Predicted Probabilities}

Motivated by the importance of addressing class imbalance and heterogeneous class probability distributions, we consider linear combinations of noise per-class. Here, we index $\predprobshortj$ by $j$ to match the comparison $\hat{p} (\tilde{y} \smalleq j ;\bm{x}, \bm{\theta}) \ge t_j$ from the construction of $\cj$ (see Eqn. \ref{eqn_paper_confident_joint}).

\begin{condition} [Per-Class Diffracted]
    $\predprobshortj$ is \emph{per-class diffracted} if there exist linear combinations of class-conditional error in the predicted probabilities s.t. $\predprobshortj = \epsilon_j^{(1)} \perfprobshortj + \epsilon_j^{(2)}$ where $\epsilon_j^{(1)}, \epsilon_j^{(2)} \in \mathbb{R}$ and $\epsilon_j$ can be any distribution. \emph{This relaxes the \emph{ideal} condition with noise that is relevant for neural networks, which are known to be class-conditionally overly confident \citep{kilian_weinberger_calibration_Guo2017_icml}.}
\end{condition}



\begin{corollary}[Per-Class Robustness] \label{cor:per_class_robustness} 
For a noisy dataset, $\bm{X} \coloneqq (\bm{x}, \tilde{y})^n \smallin {(\mathbb{R}^d, [m])}^n$ and model $\bm{\theta}{\scriptstyle:} \bm{x} \smallra \hat{p}(\tilde{y})$, 
if $\predprobshortj$ is \textbf{per-class diffracted} without label collisions and each diagonal entry of $\nm$ maximizes its row, then $\hat{\bm{X}}_{ \tilde{y} \smalleq i,y^* \smalleq j} = \bm{X}_{\tilde{y} \smalleq i, y^* \smalleq j}$ and $\estjoint \approxeq \joint$.
\end{corollary}

Cor. \ref{cor:per_class_robustness} shows us that $\cj$ in confident learning (which counts $\hat{\bm{X}}_{ \tilde{y} \smalleq i,y^* \smalleq j}$) is robust to any linear combination of per-class error in probabilities. This is not the case for $\bm{C}_{\text{confusion}}$ because Cor. \ref{cor:per_class_robustness} no longer requires that the diagonal of $\nm$ maximize its column as before in Thm. \ref{thm:exact_label_errors}: for intuition, consider an extreme case of per-class diffraction where the probabilities of only one class are all dramatically increased. Then $\bm{C}_{\text{confusion}}$, which relies on $\tilde{y}^*_k \gets \argmax_{i \in [m]} \; \hat{p}(\tilde{y} \smalleq i \vert y^* \smalleq j; \bm{x}_k) $, will count only that one class for all $y^*$ such that all entries in the $\bm{C}_{\text{confusion}}$ will be zero except for one column, i.e. $\bm{C}_{\text{confusion}}$ cannot count entries in any other column, so $\hat{\bm{X}}_{ \tilde{y} \smalleq i,y^* \smalleq j} \neq \bm{X}_{\tilde{y} \smalleq i, y^* \smalleq j}$. In comparison, for $\cj$, the increased probabilities of the one class would be subtracted by the class-threshold, re-normalizing the columns of the matrix, such that, $\cj$ satisfies Cor. \ref{cor:per_class_robustness} using thresholds for robustness to distributional shift and class-imbalance.

Cor. \ref{cor:per_class_robustness} only allows for $m$ alterations in the probabilities and there are only $m^2$ unique probabilities under the ideal condition, whereas in real-world conditions, an error-prone model could potentially output $n \times m$ unique probabilities. Next, in Thm. \ref{thm:robustness}, we examine a reasonable sufficient condition where CL is robust to erroneous probabilities for every example and class.

\begin{condition}[Per-Example Diffracted]
$\predprobshortj$ is \emph{per-example diffracted} if $ \forall j \smallin [m], \forall \bm{x} \smallin \bm{X}$, we have error as $\predprobshortj = \perfprobshortj + \errorxj$ where 
\begin{equation} \label{eqn:piecewise_error}
    \errorxj \sim \begin{cases} 
      \mathcal{U}(\epsilon_j \smallplus t_j \smallminus \perfprobshortj \, , \, \epsilon_j \smallminus t_j \smallplus \perfprobshortj] & \perfprobshortj \geq t_j \\
      \mathcal{U}[\epsilon_j \smallminus t_j \smallplus \perfprobshortj \, , \, \epsilon_j \smallplus t_j \smallminus \perfprobshortj) & \perfprobshortj < t_j
  \end{cases}
\end{equation}
where $\epsilon_j = \mathop{\mathbb{E}}_{\bm{x} \in \bm{X}} \big[\errorxj\big]$ and $\mathcal{U}$ denotes a uniform distribution (we discuss a more general case in the Appendix).
\end{condition}



\begin{theorem}[Per-Example Robustness] \label{thm:robustness} 
For a noisy dataset, $\bm{X} \coloneqq (\bm{x}, \tilde{y})^n \in {(\mathbb{R}^d, [m])}^n$ and model $\bm{\theta}{\scriptstyle:} \bm{x} \smallra \hat{p}(\tilde{y})$, 
if $\predprobshortj$ is \textbf{per-example diffracted} without label collisions and each diagonal entry of $\nm$ maximizes its row, then $\hat{\bm{X}}_{ \tilde{y} \smalleq i,y^* \smalleq j} \approxeq  \bm{X}_{\tilde{y} \smalleq i, y^* \smalleq j}$ and $\estjoint \approxeq \joint$.
\end{theorem}

In Thm. \ref{thm:robustness}, we observe that if each example's predicted probability resides within the residual range of the ideal probability and the threshold, then CL exactly identifies the label errors and consistently estimates $\joint$. Intuitively, if $\predprobshortj \geq t_j$ whenever $\perfprobshortj \geq t_j$ and $\predprobshortj < t_j$ whenever $\perfprobshortj < t_j$, then regardless of error in $\predprobshortj$, CL exactly finds label errors. As an example, consider an image $\bm{x}_k$ that is mislabeled as \emph{fox}, but is actually a \emph{dog} where $t_{fox} =0.6$, $p^*(\tilde{y} \smalleq fox; \bm{x} \in \bm{X}_{y^* \smalleq dog}, \bm{\theta}) = 0.2 $, $t_{dog} =0.8$, and  $p^*(\tilde{y} \smalleq dog; \bm{x} \in \bm{X}_{y^* \smalleq dog}, \bm{\theta}) = 0.9 $. Then as long as $-0.4 \leq \epsilon_{\bm{x}, fox} < 0.4$ and $-0.1 < \epsilon_{\bm{x}, dog} \le 0.1$, CL will surmise $y_k^* = dog$, not $fox$, even though $\tilde{y}_k = fox$ is given. We empirically substantiate this theoretical result in Section \ref{sec:exp_imagenet}.


Thm. \ref{thm:robustness} addresses the \emph{epistemic} uncertainty of latent label noise, via the statistic, $\joint$, while accounting for the \emph{aleatoric} uncertainty of inherently erroneous predicted probabilities.





\section{Experiments}\label{sec:experiments}
This section empirically validates CL on CIFAR \citep{cifar10} and ImageNet \citep{ILSVRC15_imagenet} benchmarks. Sec. \ref{sec:exp_synth} presents CL performance on noisy examples in CIFAR where true labels are presumed known. Sec. \ref{sec:exp_imagenet} shows real-world label errors found in the original, unperturbed MNIST, ImageNet, WebVision, and Amazon Reviews datasets, and shows performance advantages using cleaned data provided by CL to train ImageNet. Unless otherwise specified, we compute out-of-sample predicted probabilities $\bm{\hat{P}}_{k,j}$ using four-fold cross-validation and ResNet architectures.

\subsection{Asymmetric Label Noise on CIFAR-10 dataset} \label{sec:exp_synth}

We evaluate CL on three criteria: (a) joint estimation  (Fig. \ref{fig:cifar10_joint_estimation}), (b) accuracy finding label errors  (Table \ref{table:cifar10_label_error_measures}), and (c) accuracy learning with noisy labels (Table \ref{table:cifar10_benchmark}).


\paragraph{Noise Generation} Following prior work by~\cite{Sukhbaatar_fergus_iclr_2015, DBLP:conf/iclr/GoldbergerB17_smodel}, we verify CL performance on the commonly used asymmetric label noise, where the labels of error-free/clean data are randomly flipped, for its resemblance to real-world noise. We generate noisy data from clean data by randomly switching some labels of training examples to different classes non-uniformly according to a randomly generated $\nm$ noise transition matrix. We generate $\nm$ matrices with different traces to run experiments for different noise levels. The noise matrices used in our experiments are in the Appendix in Fig. \ref{cifar10_ground_truth_noise_matrices}.
We generate noise in the CIFAR-10 training dataset across varying \emph{sparsities}, the fraction of off-diagonals in $\joint$ that are zero, and the percent of incorrect labels (noise).
We evaluate all models on the unaltered test set.


\paragraph{Baselines and our method} In Table \ref{table:cifar10_benchmark}, we compare CL performance versus seven recent highly competitive approaches and a vanilla baseline for multiclass learning with noisy labels on CIFAR-10, including \emph{INCV} \citep{chen2019confusion} which finds clean data with multiple iterations of cross-validation then trains on the clean set, \emph{SCE-loss} (symmetric cross entropy) \citep{wang2019sceloss_symmetric} which adds a reverse cross entropy term for loss-correction, \emph{Mixup} \citep{icml2018mixup} which linearly combines examples and labels to augment data, \emph{MentorNet} \citep{jiang2018mentornet} which uses curriculum learning to avoid noisy data in training, \emph{Co-Teaching} \citep{han2018coteaching} which trains two models in tandem to learn from clean data, \emph{S-Model} \citep{DBLP:conf/iclr/GoldbergerB17_smodel} which uses an extra softmax layer to model noise during training, and \emph{Reed} \citep{noisy_boostrapping_google_reed_iclr_2015} which uses loss-reweighting; and a \emph{Baseline} model that denotes a vanilla training with the noisy labels.

\paragraph{Training settings} All models are trained using ResNet-50 with the common setting: learning rate 0.1 for epoch [0,150), 0.01 for epoch [150,250), 0.001 for epoch [250,350); momentum 0.9; and weight decay 0.0001, except \emph{INCV}, \emph{SCE-loss}, and \emph{Co-Teaching} which are trained using their official GitHub code. Settings are copied from the \href{https://github.com/kuangliu/pytorch-cifar/tree/5e3f99093dfe7392fcbbc0b39582e4b0d3a64511#learning-rate-adjustment}{kuangliu/pytorch-cifar} GitHub open-source code and were not tuned by hand. We report the highest score across hyper-parameters $\alpha \in \{1, 2, 4, 8\}$ for \emph{Mixup} and $p \in \{0.7, 0.8, 0.9\}$ for \emph{MentorNet}. For fair comparison with \emph{Co-Teaching}, \emph{INCV}, and \emph{MentorNet}, we also train using the \emph{co-teaching} approach with $\text{forget rate} = 0.5 \times \text{[noise fraction]}$, and report the max accuracy of the two trained models for each method.
We observe that dropping the last partial batch of each epoch during training improves stability by avoiding weight updates from, in some cases, a single noisy example). Exactly the same noisy labels are used for training all models for each column of Table \ref{table:cifar10_benchmark}. For our method, we fix its hyper-parameter, \emph{i.e.} the number of folds in cross-validation across different noise levels, and do not tune it on the validation set.

For each CL method, sparsity, and noise setting, we report the mean accuracy in Table \ref{table:cifar10_benchmark}, averaged over ten trials, by varying the random seed and initial weights of the neural network for training. Standard deviations are reported in Table \ref{table:cifar10_benchmark_std} to improve readability. For each column in Table \ref{table:cifar10_benchmark}, the corresponding standard deviations in in Table \ref{table:cifar10_benchmark_std} are significantly less than the performance difference between CL methods and baseline methods. Notably, all standard deviations are significantly ($\sim$10x) less than the mean performance difference between the top-performing CL method and baseline methods for each setting, averaged over random weight initialization. Standard deviations are only reported for CL methods because of difficulty reproducing consistent results for some of the other methods.

\begin{table*}[t]

\setlength\tabcolsep{2pt} 
\caption{Test accuracy (\%) of confident learning versus recent methods for learning with noisy labels in CIFAR-10. Scores reported for CL methods are averaged over ten trials with standard deviations shown in Table \ref{table:cifar10_benchmark_std}. CL methods estimate label errors, remove them, then train on the cleaned data. Whereas other methods decrease in performance from low sparsity (e.g., 0.0) to high sparsity (e.g. 0.6), CL methods are robust across sparsity, as indicated by comparing the two column-wise red highlighted cells. Data-centric AI methods (\textit{CL, INCV, Mixup}) outperform model-centric methods (\textit{SCE-Loss, MentorNet, Co-Teaching, S-Model}).}
\vskip -0.1in
\label{table:cifar10_benchmark}
\begin{center}\
\resizebox{\textwidth}{!}{ 

\begin{tabular}{l|cccc|cccc|cccc}
\toprule
Noise &  \multicolumn{4}{c}{20\%} & \multicolumn{4}{c}{40\%} & \multicolumn{4}{c}{70\%} \\
Sparsity  &    0 &      0.2 &      0.4 &      0.6 &      0 &      0.2 &      0.4 &      0.6 &      0 &      0.2 &      0.4 &      0.6 \\
\midrule

CL: $\bm{C}_{\text{confusion}}$    &  89.6 &  89.4 &  90.2 &  89.9 &  \cellcolor{highlightcolor} 83.9 &  83.9 &  83.2 &  \cellcolor{highlightcolor} 84.2 &  31.5 &  39.3 &  33.7 &  30.6 \\
CL: PBC                            &  90.5 &  90.1 &  90.6 &  90.7 &  \cellcolor{highlightcolor} 84.8 &  85.5 &  85.3 &  \cellcolor{highlightcolor} 86.2 &  33.7 &  40.7 &  35.1 &  31.4 \\
CL: $\cj$                          &  \textbf{91.1} &  \textbf{90.9} &  \textbf{91.1} &  \textbf{91.3} & \cellcolor{highlightcolor}  86.7 &   86.7 &  86.6 & \cellcolor{highlightcolor} 86.9 &  32.4 &  \textbf{41.8} &  34.4 &  34.5 \\
CL: C+NR                           &  90.8 &  90.7 &  91.0 &  91.1 &  \cellcolor{highlightcolor} \textbf{87.1} &  \textbf{86.9} &  \textbf{86.7} & \cellcolor{highlightcolor} \textbf{87.2} &  \textbf{41.1} &  41.7 &  39.0 &  32.9 \\
CL: PBNR                           &  90.7 &  90.5 &  90.9 &  90.9 &  \cellcolor{highlightcolor} \textbf{87.1} &  86.8 &  86.6 &  \cellcolor{highlightcolor} \textbf{87.2} &  41.0 &  \textbf{41.8} &  \textbf{39.1} &  \textbf{36.4} \\

\midrule

INCV (Chen et al., 2019)  & 87.8 & 88.6  & 89.6  & 89.2  & \cellcolor{highlightcolor} 84.4 &  76.6 &  85.4 &  \cellcolor{highlightcolor} 73.6  & 28.3  & 25.3 &  34.8  & 29.7 \\
Mixup (Zhang et al., 2018)  & 85.6 & 86.8 & 87.0 & 84.3 & \cellcolor{highlightcolor} 76.1 & 75.4 &  68.6 & \cellcolor{highlightcolor} 59.8 & 32.2 & 31.3 & 32.3 & 26.9 \\
SCE-loss (Wang et al., 2019)  &  87.2  &  87.5  &  88.8  &  84.4  &  \cellcolor{highlightcolor} 76.3  &  74.1  &  64.9  & \cellcolor{highlightcolor} 58.3  &  33.0  &  28.7  &  30.9  &  24.0  \\
MentorNet (Jiang et al., 2018) &  84.9 &  85.1 &  83.2 &  83.4 &  \cellcolor{highlightcolor} 64.4 &  64.2 &  62.4 &  \cellcolor{highlightcolor} 61.5 &  30.0 &  31.6 &  29.3 &  27.9 \\
Co-Teaching (Han et al., 2018)  &  81.2 &  81.3 &  81.4 &  80.6 & \cellcolor{highlightcolor} 62.9 &  61.6 &  60.9 & \cellcolor{highlightcolor} 58.1 &  30.5 &  30.2 &  27.7 &  26.0 \\
S-Model (Goldberger et al., 2017)    &  80.0 &  80.0 &  79.7 &  79.1 &  58.6 &  61.2 &  59.1 &  57.5 &  28.4 &  28.5 &  27.9 &  27.3 \\
Reed (Reed et al., 2015)      &  78.1 &  78.9 &  80.8 &  79.3 &  60.5 &  60.4 &  61.2 &  58.6 &  29.0 &  29.4 &  29.1 &  26.8 \\
Baseline   &  78.4 &  79.2 &  79.0 &  78.2 &  60.2 &  60.8 &  59.6 &  57.3 &  27.0 &  29.7 &  28.2 &  26.8 \\

\bottomrule
\end{tabular}
}

\end{center}
\vskip -.2in
\end{table*}

\begin{table*}[!b]
\setlength\tabcolsep{2pt} 
\renewcommand{\arraystretch}{0.9}
\caption{Standard deviations (\% units) associated with the mean score (over ten trials) for scores reported for CL methods in Table \ref{table:cifar10_benchmark}. Each trial uses a different random seed and network weight initialization. No standard deviation exceeds 2\%.}
\vskip -0.1in
\label{table:cifar10_benchmark_std}
\begin{center}
\resizebox{\textwidth}{!}{ 

\setlength{\tabcolsep}{6.5pt}
\begin{tabular}{l|cccc|cccc|cccc}
\toprule
Noise &  \multicolumn{4}{c}{20\%} & \multicolumn{4}{c}{40\%} & \multicolumn{4}{c}{70\%} \\
Sparsity  &    0 &      0.2 &      0.4 &      0.6 &      0 &      0.2 &      0.4 &      0.6 &      0 &      0.2 &      0.4 &      0.6 \\
\midrule
CL: $\bm{C}_{\text{confusion}}$  &  0.07 &  0.10 &  0.17 &  0.08 &  0.19 &  0.22 &  0.23 &  0.20 &  0.93 &  0.24 &  0.13 &  0.26 \\
CL: PBC     &  0.14 &  0.12 &  0.11 &  0.10 &  0.15 &  0.17 &  0.16 &  0.10 &  0.12 &  0.22 &  0.11 &  0.30 \\
CL: $\cj$ &  0.17 &  0.09 &  0.17 &  0.11 &  0.10 &  0.20 &  0.09 &  0.13 &  1.02 &  0.15 &  0.18 &  1.63 \\
CL: C+NR    &  0.09 &  0.10 &  0.08 &  0.08 &  0.11 &  0.14 &  0.16 &  0.10 &  0.42 &  0.33 &  0.26 &  1.90 \\
CL: PBNR    &  0.15 &  0.09 &  0.09 &  0.10 &  0.18 &  0.10 &  0.15 &  0.12 &  0.26 &  0.28 &  0.24 &  1.43 \\
\bottomrule
\end{tabular}
}

\end{center}
\end{table*}


\begin{figure*}[!b]
    \centering
    \begin{subfigure}[b]{0.345 \textwidth}
        \centering
        \includegraphics[width=1\textwidth]{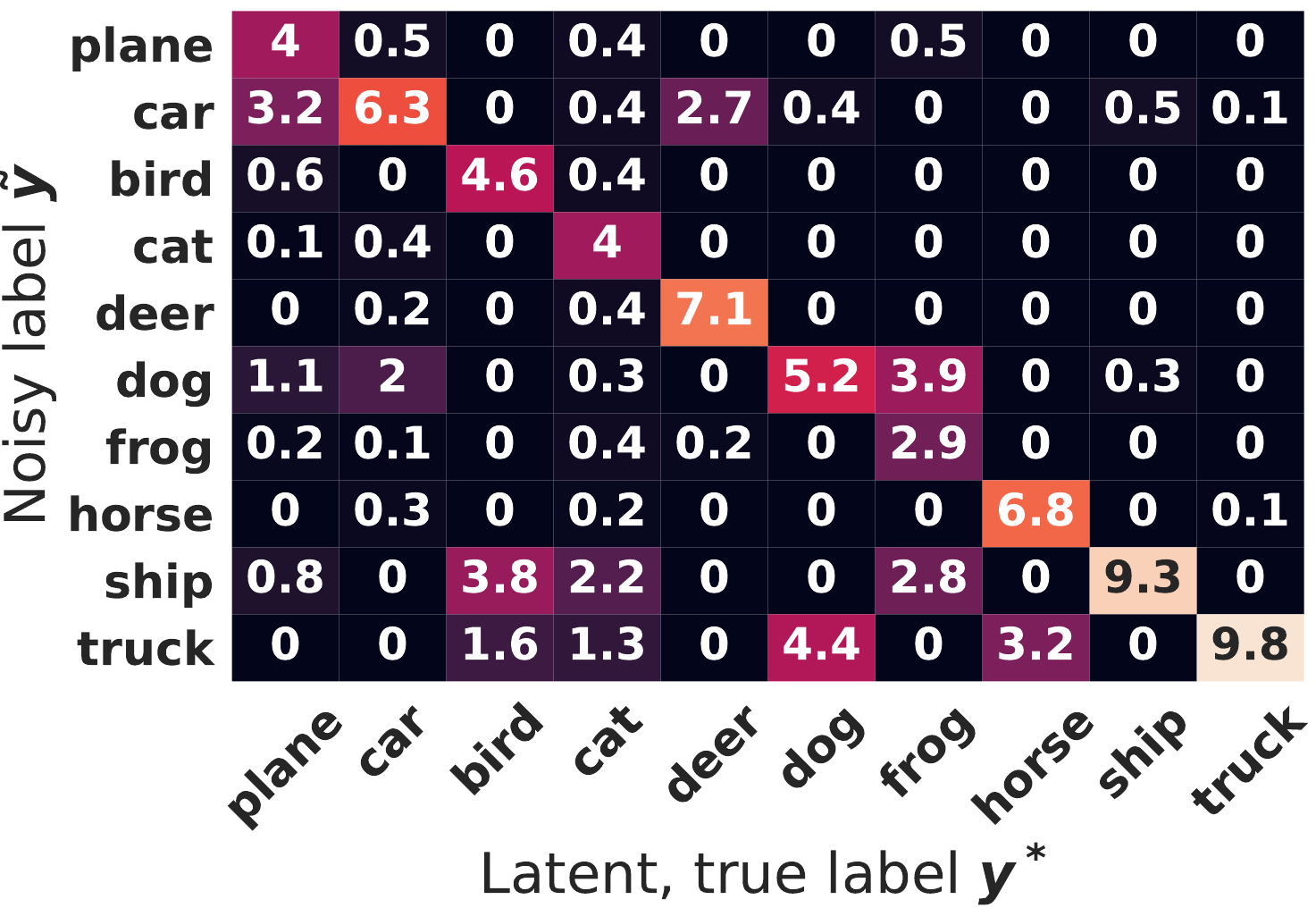}
        \caption[Network2]%
        {True $\joint$ (unknown to CL)}    
        \label{subfig:true_joint}
    \end{subfigure}
    \begin{subfigure}[b]{0.285\textwidth}  
        \centering 
        \includegraphics[width=1\textwidth]{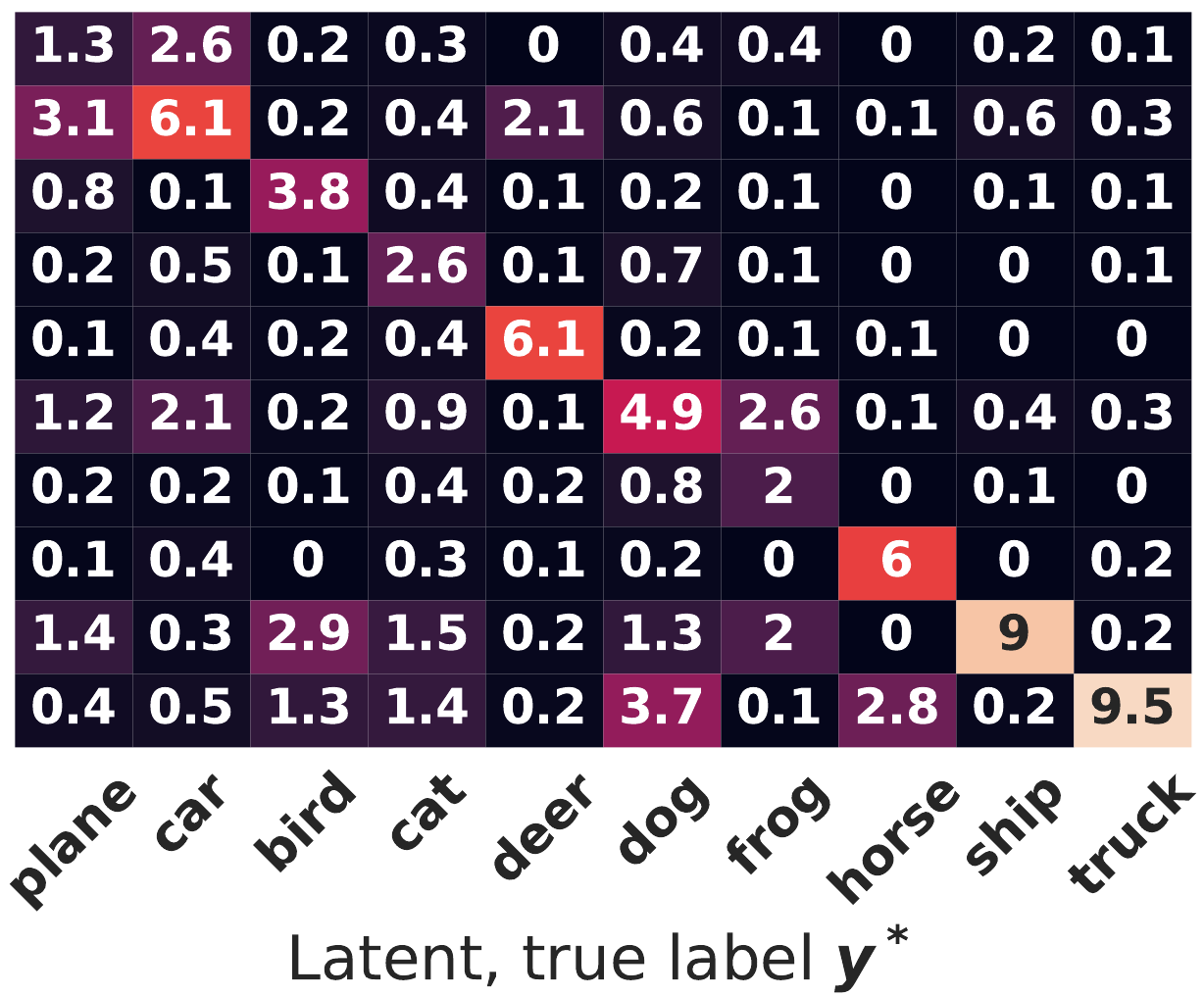}
        \caption[]%
        {CL estimated $\estjoint$}   
        \label{subfig:est_joint}
    \end{subfigure}
    \begin{subfigure}[b]{0.349\textwidth}   
        \centering 
        \includegraphics[width=1\textwidth]{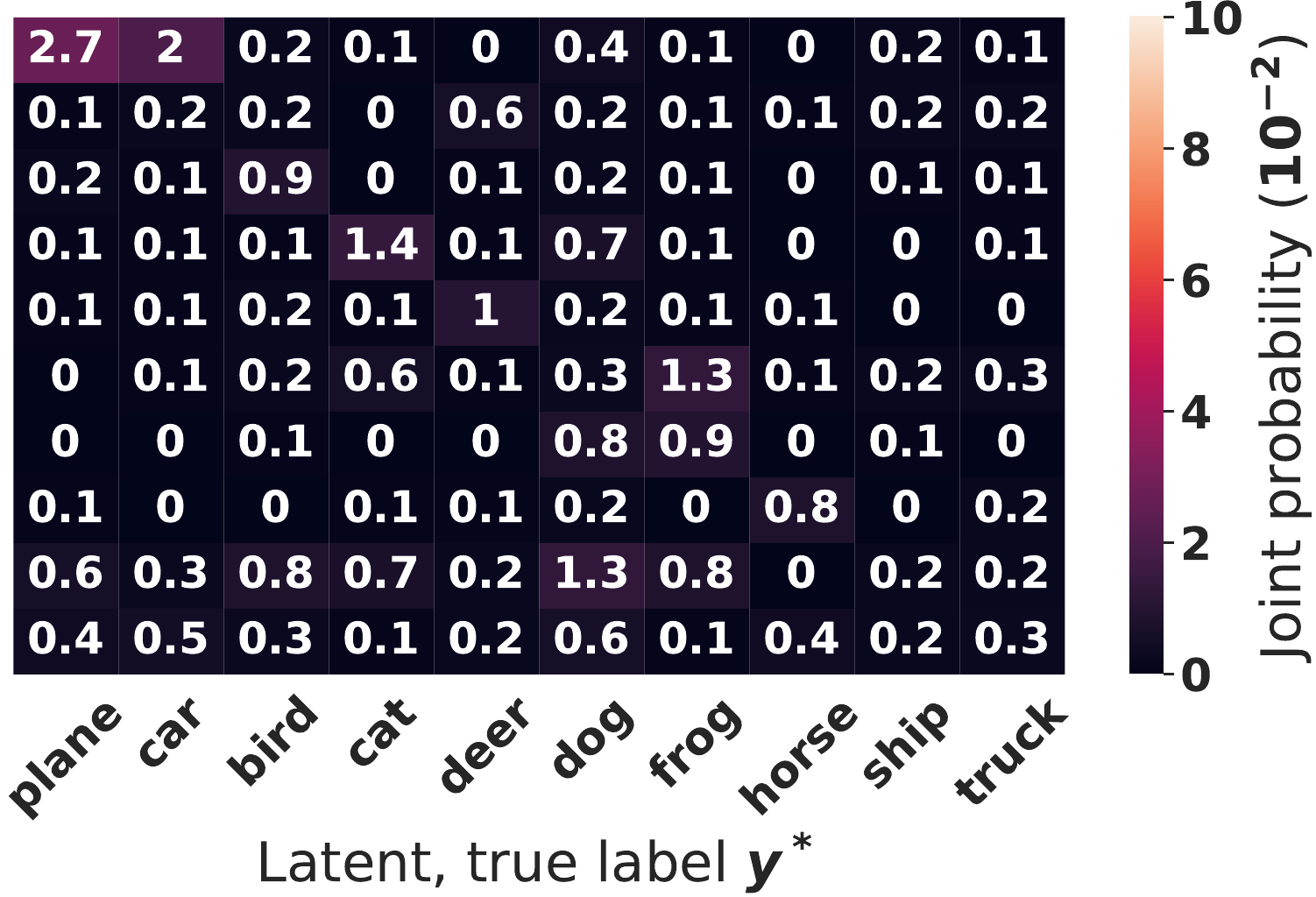}
        \caption[]%
        {Absolute diff. $ \lvert \bm{Q}_{\tilde{y}, y^*} - \hat{\bm{Q}}_{\tilde{y}, y^*} \rvert$}
        \label{subfig:abs_diff}
    \end{subfigure}
    \caption{Our estimation of the joint distribution of noisy labels and true labels for CIFAR with 40\% label noise and 60\% sparsity. Observe the similarity (RSME $=.004$) between (a) and (b) and the low absolute error in every entry in (c). Probabilities are scaled up by 100.} 
    \label{fig:cifar10_joint_estimation}
\end{figure*}

We also evaluate CL's accuracy in finding label errors. In Table \ref{table:cifar10_label_error_measures}, we compare five variants of CL methods across noise and sparsity and report their precision, recall, and F1 in recovering the true label. The results show that CL is able to find the label errors with high recall and reasonable F1.

\begin{table*}[t]

\setlength\tabcolsep{2pt} 
\renewcommand{\arraystretch}{1}
\caption{Mean accuracy, F1, precision, and recall measures of CL methods for finding label errors in CIFAR-10, averaged  over ten trials.}
\vskip -0.1in
\label{table:cifar10_label_error_measures}
\begin{center}
\resizebox{\textwidth}{!}{ 

\setlength{\tabcolsep}{3pt}
\begin{tabular}{l|cccc|cccc|cccc|cccc}

\toprule
Measure &\multicolumn{4}{c}{Accuracy (\%) $\pm$ Std. Dev. (\%)} & \multicolumn{4}{c}{F1 (\%)} & \multicolumn{4}{c}{Precision (\%)}  & \multicolumn{4}{c}{Recall (\%)} \\
Noise & \multicolumn{2}{c}{20\%} & \multicolumn{2}{c}{40\%} & \multicolumn{2}{c}{20\%} & \multicolumn{2}{c}{40\%} & \multicolumn{2}{c}{20\%} & \multicolumn{2}{c}{40\%}  & \multicolumn{2}{c}{20\%} & \multicolumn{2}{c}{40\%} \\
Sparsity &   0.0 &   0.6 &   0.0 &   0.6 &   0.0 &   0.6  &   0.0 &   0.6 &   0.0 &   0.6 &   0.0 &   0.6  &   0.0 &   0.6 &   0.0 &   0.6 \\
\midrule
CL: $\bm{C}_{\text{confusion}}$ &  $84{\pm}0.07$  &  $85{\pm}0.09$ &  $85{\pm}0.24$ &  $81{\pm}0.21$ &  71 &  72 &  84 &  79 & 56 &  58 &  74 &  70 & \textbf{98} &  \textbf{97} &  \textbf{97} &  \textbf{90} \\
CL: $\cj$ &  $89{\pm}0.15$ &  \textbf{90}${\pm}0.10$ &  $86{\pm}0.15$ &  \textbf{84}${\pm}0.12$ &  75 &  78 &  84 &  \textbf{80} &  \textbf{67} &  \textbf{70} &  78 &  77 & 86 &  88 &  91 &  84 \\
CL: PBC           &  $88{\pm}0.22$ &  $88{\pm}0.11$ &  $86{\pm}0.17$ &  $82{\pm}0.13$ &  76 &  76 &  84 &  79 &  64 &  65 &  76 &  74 &  96 &  93 &  94 &  85\\
CL: PBNR          &  $89{\pm}0.11$ &  \textbf{90}${\pm}0.08$ &  \textbf{88}${\pm}0.12$ &  \textbf{84}${\pm}0.11$ &  77 &  \textbf{79} &  \textbf{85} &  \textbf{80} &  65 &  68 &  \textbf{82} &  \textbf{79} &  93 &  94 &  88 &  82 \\
CL: C+NR          &  \textbf{90}${\pm}0.21$ &  \textbf{90}${\pm}0.10$ &  $87{\pm}0.23$ &  $83{\pm}0.14$ &  \textbf{78} &  78 &  84 &  78 &  \textbf{67} &  69 &  \textbf{82} &  \textbf{79} &  93 &  90 &  87 &  78 \\
\bottomrule
\end{tabular}
}

\end{center}
\end{table*}

\paragraph{Robustness to Sparsity} Table \ref{table:cifar10_benchmark} reports CIFAR test accuracy for learning with noisy labels across noise amount and sparsity, where the first five rows report our CL approaches. As shown, CL consistently performs well compared to prior art across all noise and sparsity settings.
We observe significant improvement in high-noise and/or high-sparsity regimes.
The simplest CL method $CL: \bm{C}_{\text{confusion}}$ performs similarly to \emph{INCV} and comparably to prior art with best performance by \textbf{$\cj$} across all noise and sparsity settings. The results validate the benefit of directly modeling the joint noise distribution and show that our method is competitive compared to highly competitive, robust learning methods.


To understand why CL performs well, we evaluate CL joint estimation across noise and sparsity with RMSE in Table \ref{table:cifar10_rmse} in the Appendix and estimated $\estjoint$ in Fig. \ref{cifar10_abs_diff_ALL} in the Appendix. For the 20\% and 40\% noise settings, on average, CL achieves an RMSE of $.004$ relative to the true joint $\joint$ across all sparsities. The simplest CL variant, $\bm{C}_{\text{confusion}}$ normalized via Eqn. (\ref{eqn_calibration}) to obtain $\hat{\bm{Q}}_{\text{confusion}}$, achieves a slightly worse RMSE of $.006$. 

In Fig. \ref{fig:cifar10_joint_estimation}, we visualize the quality of CL joint estimation in a challenging high-noise (40\%), high-sparsity (60\%) regime on CIFAR. Subfigure (a) demonstrates high sparsity in the latent true joint $\joint$, with over half the noise in just six noise rates. Yet, as can be seen in subfigures (b) and (c), CL still estimates over 80\% of the entries of $\joint$ within an absolute difference of $.005$. The results empirically substantiate the theoretical bounds of Section~\ref{sec:theory}.

In Table \ref{table:incv_stuff} (see Appendix), we report the training time required to achieve the accuracies reported in Table \ref{table:cifar10_benchmark} for INCV and confident learning. As shown in Table \ref{table:incv_stuff}, INCV training time exceeded 20 hours. In comparison, CL takes less than three hours on the same machine: an hour for cross-validation, less than a minute to find errors, and an hour to re-train.


\subsection{Real-world Label Errors in ILSVRC12 ImageNet Train Dataset}\label{sec:exp_imagenet}

\citet{ILSVRC15_imagenet} suggest label errors exist in ImageNet due to human error, but to our knowledge, few attempts have been made to find label errors in the ILSVRC 2012 training set, characterize them, or re-train without them. Here, we consider each application. We use ResNet18 and ResNet50 architectures with standard settings: 0.1 initial learning rate, 90 training epochs with 0.9 momentum.

\begin{table}[ht]
\renewcommand{\arraystretch}{0.8}
\caption{Ten largest non-diagonal entries in the confident joint $\cj$ for ImageNet train set used for ontological issue discovery. A duplicated class detected by CL is highlighted in red.}
\label{table_imagenet_characterization_top10}
\begin{center}

\begin{tabular}{rrrrrrr} 
\toprule
  $\cj$   &  $\tilde{y}$ name   &   $y^*$ name   &   $\tilde{y}$ nid   &   $y^*$ nid   &  $\bm{C}_{\text{confusion}}$  &   $\estjoint$  \\
\midrule
    645 &       projectile &       missile &       n04008634 &  n03773504 &   494  &            0.00050 \\
    539 &              tub &       bathtub &       n04493381 &  n02808440 &   400  &            0.00042 \\
    476 &      breastplate &       cuirass &       n02895154 &  n03146219 &   398  &            0.00037 \\
    437 &     green\_lizard &     chameleon &      n01693334 &  n01682714 &   369  &            0.00034 \\
    435 &        chameleon &  green\_lizard &      n01682714 &  n01693334 &   362  &            0.00034 \\
    433 &          missile &    projectile &       n03773504 &  n04008634 &   362  &            0.00034 \\
    \cellcolor{highlightcolor} 417 &         \cellcolor{highlightcolor} maillot &      \cellcolor{highlightcolor} maillot &      \cellcolor{highlightcolor} n03710637 &  \cellcolor{highlightcolor} n03710721 &  \cellcolor{highlightcolor} 338  &           \cellcolor{highlightcolor}  0.00033 \\
    416 &     horned\_viper &    sidewinder &      n01753488 &  n01756291 &   336  &            0.00033 \\
    410 &             corn &           ear &       n12144580 &  n13133613 &   333  &            0.00032 \\
    407 &        keyboard &     space\_bar &      n04505470 &  n04264628 &   293  &            0.00032 \\
\bottomrule
\end{tabular}

\end{center}
\end{table}




\paragraph{Ontological discovery for dataset curation} Because ImageNet is an one-hot class dataset, the classes are required to be mutually exclusive. Using ImageNet as a case study, we observe auto-discovery of ontological issues at the class level in Table \ref{table_imagenet_characterization_top10}, operationalized by listing the 10 largest non-diagonal entries in $\cj$. For example, the class \emph{maillot} appears twice, the existence of \emph{is-a} relationships like \emph{bathtub is a tub}, misnomers like \emph{projectile} and \emph{missile}, and unanticipated issues caused by words with multiple definitions like \emph{corn} and \emph{ear}. We include the baseline $\bm{C}_{\text{confusion}}$ to show that while $\bm{C}_{\text{confusion}}$ finds fewer label errors than $\cj$, they rank ontological issues similarly.

\begin{figure*}[ht]
\centerline{\includegraphics[width=\textwidth]{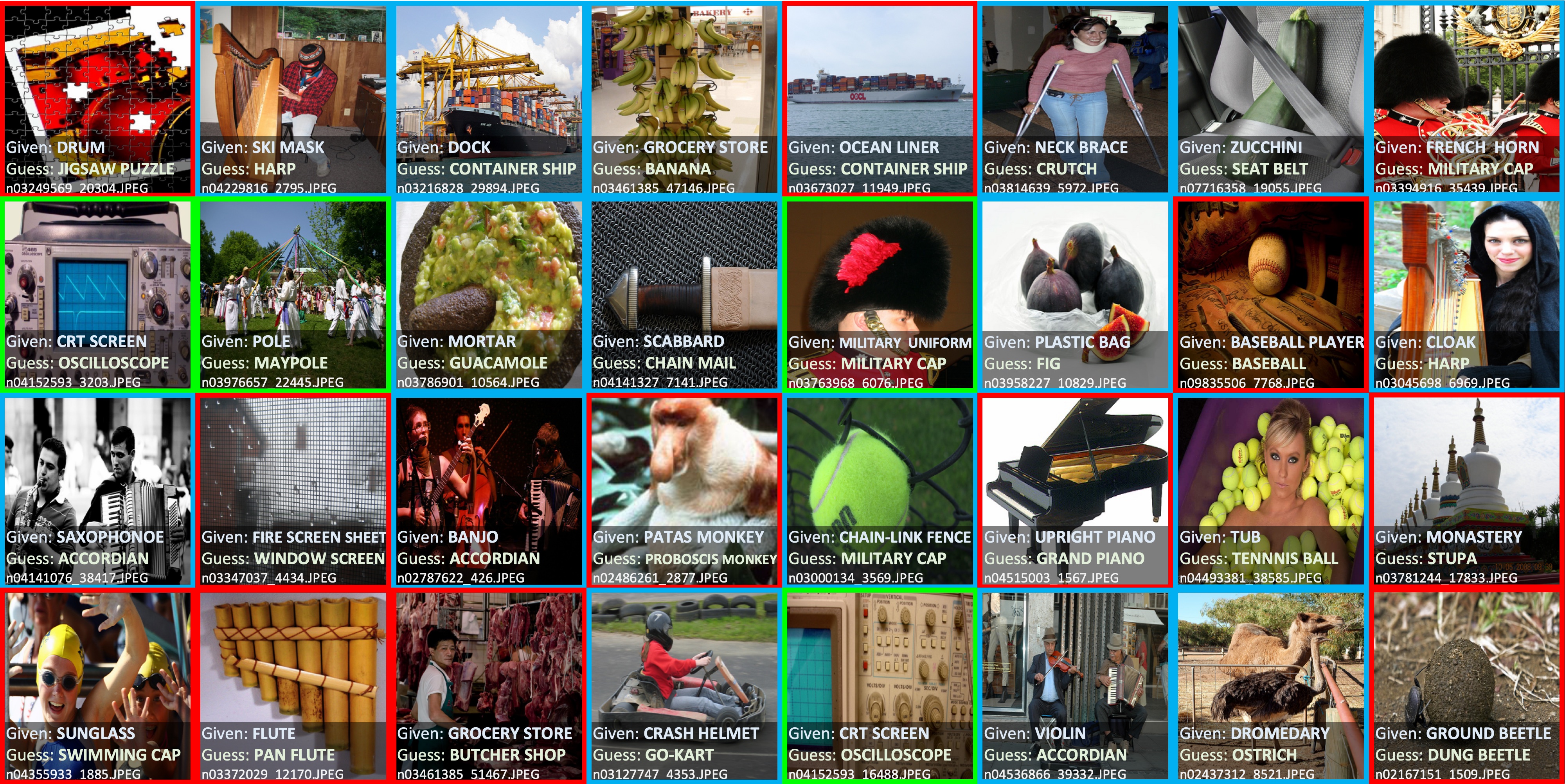}}
\caption{Top 32 (ordered automatically by normalized margin) identified label issues in the 2012 ILSVRC ImageNet train set using CL: PBNR. Errors are boxed in red. Ontological issues are boxed in green. Multi-label images are boxed in blue.} 
\label{label_errors_imagenet_train_top32}
\end{figure*}


\begin{figure*}[hbt!]
    \centering
    \begin{subfigure}[b]{0.475\textwidth}
        \centering
        \includegraphics[width=\textwidth]{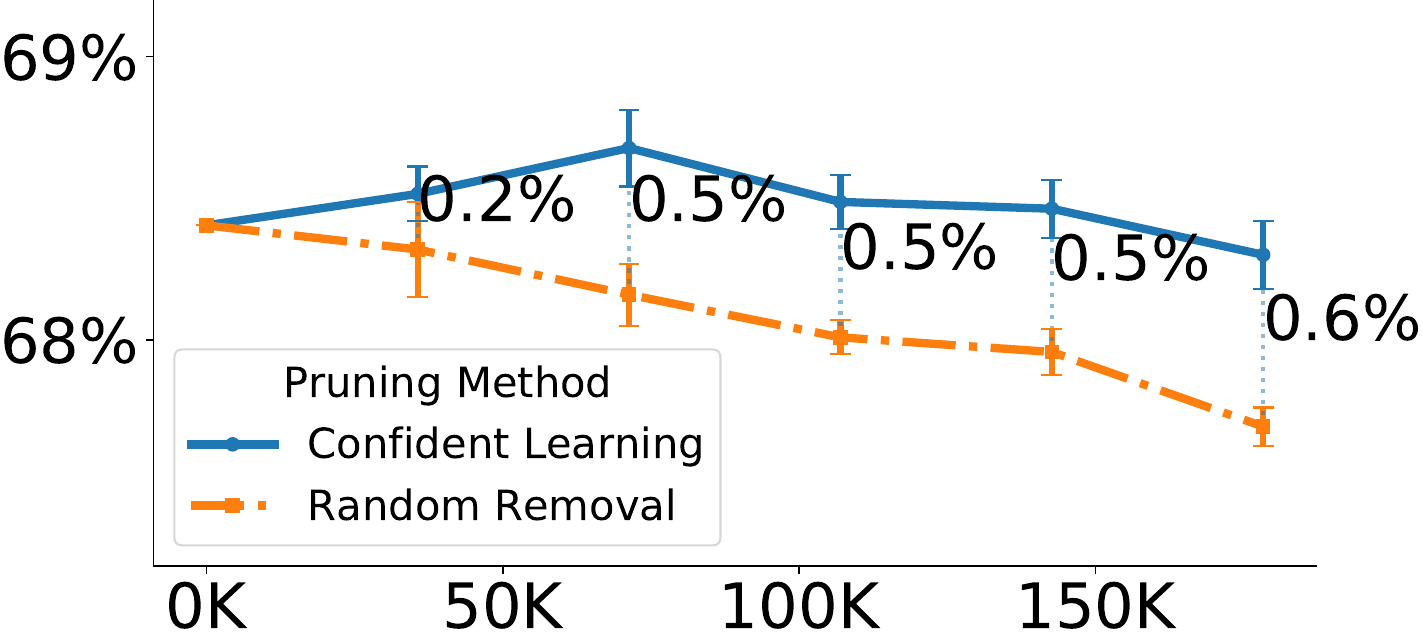}
        \caption[Network2]%
        {Accuracy on the ILSVRC2012 validation set}    
        \label{fig_resnet18_val_acc}
    \end{subfigure}
    \hfill
    \begin{subfigure}[b]{0.475\textwidth}  
        \centering 
        \includegraphics[width=\textwidth]{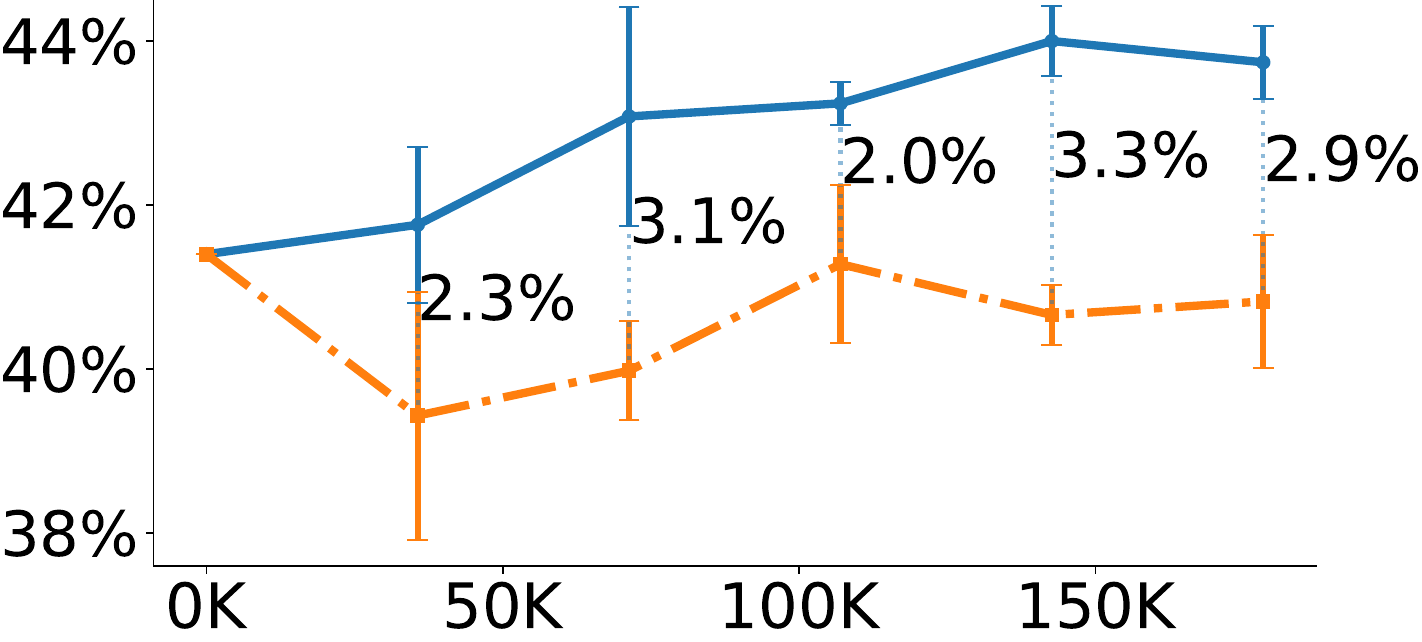}
        \caption[]%
        {Accuracy on the top 20 noisiest classes}   
        \label{fig_resnet18_top20_noise}
    \end{subfigure}
    \vskip\baselineskip
    \begin{subfigure}[b]{0.475\textwidth}   
        \centering 
        \includegraphics[width=\textwidth]{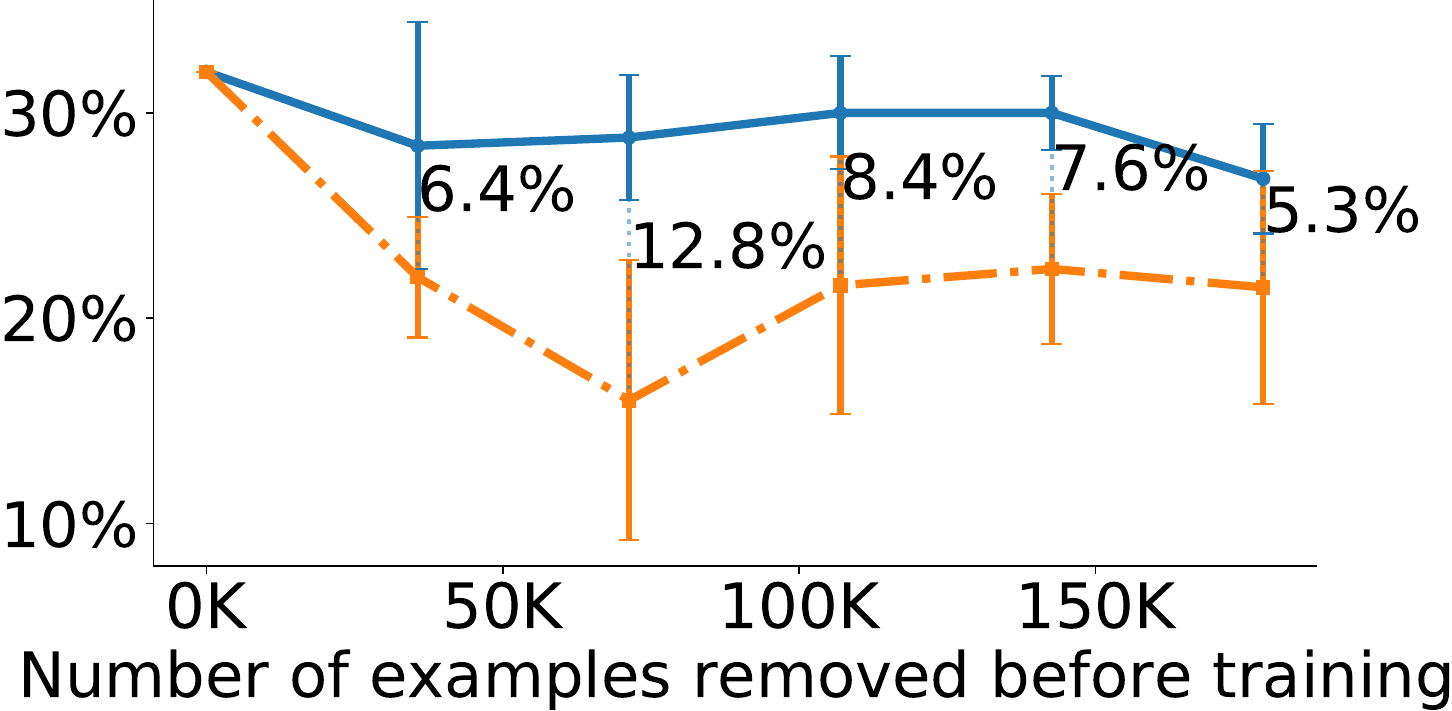}
        \caption[]%
        {Accuracy on the noisiest class: foxhound}
        \label{fig_resnet18_foxhound}
    \end{subfigure}
    \quad
    \begin{subfigure}[b]{0.475\textwidth}   
        \centering 
        \includegraphics[width=\textwidth]{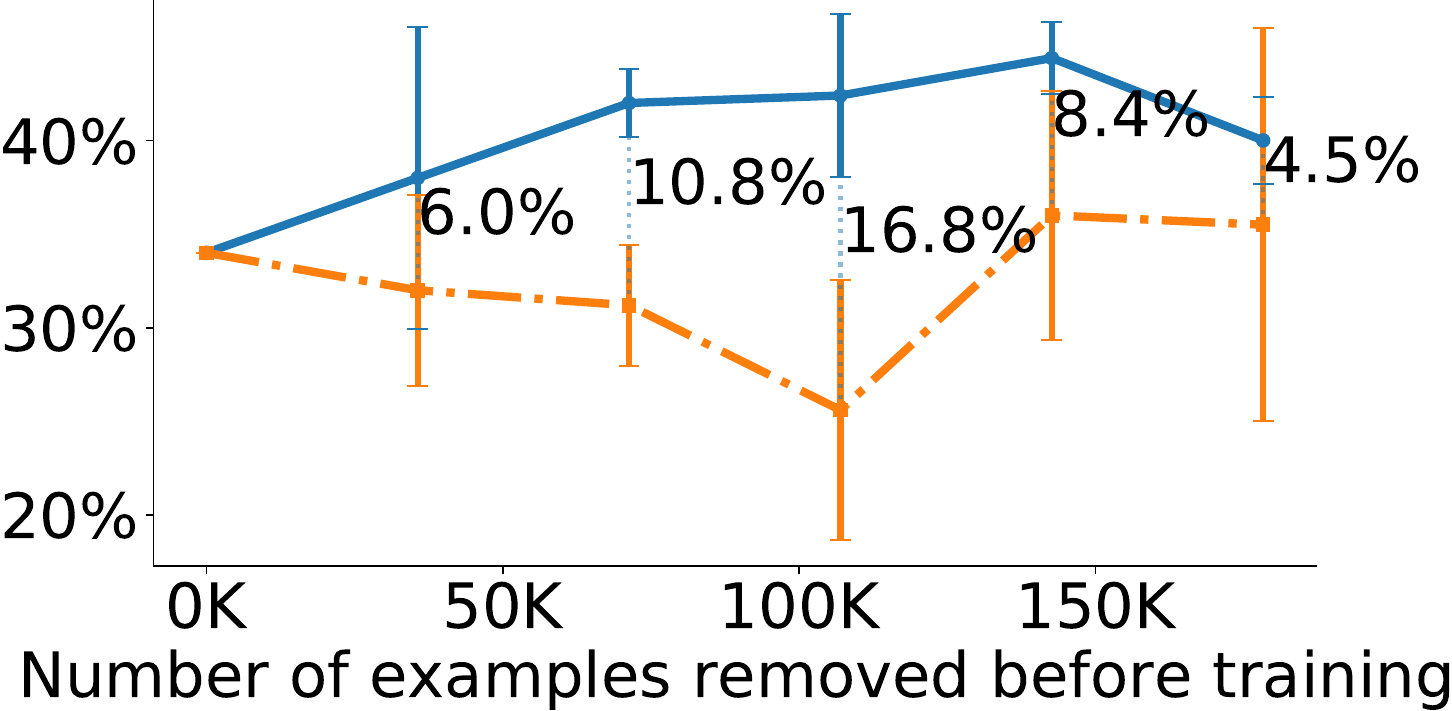}
        \caption[]%
        {Accuracy on known erroneous class: maillot} 
        \label{fig_resnet18_maillot}
    \end{subfigure}
    \caption{ResNet-18 Validation Accuracy on ImageNet (ILSVRC2012) when 20\%, 40\%, ..., 100\% of the label issues found using confident learning are removed prior to training (blue, solid line) compared with random examples removed prior to training (orange, dash-dotted line). Each subplot is read from left-to-right as incrementally more CL-identified issues are removed prior to training (shown by the x-axis). The translucent black dotted verticals bars measure the improvement when removing examples with CL vs random examples. Each point in all subfigures represents an independent training of ResNet-18 from scratch. Each point on the graph depicts the average accuracy of 5 trials (varying random seeding and weight initialization). The capped, colored vertical bars depict the standard deviation. } 
    \label{fig_imagenet_resnet18_training}
\end{figure*}

\paragraph{Finding label issues}
Fig. \ref{label_errors_imagenet_train_top32} depicts the top 16 label issues found using CL: PBNR with ResNet50 ordered by the normalized margin. We use the term \emph{issue} versus \emph{error} because examples found by CL consist of a mixture of multi-label images, ontological issues, and actual label errors. Examples of each are indicated by colored borders in the figure.
To evaluate CL in the absence of true labels, we conducted a small-scale human validation on a random sample of 500 errors (as identified using CL: PBNR) and found 58\% were either multi-label, ontological issues, or errors. ImageNet data are often presumed error-free, yet ours is the first attempt to identify label errors automatically in ImageNet training images.


\begin{figure*}[hbt!]
    \centering
    \begin{subfigure}[b]{0.475\textwidth}
        \centering
        \includegraphics[width=\linewidth]{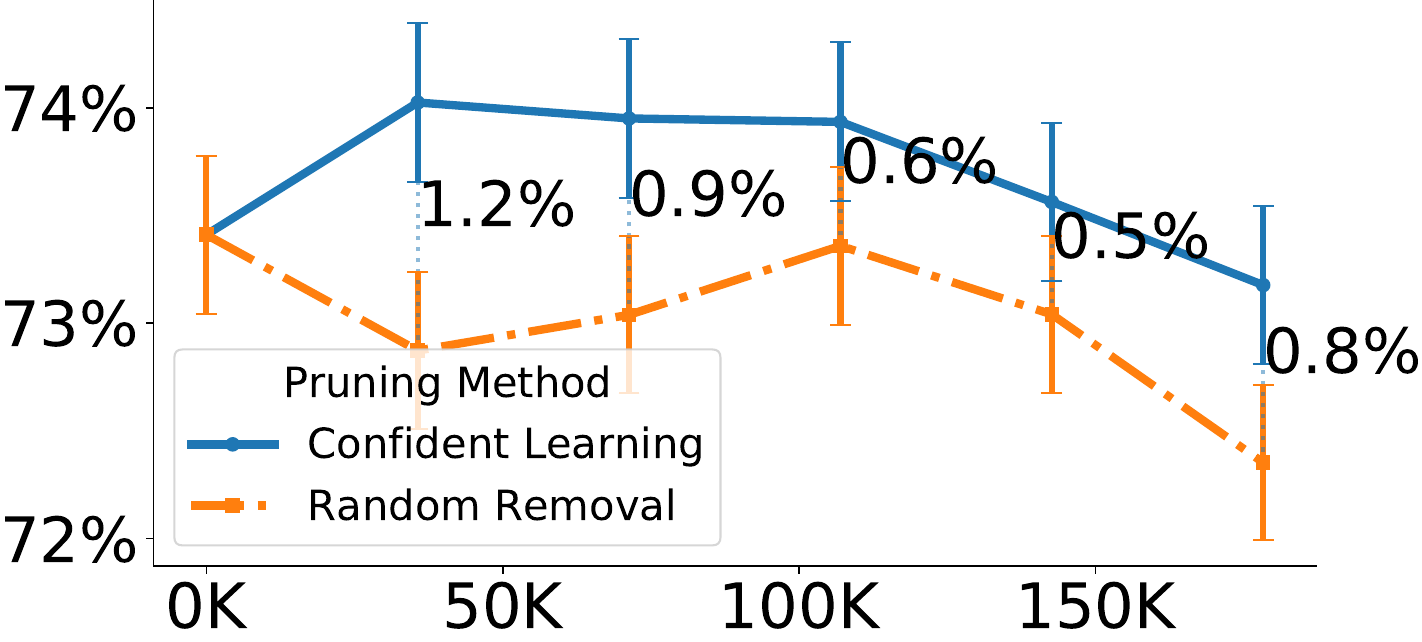}
        \caption[Network2]%
        {Accuracy on the ILSVRC2012 validation set}    
        \label{fig_resnet50_val_acc}
    \end{subfigure}
    \hfill
    \begin{subfigure}[b]{0.475\textwidth}  
        \centering 
        \includegraphics[width=\linewidth]{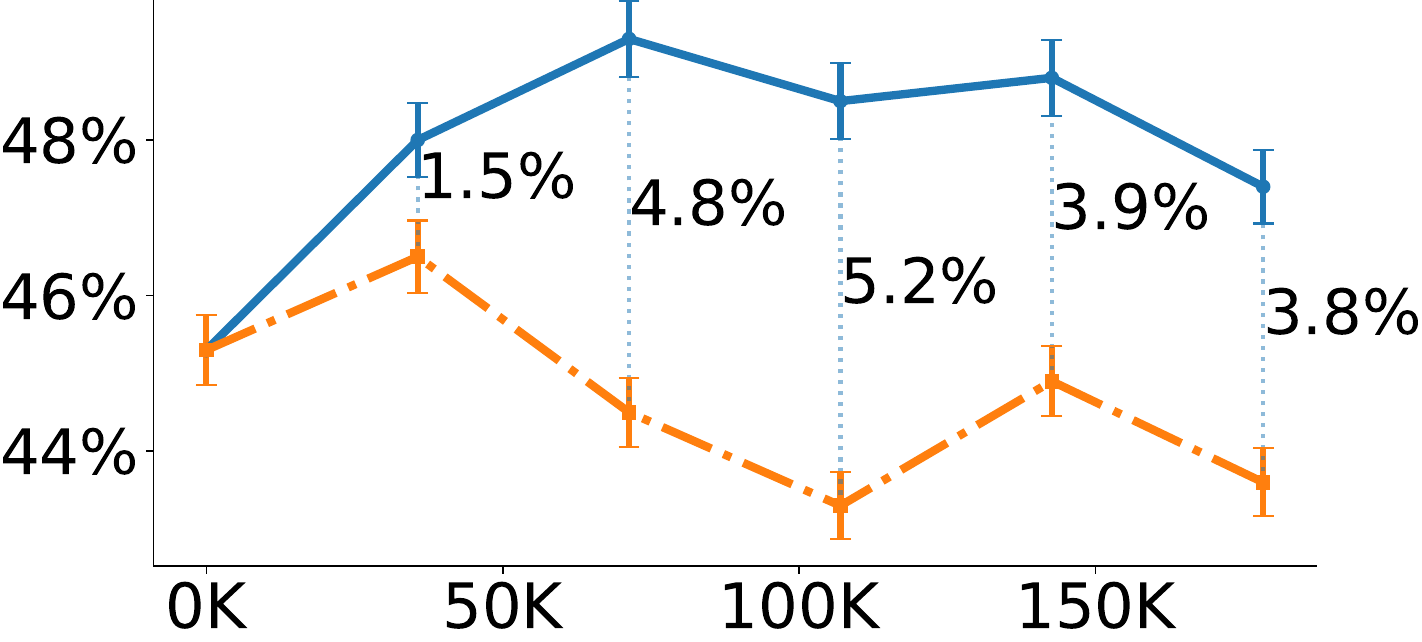}
        \caption[]%
        {Accuracy on the top 20 noisiest classes}   
        \label{fig_resnet50_top20_noise}
    \end{subfigure}
    \vskip\baselineskip
    \begin{subfigure}[b]{0.475\textwidth}   
        \centering 
        \includegraphics[width=\linewidth]{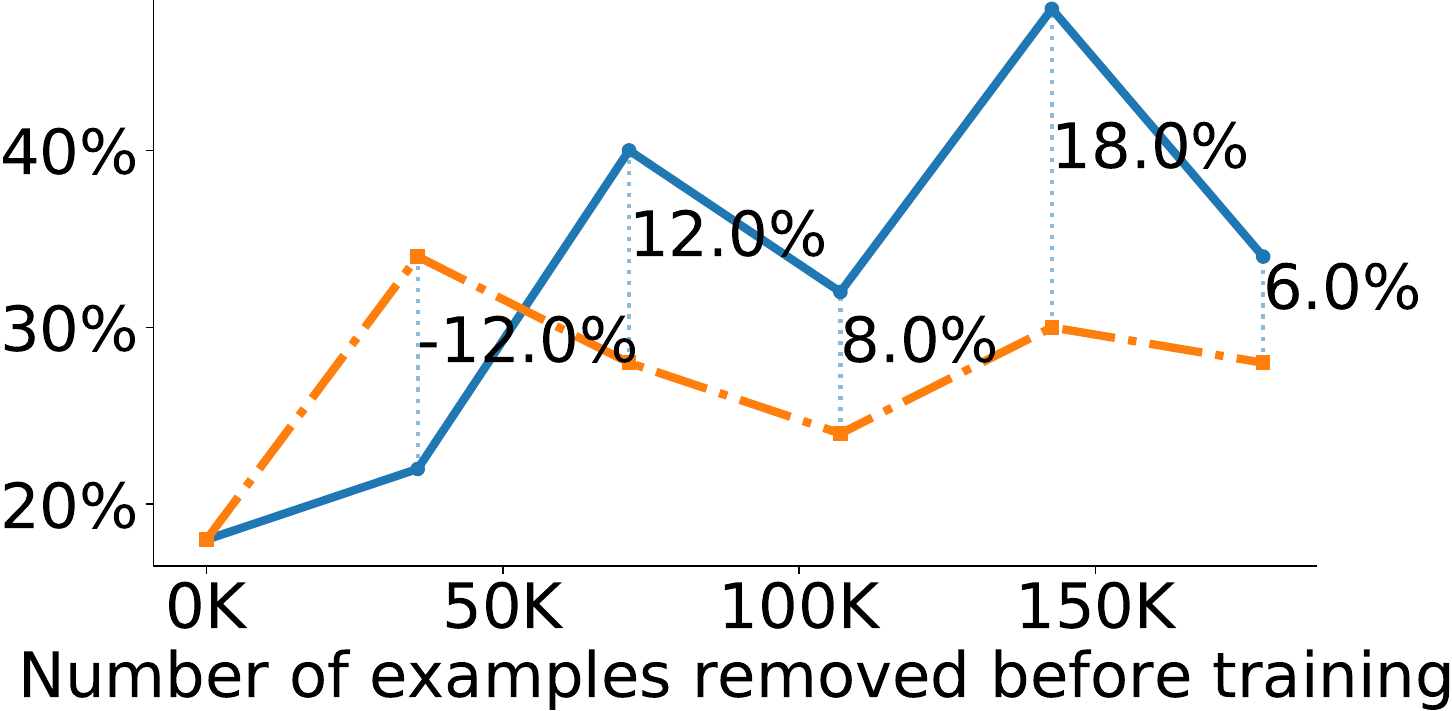}
        \caption[]%
        {Accuracy on the noisiest class: foxhound}
        \label{fig_resnet50_foxhound}
    \end{subfigure}
    \quad
    \begin{subfigure}[b]{0.475\textwidth}   
        \centering 
        \includegraphics[width=\linewidth]{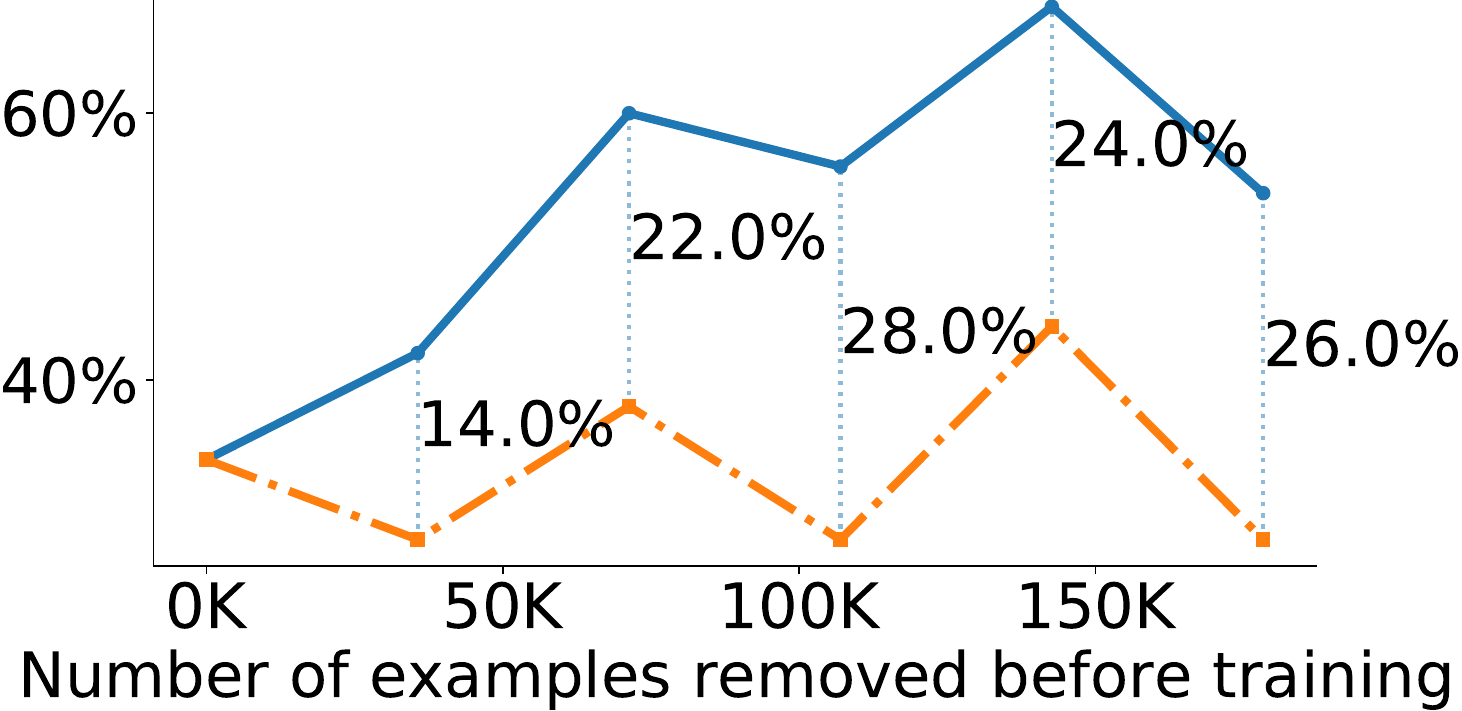}
        \caption[]%
        {Accuracy on known erroneous class: maillot} 
        \label{fig_resnet50_maillot}
    \end{subfigure}
    \caption{Replication of the experiments in Fig. \ref{fig_imagenet_resnet18_training} with ResNet-50. Each point in each subfigure depicts the accuracy of a single trial (due to computational limitations). Error bars, shown by the colored vertical lines, are estimated via Clopper-Pearson intervals for subfigures (a) and (b). For additional information, see the caption of Fig. \ref{fig_imagenet_resnet18_training}.} 
    \label{fig_imagenet_training}
\end{figure*}

\paragraph{Training ResNet on ImageNet with label issues removed}
By providing cleaned data for training, we explore how CL can be used to achieve similar or better validation accuracy on ImageNet when trained with less data.
To understand the performance differences, we train ResNet-18 (Fig. \ref{fig_imagenet_resnet18_training}) on progressively less data, removing 20\%, 40\%,..., 100\% of ImageNet train set label issues identified by CL and training from scratch each time. Fig. \ref{fig_imagenet_resnet18_training} depicts the top-1 validation accuracy when training with cleaned data from CL versus removing uniformly random examples, on each of (a) the entire ILSVRC validation set, (b) the 20 (noisiest) classes with the smallest diagonal in $\cj$, (c) the foxhound class, which has the smallest diagonal in $\cj$, and (d) the maillot class, a known erroneous class, duplicated accidentally in ImageNet, as previously published \citep{hoffman2015detector}, and verified (c.f. line 7 in Table \ref{table_imagenet_characterization_top10}). For readability, we plot the best performing CL method at each point and provide the individual performance of each CL method in the Appendix (see Fig. \ref{fig:imagenet_learning_with_noisy_labels}). For the case of a single class, as shown in Fig. \ref{fig_imagenet_resnet18_training}(c) and \ref{fig_imagenet_resnet18_training}(d), we show the recall using the model's top-1 prediction, hence the comparatively larger variance in classification accuracy reported compared to (a) and (b). We observed that CL outperforms the random removal baseline in nearly all experiments, and improves on the no-data-removal baseline accuracy, depicted by the left-most point in the subfigures, on average over the five trials for the 1,000 and 20 class settings, as shown in  Fig. \ref{fig_imagenet_resnet18_training}(a) and \ref{fig_imagenet_resnet18_training}(b). To verify the result is not model-specific, we repeat each experiment for a single trial with ResNet-50 (Fig. \ref{fig_imagenet_training}) and find that CL similarly outperforms the random removal baseline. 


These results suggest that CL can reduce the size of a real-world noisy training dataset by 10\% while still moderately improving the validation accuracy (Figures \ref{fig_resnet18_val_acc}, \ref{fig_resnet18_top20_noise}, \ref{fig_resnet50_val_acc}, \ref{fig_resnet50_top20_noise}) and significantly improving the validation accuracy on the erroneous maillot class (Figures \ref{fig_resnet18_maillot}, \ref{fig_resnet50_maillot}). While we find CL methods may improve the standard ImageNet training on clean training data by filtering out a subset of training examples, the significance of this result lies not in the magnitude of improvement, but as a warrant of exploration in the use of cleaning methods when training with ImageNet, which is typically assumed to have correct labels. Whereas many of the label issues in ImageNet are due to multi-labeled examples \citep{yun2021relabelimagenet}, next we consider a dataset with disjoint classes.

\subsection{Amazon Reviews Dataset: CL using logistic regression on noisy text data}

The Amazon Reviews dataset is a corpus of textual reviews labeled with 1-star to 5-star ratings from Amazon customers used to benchmark sentiment analysis models~\citep{he2016ups_amazonreviews}. We study the 5-core (9.9 GB) variant of the dataset -- the subset of data in which all users and items have at least 5 reviews. 2-star and 4-star reviews are removed due to ambiguity with 1-star and 5-star reviews, respectively. Left in the dataset, 2-star and 4-star reviews could inflate error counts, making CL appear to be more effective than it is.

This subsection serves three goals. First, we use a logistic regression classifier, as opposed to a deep-learning model, for our experiments in this section to evaluate CL for non-deep-learning methods. Second, we seek to understand how CL may improve learning with noise in the label space of text data, but not noise in the text data itself (e.g. typos). Towards this goal, we consider non-empty reviews with more ``helpful'' up-votes than down-votes -- the resulting dataset consists of approximately ten million reviews. Finally, Theorem \ref{thm:robustness} shows that CL is robust to class-imbalance, but datasets like ImageNet and CIFAR-10 are balanced by construction: the Amazon Reviews dataset, however, is naturally and extremely imbalanced -- the distribution of given labels (i.e., the noisy prior), is: 9\% 1-star reviews:, 12\% 3-star reviews, and 79\% 5-star reviews.  We seek to understand if CL can find label errors and improve performance in learning with noisy labels in this class-imbalanced setting.

\paragraph{Training settings} \label{amazon_reviews_training_settings} To demonstrate that non-deep-learning methods can be effective in finding label issues under the CL framework, we use a multinomial logistic regression classifier for both finding label errors and learning with noisy labels. The built-in SGD optimizer in the open-sourced fastText library \citep{fasttext_bag_of_tricks} is used with settings: initial learning rate = 0.1, embedding dimension = 100, and n-gram = 3). Out-of-sample predicted probabilities are obtained via 5-fold cross-validation. For input during training, a review is represented as the mean of pre-trained, tri-gram, word-level fastText embeddings \citep{bojanowski2017_fasttext_word_representations}.

\paragraph{Finding label issues} Table \ref{amazon} shows examples of label issues in the Amazon Reviews dataset found automatically using the CL: C+NR variant of confident learning. We observe qualitatively that most label issues identified by CL in this context are reasonable except for sarcastic reviews, which appear to be poorly modeled by the bag-of-words approach.

\begin{table}[ht]
\renewcommand{\arraystretch}{0.8}
\caption{Top 20 CL-identified label issues in the Amazon Reviews text dataset using CL: C+NR, ordered by normalized margin. A logistic regression classifier trained on fastText embeddings is used to obtain out-of-sample predicted probabilities. Most errors are reasonable, with the exception of sarcastic reviews, which are poorly modeled by the bag-of-words model. }
\vskip -0.15in
\label{amazon}
\begin{center}
\resizebox{.99\textwidth}{!}{ 

\begin{tabular}{rcc}
\toprule
\textbf{Review} & \textbf{Given Label} & \textbf{CL Guess} \\
\midrule
                                                     A very good addition to kindle. Cleans and scans. Very easy  TO USE &        $\star$ &     $\star\star\star\star\star$ \\
                                                                                         Buy it and enjoy a great story. &        $\star\star\star$ &     $\star\star\star\star\star$ \\
               Works great! I highly recommend it to everyone that enjoys singing hymns! Love it! Love it! Love it! :) . &        $\star\star\star$ &     $\star\star\star\star\star$ \\
                Awesome it was better than all the other my weirder school books.  I love it! The best book ever.Awesome &        $\star$ &     $\star\star\star\star\star$ \\
     I gave this 5 stars under duress.  I would rather give it 3 stars.  it plays fine but it is a little boring so far. &        $\star\star\star\star\star$ &     $\star\star\star$ \\
                                                                          only six words: don't waist your money on this &        $\star\star\star\star\star$ &     $\star$ \\
                        I love it so much at first I though it would be boring but turns out its fun for all ages get it &        $\star$ &     $\star\star\star\star\star$ \\
         Excellent read, could not put it down! Keep up the great works ms. Brown. Cannot wait to download the next one. &        $\star$ &     $\star\star\star\star\star$ \\
                         This is one of the easiest to use games I have ever played. It is adaptable and fun. I love it. &        $\star$ &     $\star\star\star\star\star$ \\
                                                                               So this is what today's music has become? &        $\star$ &     $\star\star\star\star\star$ \\
           Sarah and Charlie, what a wonderful story. I loved this book and look forward to reading more of this series. &        $\star\star\star$ &     $\star\star\star\star\star$ \\
                              I've had this for over a year and it works very well.  I am very happy with this purchase. &        $\star$ &     $\star\star\star\star\star$ \\
                                               this show is insane and I love it. I will be ordering more seasons of it. &        $\star\star\star$ &     $\star\star\star\star\star$ \\
                                                                            Just what the world needs, more generic r\&b. &        $\star$ &     $\star\star\star\star\star$ \\
                                    I did like the Making Of This Is movie it  okay it not the best okay it not great  . &        $\star$ &     $\star\star\star$ \\
                                                        Tough game. But of course it has the very best sound track ever! &        $\star$ &     $\star\star\star\star\star$ \\
                                                                           unexpected kid on the way thanks to this shit &        $\star$ &     $\star\star\star\star\star$ \\
                                   The kids are fascinated by it, Plus my wife loves it.. I love it I love it we love it &        $\star\star\star$ &     $\star\star\star\star\star$ \\
 Loved this book! A great story and insight into the time period and life during those times. Highly recommend this book &        $\star\star\star$ &     $\star\star\star\star\star$ \\
     Great reading I could not put it down. Highly recommend reading this book. You will not be disappointed. Must read. &        $\star\star\star$ &     $\star\star\star\star\star$ \\
\bottomrule
\end{tabular}

}

\end{center}
\end{table}

\begin{table}[ht]
\renewcommand{\arraystretch}{0.8}
\caption{Ablation study (varying train set size, test split, and epochs) comparing test accuracy (\%) of CL methods versus a standard training baseline for classifying noisy, real-world Amazon reviews text data as either 1-star, 3-stars, or 5-stars. A simple multinomial logistic regression classifier is used. Mean top-1 accuracy and standard deviations are reported over five trials. The number of estimated label errors CL methods removed prior to training is shown in the ``Pruned'' column. Baseline training begins to overfit to noise with additional epochs trained, whereas CL test accuracy continues to increase \emph{(cf. N=1000K, Epochs: 50)}.}
\vskip -0.15in
\label{amazon_benchmarks}
\begin{center}
\resizebox{.99\textwidth}{!}{ 

\begin{tabular}{ll|llll|lll}
\toprule
  Test & \multicolumn{1}{r|}{Train set size}  & \multicolumn{4}{c|}{$N = 1000K$ } & \multicolumn{3}{c}{$N = 500K$} \\
\multicolumn{2}{l|}{}    & Epochs: 5 & Epochs: 20 & Epochs: 50 & Pruned & Epochs: 5 & Epochs: 20 & Pruned \\
\midrule
10th & CL: $\bm{C}_{\text{confusion}}$ &   85.2$\pm$0.06  &   89.2$\pm$0.02 &  90.0$\pm$0.02 & 291K &  86.6$\pm$0.03 &   86.6$\pm$0.03 &   259K \\
   & CL: C+NR                        &   86.3$\pm$0.04  &   \textbf{89.8$\pm$0.01} &  \textbf{90.2$\pm$0.01} & 250K &  \textbf{87.5$\pm$0.05} &   \textbf{87.5$\pm$0.03} &   244K \\
   & CL: $\cj$                       &   \textbf{86.4$\pm$0.01}  &   \textbf{89.8$\pm$0.02} &  90.1$\pm$0.02 & 246K &  \textbf{87.5$\pm$0.02} &   \textbf{87.5$\pm$0.02} &   243K \\
   & CL: PBC                         &   86.2$\pm$0.03  &   89.7$\pm$0.01 & \textbf{90.2$\pm$0.01} &  257K &  87.4$\pm$0.03 &   87.4$\pm$0.03 &   247K \\
   & CL: PBNR                        &   86.2$\pm$0.07  &   89.7$\pm$0.01 &  \textbf{90.2$\pm$0.01} & 257K &  87.4$\pm$0.05 &   87.4$\pm$0.05 &   247K \\
   & Baseline                        &   83.9$\pm$0.11  &   86.3$\pm$0.06 &  84.4$\pm$0.04 &  0K &  82.7$\pm$0.07 &   82.8$\pm$0.07 &     0K \\
\midrule
11th & CL: $\bm{C}_{\text{confusion}}$ &   85.3$\pm$0.05 &   89.3$\pm$0.01 &  90.0$\pm$0.0 & 294K &  86.6$\pm$0.04 &   86.6$\pm$0.06 &   261K \\
   & CL: C+NR                        &   \textbf{86.4$\pm$0.06} &   \textbf{89.8$\pm$0.01} &  90.2$\pm$0.01 & 252K &  \textbf{87.5$\pm$0.04} &   \textbf{87.5$\pm$0.03} &   247K \\
   & CL: $\cj$                       &   86.3$\pm$0.05 &   \textbf{89.8$\pm$0.01} &  90.1$\pm$0.02 & 249K &  \textbf{87.5$\pm$0.03} &   \textbf{87.5$\pm$0.02} &   246K \\
   & CL: PBC                         &   86.2$\pm$0.03 &   \textbf{89.8$\pm$0.01} &  \textbf{90.3$\pm$0.0} & 260K &  87.4$\pm$0.03 &   87.4$\pm$0.05 &   250K \\
   & CL: PBNR                        &   86.2$\pm$0.06 &   \textbf{89.8$\pm$0.01} &  90.2$\pm$0.02 & 260K &  87.4$\pm$0.05 &   87.4$\pm$0.03 &   249K \\
   & Baseline                        &   83.9$\pm$0.0  &   86.3$\pm$0.05 &   84.4$\pm$0.12 & 0K &  82.7$\pm$0.04 &   82.7$\pm$0.09 &     0K \\
\bottomrule
\end{tabular}

}

\end{center}
\end{table}

\paragraph{Learning with noisy labels / weak supervision} We compare the CL methods, which prune errors from the train set and subsequently provide clean data for training, versus a standard training baseline (denoted \emph{Baseline} in Table \ref{amazon_benchmarks}), which trains on the original, uncleaned train dataset. The same training settings used to find label errors (see Subsection \ref{amazon_reviews_training_settings}) are used to obtain all scores reported in Table \ref{amazon_benchmarks} for all methods. For a fair comparison, all mean accuracies in Table \ref{amazon_benchmarks} are reported on the same held-out test set, created by splitting the Amazon reviews dataset into a train set and test set such that every tenth example is placed in a test set and the remaining data is available for training (the Amazon Reviews 5-core dataset provides no explicit train set and test set).

The Amazon Reviews dataset is naturally noisy, but the fraction of noise in the dataset is estimated to be less than 4\% \citep{northcutt2021labelerrors}, which makes studying the benefits of providing clean data for training challenging. To increase the percentage of noisy labels without adding synthetic noise, we subsample 1 million training examples from the train set by combining the label issues identified by all five CL methods from the original training data (244K examples) and a uniformly random subsample (766k examples) of the remaining ``cleaner'' training data. This process increases the percentage of label noise to 24\% (estimated) in the train set and, importantly, does \emph{not} increase the percentage of noisy labels in the test set -- large amounts of test set label noise have been shown to severely impact benchmark rankings \citep{northcutt2021labelerrors}.

To mitigate the bias induced by the choice of train set size, test set split, and the number of epochs trained, we conduct an ablation study shown in Table \ref{amazon_benchmarks}. For the train set size, we repeat each experiment with train set sizes of 1-million examples and $500,000$ examples. For the test set split, we repeat all experiments by removing every \emph{eleventh} example (instead of tenth) in our train/test split (c.f. the first column in Table \ref{amazon_benchmarks}), minimizing the overlap (9\%) between the two test sets. For each number of epochs trained, we repeat each experiment with 5, 20, and 50 epochs. We omit ($N = 500K$, Epochs: 50) because no learning occurs after 5 epochs.

Every score reported in Table \ref{amazon_benchmarks} is the mean and standard deviation of five trials: each trial varies the randomly selected subset of training data and the initial weights of the logistic regression model used for training. 

The results in Table \ref{amazon_benchmarks} reveal three notable observations. First, all CL methods outperform the baseline method by a significant margin in all cases. Second, CL methods outperform the baseline method even with nearly half of the training data pruned (Table \ref{amazon_benchmarks}, cf. N=500K). Finally, for the train set size $N = 1000K$, baseline training begins to overfit to noise with additional epochs trained, whereas CL test accuracy continues to increase \emph{(cf. N=1000K, Epochs: 50)}, suggesting CL robustness to overfitting to noise during training. The results in Table \ref{amazon_benchmarks} suggest CL's efficacy for noisy supervision with logistic regression in the context of text data.

\subsection{Real-world Label Errors in Other Datasets}

We use CL to find label errors in the purported ``error-free" MNIST dataset comprised of preprocessed black-and-white handwritten digits, and also in the noisy-labeled WebVision dataset \citep{li2017webvision} comprised of color images collected from online image repositories and using the search query as the noisy label.

\begin{figure*}[ht]
\centerline{\includegraphics[width=\textwidth]{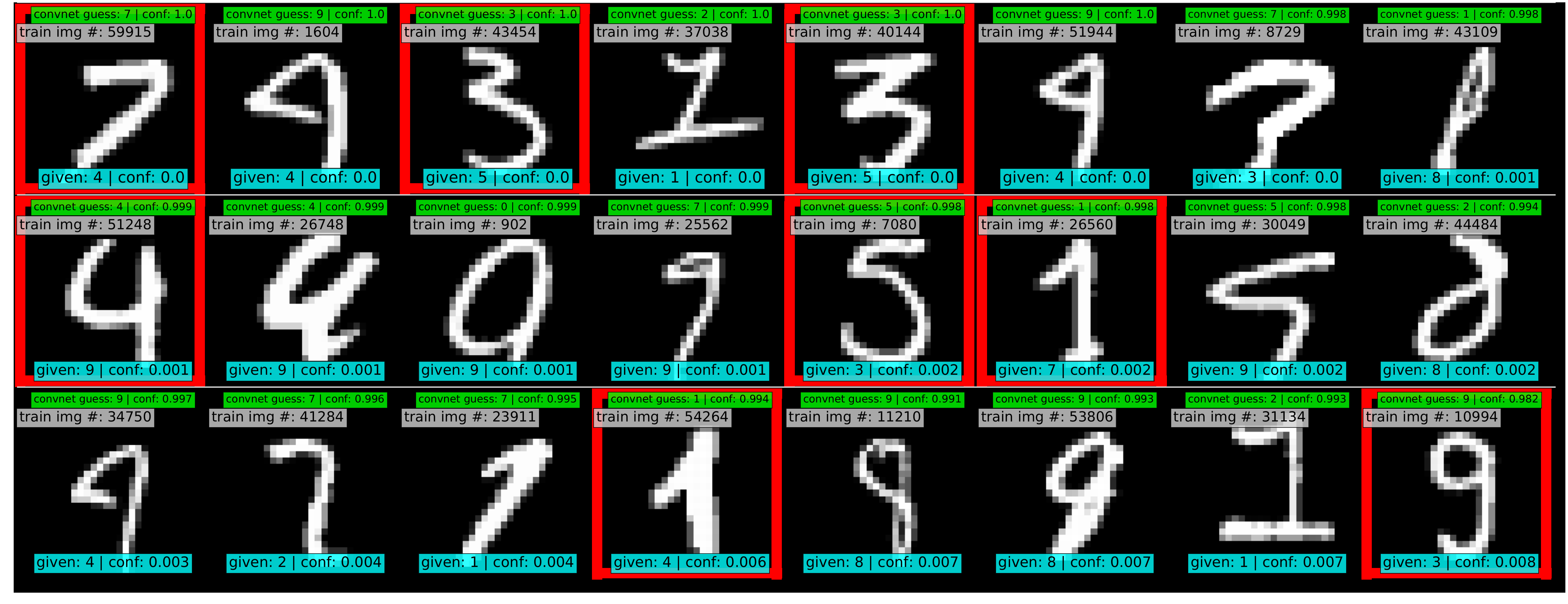}}
\vskip -0.1in
\caption{Label errors in the original, unperturbed MNIST train dataset identified using CL: PBNR. These are the top 24 errors found by CL, ordered left-right, top-down by increasing self-confidence, denoted \emph{conf} in teal. The predicted $\argmax \hat{p}(\tilde{y}=k ; x, \model )$ label is in green. Overt errors are in red. This dataset is assumed ``error-free'' in tens of thousands of studies.} 
\label{mnist_training_label_errors24_prune_by_noise_rate}
\vskip -.2in
\end{figure*}

\begin{figure*}[!b]
\centerline{\includegraphics[width=\textwidth]{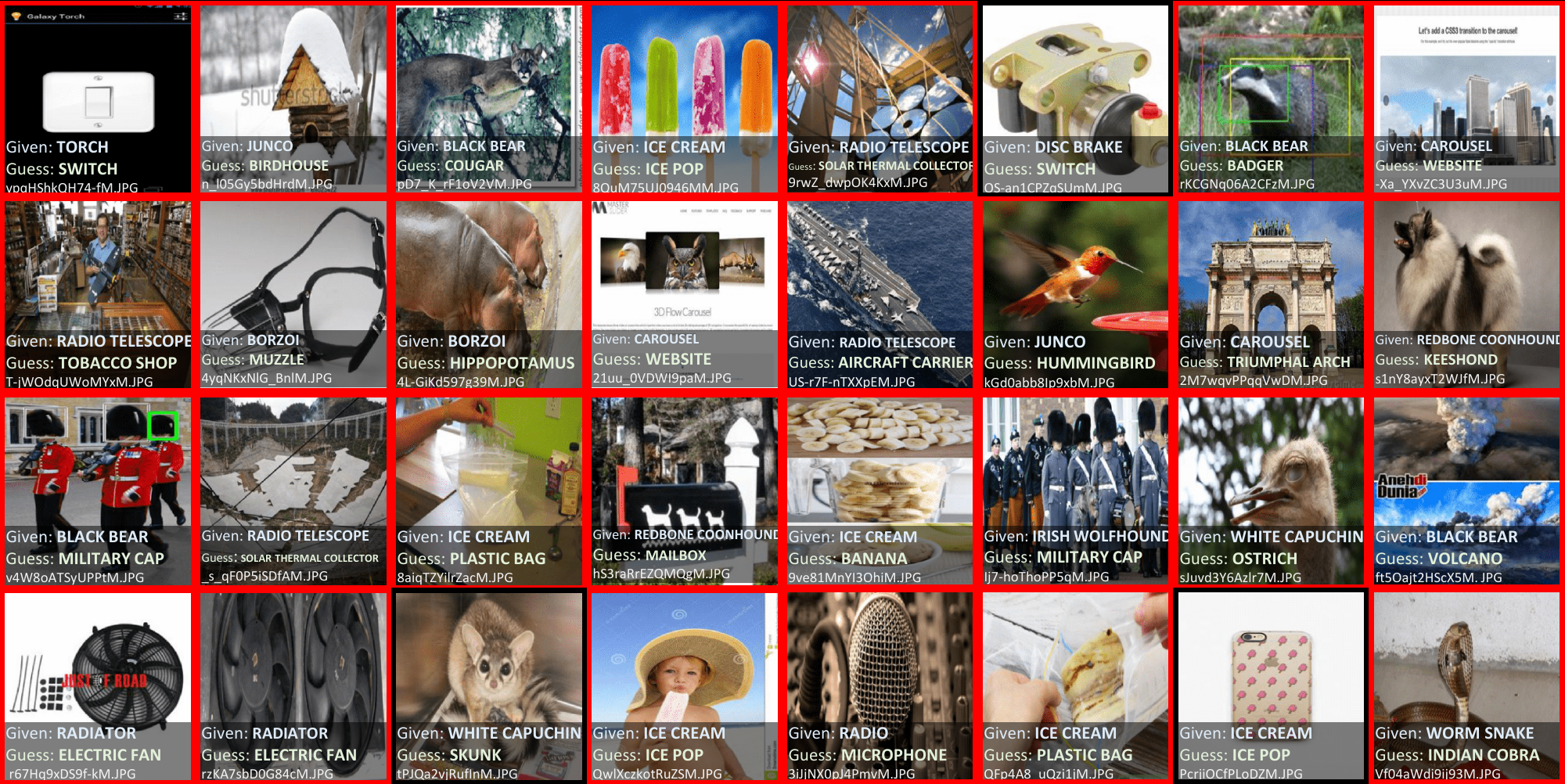}}
\caption{Top 32 identified label issues in the WebVision train set using CL: $\cj$. Out-of-sample predicted probabilities are obtained using a model pre-trained on ImageNet, avoiding training entirely. Errors are boxed in red. Ambiguous cases or mistakes are boxed in black. Label errors are ordered automatically by normalized margin.} 
\label{label_errors_webvision_train_top32}
\end{figure*}

To our surprise, the original, unperturbed MNIST dataset, which is predominately assumed error-free, contains blatant label errors, highlighted by the red boxes in Fig. \ref{mnist_training_label_errors24_prune_by_noise_rate}. To find label errors in MNIST, we pre-trained a simple 2-layer CNN for 50 epochs, then used cross-validation to obtain $\probmatrix$, the out-of-sample predicted probabilities for the train set. CL: PBNR was used to identify the errors. The top 24 label errors, ordered by self-confidence, are shown in Fig. \ref{mnist_training_label_errors24_prune_by_noise_rate}. For verification, the indices of the train label errors are shown in grey.

To find label errors in WebVision, we used a pre-trained model to obtain $\probmatrix$, observing two practical advantages of CL: (1) a pre-trained model can be used to obtain $\probmatrix$ out-of-sample instead of cross-validation and (2) this makes CL fast. For example, finding label errors in WebVision, with over a million images and 1,000 classes, took three minutes on a laptop using a pre-trained ResNext model that had never seen the noisy WebVision train set before. We used the CL: $\cj$ method to find the label errors and ordered errors by normalized margins. Examples of WebVision label errors found by CL are shown in Fig. \ref{label_errors_webvision_train_top32}.

\section{Related work}
\label{sec:related}


We first discuss prior work on confident learning, then review how CL relates to noise estimation and robust learning.

\paragraph{Confident learning} \; Our results build on a large body of work termed ``confident learning''. \citet{elkan_cost_sensitive_learning_thresholding} and \citet{forman2005counting} pioneered counting approaches to estimate false positive and false negative rates for binary classification. We extend counting principles to multi-class setting. To increase robustness against epistemic error in predicted probabilities and class imbalance, \citet{Elkan:2008:LCO:1401890.1401920} introduced thresholding, but required uncorrupted positive labels. CL generalizes the use of thresholds to multi-class noisy labels. CL also reweights the loss during training to adjust priors for the data removed. This choice builds on formative works \citep{NIPS2013_5073, rooyen_menon_unhinged_nips15} which used loss reweighting to prove equivalent empirical risk minimization for learning with noisy labels. 
More recently, \citet{deep_self_learning_arxiv_2019} proposed an empirical deep self-supervised learning approach to avoid probabilities by using embedding layers of a neural network. In comparison, CL is non-iterative and theoretically grounded. \citet{icml_lipton_label_shift_confusion_matrix} estimate label noise using approaches based on confusion matrices and cross-validation. However, unlike CL, the former assumes a less general form of label shift than class-conditional noise. \cite{o2unet_iccv} demonstrate the empirical efficacy of first finding label errors, then training on clean data, but the study evaluates only uniform (symmetric) and pair label noise -- CL augments these empirical findings with theoretical justification for the broader class of asymmetric and class-conditional label noise.


\paragraph{Theory: a model-free, data-free approach} Theoretical analysis with noisy labels often assumes a restricted class of models or data to disambiguate model noise from label noise. For example, \cite{ICML_iterative_trimmed_loss} provide theoretical guarantees for learning with noisy labels in a more general setting than CL that includes adversarial examples and noisy data, but limit their findings to generalized linear models. CL theory is model and dataset agnostic, instead restricting the magnitude of example-level noise. In a formative related approach, \cite{neurips2019novelinformationtheory} prove that using the loss function $-\log \left(\vert \det (\joint) \right) \rvert$ enables noise robust training for any model and dataset, further justified by performant empirical results. Similar to confident learning, their approach hinges on the use of $\joint$, however, they require that $\nm$ is invertible and estimate $\joint$ using $\bm{C}_{\text{confusion}}$, which is sensitive to class-imbalance and heterogeneous class probability distributions (see Sec. \ref{sec:label_noise_characterization}). In Sec. \ref{sec:theory}, we show sufficient conditions in Thm. \ref{thm:robustness} where $\cj$ exactly finds label errors, regardless of each class's probability distribution.

\paragraph{Uncertainty quantification and label noise estimation} \; A number of formative works developed solutions to estimate noise rates using convergence criterion \citep{scott2015rate}, positive-unlabeled learning \citep{Elkan:2008:LCO:1401890.1401920}, and predicted probability ratios \citep{northcutt2017rankpruning}, but are limited to binary classification. Others prove equivalent empirical risk for \emph{binary} learning with noisy labels \citep{NIPS2013_5073,Liu:2016:CNL:2914183.2914328,Sugiyama:2012:DRE:2181148} assuming noise rates are known, which is rarely true in practice. Unlike these binary approaches, CL estimates label uncertainty in the multiclass setting, where prior work often falls into five categories: (1) theoretical contributions \citep{scott19_jmlr_mutual_decontamination}, (2) loss modification for label noise robustness \citep{DBLP:conf/icml/PatriniNNC16, patrini2017making, Sukhbaatar_fergus_iclr_2015, rooyen_menon_unhinged_nips15}, (3) deep learning and model-specific approaches \citep{Sukhbaatar_fergus_iclr_2015, DBLP:conf/icml/PatriniNNC16, 7837934_multiclass_learning_with_noise_using_dropout}, (4) crowd-sourced labels via multiple workers \citep{zhang2017improving, dawid1979maximum, ratner_chris_re_large_training_sets_NIPS2016_6523}, (5) factorization, distillation \citep{Li_distillation2017}, and imputation \citep{amjad_devavrat_svd_imputation_2017} methods, among other \citep{Saez2014_multiclass_learning_by_binary_subproblem}. 
Unlike these approaches, CL provides a consistent estimator for exact estimation of the joint distribution of noisy and true labels directly, under practical conditions.

\paragraph{Label-noise robust learning} \; Beyond the above noise estimation approaches, extensive studies have investigated training models on noisy datasets, e.g. \citep{beigman2009learning,brodley1999identifying}. Noise-robust learning is important for deep learning because modern neural networks trained on noisy labels generalize poorly on clean validation data~\citep{zhang2017understanding}. A notable recent rend in noise robust learning is benchmarking with symmetric label noise in which labels are uniformly flipped, e.g.~\citep{DBLP:conf/iclr/GoldbergerB17_smodel,arazo2019unsupervised}. However, noise in real-world datasets is highly non-uniform and often sparse. For example, in ImageNet \citep{ILSVRC15_imagenet}, \emph{missile} is likely to be mislabeled as \emph{projectile}, but has a near-zero probability of being mislabeled as most other classes like \emph{wool}, \emph{ox}, or \emph{wine}. To approximate real-world noise, an increasing number of studies examined asymmetric noise using, e.g. loss or label correction~\citep{patrini2017making,noisy_boostrapping_google_reed_iclr_2015,DBLP:conf/iclr/GoldbergerB17_smodel}, per-example loss reweighting~\citep{jiang2020beyond,jiang2018mentornet,shu2019meta}, Co-Teaching~\citep{han2018coteaching}, semi-supervised learning~\citep{hendrycks2018using,Li_distillation2017,vahdat2017toward}, symmetric cross entropy~\citep{wang2019sceloss_symmetric}, and semi-supervised learning~\citep{iclr_li2020dividemix}, among others. These approaches work by introducing novel new models or insightful modifications to the loss function during training. CL takes a loss-agnostic approach, instead focusing on generating clean data for training by directly estimating of the joint distribution of noisy and true labels.

\paragraph {Comparison of the INCV Method and Confident Learning} \label{sec:compare_incv_and_cl}

The INCV algorithm \citep{chen2019confusion} and confident learning both estimate clean data, use cross-validation, and use aspects of confusion matrices to deal with label errors in ML workflows. Due to these similarities, we discuss four key differences between confident learning and INCV.

First, INCV errors are found using an iterative version of the $\bm{C}_{\text{confusion}}$ confident learning baseline: any example with a different given label than its argmax prediction is considered a label error. This approach, while effective (see Table \ref{table:cifar10_benchmark}), fails to properly count errors for class imbalance or when a model is more confident (larger or smaller probabilities on average) for certain class than others, as discussed in Section \ref{sec:theory}. To account for this class-level bias in predicted probabilities and enable robustness, confident learning uses theoretically-supported (see Section \ref{sec:theory}) thresholds \citep{elkan_cost_sensitive_learning_thresholding, richard1991neural_thresholding} while estimating the confident joint. Second, a major contribution of CL is finding the label errors in the presumed error-free benchmarks such as ImageNet and MNIST, whereas INCV emphasizes empirical results for learning with noisy labels. Third, in each INCV training iteration, 2-fold cross-validation is performed. The iterative nature of INCV makes training slow (see Appendix Table \ref{table:incv_stuff}) and uses fewer data during training. Unlike INCV, confident learning is not iterative. In confident learning, cross-validated probabilities are computed only once beforehand from which the joint distribution of noisy and true labels is directly estimated which is used to identify clean data to be used by a single pass re-training. We demonstrate this approach is experimentally performant without iteration (see Table \ref{table:cifar10_benchmark}). Finally, confident learning is modular. CL approaches for training, finding label errors, and ordering label errors for removal are independent. In INCV, the procedure is iterative, and all three steps are tied together in a single looping process. A single iteration of INCV equates to the $\bm{C}_{\text{confusion}}$ baseline benchmarked in this paper.

\section{Conclusion and Future Work}
\label{sec:conclusion}

Following the principles of confident learning, we developed a novel approach to estimate the joint distribution of label noise and explicated theoretical and experimental insights into the benefits of doing so. We demonstrated accurate uncertainty quantification in high noise and and sparsity regimes, across multiple datasets, data modalities, and model architectures. We empirically evaluated three criteria: (1) uncertainty quantification via estimation of the joint distribution of label noise, (2) finding label errors, and (3) learning with noisy labels on CIFAR-10, and found that CL methods outperform recent prior art across all three. 

These findings emphasize the practical nature of confident learning, identifying numerous pre-existing label issues in ImageNet, Amazon Reviews, MNIST, and other datasets, and improving the performance of learning models like deep neural networks by training on a cleaned dataset. Confident learning motivates the need for further understanding of dataset uncertainty estimation, methods to clean training and test sets, and approaches to identify ontological and label issues for dataset curation. Future directions include validation of CL methods on more datasets such as the OpenML Benchmark~\citep{feurer2019openml}, the multi-modal Egocentric Communications (EgoCom) benchmark \citep{northcutt2020egocom}, and the realistic noisy label benchmark CNWL~\citep{jiang2020beyond}; evaluation of CL methods using other non-neural network models, such as random forests and XGBoost; examination of other threshold function formulations; examination of label errors in test sets and they affect machine learning benchmarks at scale~\citep{northcutt2021labelerrors}; assimilation of CL label error finding with pseudo-labeling and/or curriculum learning to \emph{dynamically} provide clean data during training; and further exploration of iterative and/or regression-based extensions of CL methods.


\section*{Acknowledgements}

We thank the following colleagues: Jonas Mueller assisted with notation. Anish Athayle suggested starting the proof in claim 1 of Theorem \ref{thm:exact_label_errors} with the identity. Tailin Wu contributed to Lemma \ref{lemma:ideal_threshold}. Niranjan Subrahmanya provided feedback on baselines for confident learning. 

\clearpage
\vskip 0.2in
\bibliography{paper}
\interlinepenalty=10000
\bibliographystyle{apalike}

\clearpage
\beginsupplement
\onecolumn
\appendix

\section{Theorems and proofs for confident learning} \label{sec:proofs}

In this section, we restate the main theorems for confident learning and provide their proofs.

\setcounter{lemma}{0}
\setcounter{corollary}{0}
\setcounter{theorem}{0}

\begin{lemma}[Ideal Thresholds] 
For a noisy dataset $\bm{X} \coloneqq (\bm{x}, \tilde{y})^n \in {(\mathbb{R}^d, [m])}^n$ and model $\model$, 
if $\hat{p}(\tilde{y}; \bm{x}, \model)$ is \emph{ideal}, then 
$\forall i \smallin [m], t_i = \sum_{j \smallin [m]} p(\tilde{y} = i \vert y^* \smalleq j) p(y^* \smalleq j \vert \tilde{y} = i)$.
\end{lemma}

\begin{proof} \label{proof:ideal_threshold}
We use $t_i$ to denote the thresholds used to partition $\bm{X}$ into $m$ bins, each estimating one of $\bm{X}_{y^*}$. By definition,
\begin{equation*}
   \forall i \smallin [m], t_i =  \EX_{\bm{x} \smallin \bm{X}_{\tilde{y}=i}} \hat{p}(\tilde{y} = i ; \bm{x}, \bm{\theta})
\end{equation*}
For any $t_i$, we show the following.
\begin{align*}
    t_i &=  \EXlimits_{\bm{x} \smallin \bm{X}_{\tilde{y} \smalleq i}} \sum_{j \smallin [m]} \hat{p}(\tilde{y} \smalleq i \vert y^* \smalleq j; \bm{x}, \bm{\theta}) \hat{p}(y^* \smalleq j; \bm{x}, \bm{\theta}) && \triangleright \text{Bayes Rule} \\
    t_i &=   \EXlimits_{\bm{x} \smallin \bm{X}_{\tilde{y} \smalleq i}} \sum_{j \smallin [m]} \hat{p}(\tilde{y} \smalleq i \vert y^* \smalleq j) \hat{p}(y^* \smalleq j; \bm{x}, \bm{\theta}) && \triangleright \text{Class-conditional Noise Process (CNP)} \\
    t_i &=   \sum_{j \smallin [m]} \hat{p}(\tilde{y} \smalleq i \vert y^* \smalleq j) \EXlimits_{\bm{x} \smallin \bm{X}_{\tilde{y} \smalleq i}}  \hat{p}(y^* \smalleq j; \bm{x}, \bm{\theta}) \\
    t_i &=   \sum_{j \smallin [m]} p(\tilde{y} = i \vert y^* = j) p(y^* = j \vert \tilde{y} = i) && \triangleright \text{Ideal Condition} 
\end{align*}
This form of the threshold is intuitively reasonable: the contributions to the sum when $i=j$ represents the probabilities of correct labeling, whereas when $i\neq j$, the terms give the probabilities of mislabeling $p(\tilde{y} = i \vert y^* = j)$, weighted by the probability $p(y^* = j \vert \tilde{y} = i)$ that the mislabeling is corrected.
\end{proof}

\begin{theorem}[Exact Label Errors] 
For a noisy dataset, $\bm{X} \coloneqq (\bm{x}, \tilde{y})^n \smallin {(\mathbb{R}^d, [m])}^n$ and model $\bm{\theta}{\scriptstyle:} \bm{x} \smallra \hat{p}(\tilde{y})$, 
if $\hat{p}(\tilde{y} ; \bm{x}, \bm{\theta})$ is \emph{ideal} and each diagonal entry of $\nm$ maximizes its row and column, 
then $\hat{\bm{X}}_{ \tilde{y} \smalleq i,y^* \smalleq j} = \bm{X}_{\tilde{y} \smalleq i, y^* \smalleq j}$ and $\estjoint \approxeq \joint$ (consistent estimator for $\joint$).
\end{theorem}

\begin{proof} \label{proof:exact_label_errors}
Alg. \ref{alg:cj} defines the construction of the confident joint. We consider Case 1: when there are collisions (trivial by the construction of Alg. \ref{alg:cj}) and case 2: when there are no collisions (harder).

\bigskip
\quad \emph{Case 1 (collisions)}: \\
When a collision occurs, by the construction of the confident joint (Eqn. \ref{eqn_paper_confident_joint}), a given example $\bm{x}_k$ gets assigned bijectively into bin $$\bm{x}_k \in \hat{\bm{X}}_{ \tilde{y} ,y^*}[{\tilde{y}}_k][\argmax_{i \in [m]} \hat{p}(\tilde{y} = i; \bm{x}, \bm{\theta})]$$
Because we have that $\hat{p}(\tilde{y} ; \bm{x}, \bm{\theta})$ is ideal, we can rewrite this as 
$$\bm{x}_k \in \hat{\bm{X}}_{ \tilde{y} ,y^*}[{\tilde{y}}_k][\argmax_{i \in [m]} \hat{p}(\tilde{y} = i \vert y^* \smalleq y_k^*; \bm{x})]$$
And because by assumption each diagonal entry in $\nm$ maximizes its column, we have
$$\bm{x}_k \in \hat{\bm{X}}_{ \tilde{y} ,y^*}[{\tilde{y}}_k][y_k^*]$$

\noindent Thus, any example $\bm{x} \in \bm{X}_{\tilde{y} \smalleq i, y^* \smalleq j}$ having a collision will be exactly assigned to $\hat{\bm{X}}_{ \tilde{y} \smalleq i,y^* \smalleq j} $.


\bigskip
\quad \emph{Case 2 (no collisions)}: \\



\noindent We want to show that $\forall i \smallin [m], j \smallin [m], \hat{\bm{X}}_{ \tilde{y} \smalleq i,y^* \smalleq j} = \bm{X}_{\tilde{y} \smalleq i, y^* \smalleq j} $.

\noindent We can partition $\bm{X}_{\tilde{y} \smalleq i}$ as
\begin{equation*}
    \bm{X}_{\tilde{y} \smalleq i} = \bm{X}_{\tilde{y} \smalleq i, y^* = j} \cup \bm{X}_{\tilde{y} \smalleq i, y^* \neq j}
\end{equation*}

\noindent We prove $\forall i \smallin [m], j \smallin [m], \hat{\bm{X}}_{ \tilde{y} \smalleq i,y^* \smalleq j} = \bm{X}_{\tilde{y} \smalleq i, y^* \smalleq j}$ by proving two claims:

\qquad \textbf{Claim 1}: $\bm{X}_{ \tilde{y} \smalleq i,y^* \smalleq j} \subseteq \hat{\bm{X}}_{\tilde{y} \smalleq i, y^* \smalleq j}$

\qquad \textbf{Claim 2}: $\bm{X}_{ \tilde{y} \smalleq i,y^* \neq j} \nsubseteq \hat{\bm{X}}_{\tilde{y} \smalleq i, y^* \smalleq j}$

We do not need to show $\bm{X}_{ \tilde{y} \neq i,y^* \smalleq j} \nsubseteq \hat{\bm{X}}_{\tilde{y} \smalleq i, y^* \smalleq j}$ and $\bm{X}_{ \tilde{y} \neq i,y^* \neq j} \nsubseteq \hat{\bm{X}}_{\tilde{y} \smalleq i, y^* \smalleq j}$ because the noisy labels $\tilde{y}$ are given, thus the confident joint (Eqn. \ref{eqn_paper_confident_joint}) will never place them in the wrong bin of $\hat{\bm{X}}_{ \tilde{y} \smalleq i,y^* \smalleq j}$. Thus, claim 1 and claim 2 suffice to show that $\hat{\bm{X}}_{ \tilde{y} \smalleq i,y^* \smalleq j} = \bm{X}_{\tilde{y} \smalleq i, y^* \smalleq j}$.

\paragraph{\emph{Proof (Claim 1) of Case 2}:}

Inspecting Eqn. (\ref{eqn_paper_confident_joint}) and Alg (\ref{alg:cj}), by the construction of $\cj$, we have that $\forall \bm{x} \in \bm{X}_{\tilde{y} = i}$,\;  $\hat{p} (\tilde{y} = j \vert y^* \smalleq j; \bm{x}, \bm{\theta}) \ge t_j \longrightarrow \bm{X}_{ \tilde{y} \smalleq i,y^* \smalleq j} \subseteq \hat{\bm{X}}_{\tilde{y} \smalleq i, y^* \smalleq j}$. When the left-hand side is true, all examples with noisy label $i$ and hidden, true label $j$ are counted in $\hat{\bm{X}}_{\tilde{y} \smalleq i, y^* \smalleq j}$.

Thus, it suffices to prove:

\begin{equation} \label{alg:case1_condition}
    \forall \bm{x} \in \bm{X}_{\tilde{y} = i}, \hat{p} (\tilde{y} = j \vert y^* \smalleq j; \bm{x}, \bm{\theta}) \ge t_j  
\end{equation}

Because the predicted probabilities satisfy the ideal condition, $\hat{p} (\tilde{y} = j \vert y^* \smalleq j, \bm{x}) = p (\tilde{y} = j \vert y^* \smalleq j), \forall \bm{x} \in \bm{X}_{\tilde{y} = i}$. Note the change from predicted probability, $\hat{p}$, to an exact probability, $p$. Thus by the ideal condition, the inequality in (\ref{alg:case1_condition}) can be written as $ p (\tilde{y} = j \vert y^* \smalleq j) \ge t_j$, which we prove below:

\begin{align*}
    p (\tilde{y} = j \vert y^* \smalleq j) &\geq p (\tilde{y} = j \vert y^* \smalleq j) \cdot 1 && \triangleright \text{Identity} \\
    &\geq p(\tilde{y} = j \vert y^* \smalleq j) \cdot \sum_{i \smallin [m]} p(y^* \smalleq i \vert \tilde{y} \smalleq j)  \\
    &\geq \sum_{i \smallin [m]} p(\tilde{y} = j \vert y^* \smalleq j) \cdot p(y^* \smalleq i \vert \tilde{y} \smalleq j) && \triangleright \text{move product into sum}   \\
    &\geq \sum_{i \smallin [m]} p(\tilde{y} = j \vert y^* \smalleq i) \cdot p(y^* \smalleq i \vert \tilde{y} \smalleq j) && \triangleright \text{diagonal entry maximizes row}   \\
    &\geq t_j  && \triangleright \text{Lemma 1, ideal condition} 
\end{align*}

\paragraph{\emph{Proof (Claim 2) of Case 2}:}

We prove $\bm{X}_{ \tilde{y} \smalleq i,y^* \neq j} \nsubseteq \hat{\bm{X}}_{\tilde{y} \smalleq i, y^* \smalleq j}$ by contradiction. Assume there exists some example $\bm{x}_k \in \bm{X}_{ \tilde{y} \smalleq i,y^* \smalleq z}$ for $z \neq j$ such that $ \bm{x}_k \in \hat{\bm{X}}_{\tilde{y} \smalleq i, y^* \smalleq j}$. By claim 1, we have that $\bm{X}_{ \tilde{y} \smalleq i,y^* \smalleq j} \subseteq \hat{\bm{X}}_{\tilde{y} \smalleq i, y^* \smalleq j}$, therefore, $\bm{x}_k \in \hat{\bm{X}}_{ \tilde{y} \smalleq i,y^* \smalleq z}$.

Thus, for some example $\bm{x}_k$, we have that $\bm{x}_k \in \hat{\bm{X}}_{\tilde{y} \smalleq i, y^* \smalleq j}$ and also $\bm{x}_k \in \hat{\bm{X}}_{ \tilde{y} \smalleq i,y^* \smalleq z}$.

However, this is a collision and when a collision occurs, the confident joint will break the tie with $\argmax$. Because each diagonal entry of $\nm$ maximizes its row and column this will always be assign $\bm{x}_k \in \hat{\bm{X}}_{ \tilde{y} ,y^*}[{\tilde{y}}_k][y_k^*]$ (the assignment from Claim 1).

This theorem also states $\estjoint \approxeq \joint$. This directly follows directly from the fact that $\forall i \smallin [m], j \smallin [m], \hat{\bm{X}}_{ \tilde{y} \smalleq i,y^* \smalleq j} = \bm{X}_{\tilde{y} \smalleq i, y^* \smalleq j} $, i.e. the confident joint \emph{exactly counts} the partitions $\bm{X}_{\tilde{y} \smalleq i, y^* \smalleq j}$ for all pairs $(i, j) \in [m] \times M$, thus $\bm{C}_{\tilde{y}, y^*} = n \joint$ and $\estjoint \approxeq \joint$. Omitting discretization error, the confident joint $\cj$, when normalized to $\estjoint$, is an exact estimator for $\joint$. 
For example, if the noise rate is $0.39$, but the dataset has only 5 examples in that class, the best possible estimate by removing errors is $2/5 = 0.4  \approxeq 0.39$.

\end{proof}


\addtocounter{corollary}{-1}
\begin{corollary}[Exact Estimation] \label{cor:consistent_estimation}
For a noisy dataset, $(\bm{x}, \tilde{y})^n \in {(\mathbb{R}^d, [m])}^n$ and $\bm{\theta}{\scriptstyle:} \bm{x} \smallra \hat{p}(\tilde{y})$, 
if $\hat{p}(\tilde{y} ; \bm{x}, \bm{\theta})$ is ideal and each diagonal entry of $\nm$ maximizes its row and column, and if $\hat{\bm{X}}_{ \tilde{y} \smalleq i,y^* \smalleq j} = \bm{X}_{\tilde{y} \smalleq i, y^* \smalleq j}$, then $\estjoint \approxeq \joint$.
\end{corollary}

\begin{proof} \label{proof:consistent_estimation}
The result follows directly from  Theorem \ref{thm:exact_label_errors}.
Because the confident joint \emph{exactly counts} the partitions $\bm{X}_{\tilde{y} \smalleq i, y^* \smalleq j}$ for all pairs $(i, j) \in [m] \times M$ by Theorem \ref{thm:exact_label_errors}, $\bm{C}_{\tilde{y}, y^*} = n \joint$, omitting discretization rounding errors.
\end{proof}

In the main text, Theorem \ref{thm:exact_label_errors} includes Corollary \ref{cor:consistent_estimation} for brevity. We have separated out Corollary \ref{cor:consistent_estimation} here to make apparent that the primary contribution of Theorem \ref{thm:exact_label_errors} is to prove $\hat{\bm{X}}_{ \tilde{y} \smalleq i,y^* \smalleq j} = \bm{X}_{\tilde{y} \smalleq i, y^* \smalleq j}$, from which the result of Corollary \ref{cor:consistent_estimation}, namely that $\estjoint \approxeq \joint$ naturally follows, omitting discretization rounding errors.

\begin{corollary}[Per-Class Robustness] 
For a noisy dataset, $\bm{X} \coloneqq (\bm{x}, \tilde{y})^n \smallin {(\mathbb{R}^d, [m])}^n$ and model $\bm{\theta}{\scriptstyle:} \bm{x} \smallra \hat{p}(\tilde{y})$, 
if $\predprobshortj$ is \textbf{per-class diffracted} without label collisions and each diagonal entry of $\nm$ maximizes its row, then $\hat{\bm{X}}_{ \tilde{y} \smalleq i,y^* \smalleq j} = \bm{X}_{\tilde{y} \smalleq i, y^* \smalleq j}$ and $\estjoint \approxeq \joint$.
\end{corollary}

\begin{proof} \label{proof:per_class_robustness}
Re-stating the meaning of \textbf{per-class diffracted}, we wish to show that if $\hat{p}(\tilde{y} ; \bm{x}, \bm{\theta})$ is diffracted with class-conditional noise s.t. $\forall j \smallin [m], \hat{p}(\tilde{y} = j ; \bm{x}, \bm{\theta}) = \epsilon_j^{(1)} \cdot p^*(\tilde{y} = j \vert y^* \smalleq y_k^*) + \epsilon_j^{(2)}$ where $\epsilon_j^{(1)} \in \mathcal{R}, \epsilon_j^{(2)} \in \mathcal{R}$ (for any distribution) without label collisions and each diagonal entry of $\nm$ maximizes its row, then $\hat{\bm{X}}_{ \tilde{y} \smalleq i,y^* \smalleq j} = \bm{X}_{\tilde{y} \smalleq i, y^* \smalleq j}$ and $\estjoint \approxeq \joint$.

First note that combining linear combinations of real-valued $\epsilon_j^{(1)}$ and  $\epsilon_j^{(2)}$ with the probabilities of class $j$ for each example may result in some examples having $\predprobshortj = \epsilon_j^{(1)} \perfprobshortj + \epsilon_j^{(2)} > 1$ or  $\predprobshortj = \epsilon_j^{(1)} \perfprobshortj + \epsilon_j^{(2)} < 0$. The proof makes no assumption about the validity of the model outputs and therefore holds when this occurs. Furthermore, confident learning does not require valid probabilities when finding label errors because confident learning depends on the \emph{rank} principle, i.e., the rankings of the probabilities, not the values of the probabilities.

When there are no label collisions, the bins created by the confident joint are:
\begin{equation} \label{eqn_robustness_cj_in_proof}
    \hat{\bm{X}}_{ \tilde{y} \smalleq i,y^* \smalleq j} \coloneqq  \{\bm{x} \in \bm{X}_{\tilde{y} = i} :  \hat{p}(\tilde{y}=j; \bm{x}, \bm{\theta})  \ge t_j \}
\end{equation}

where 
\begin{equation*}
    t_j = \mathop{\mathbb{E}}_{\bm{x} \in \bm{X}_{\tilde{y}=j}} \predprobshortj
\end{equation*}

WLOG: we re-formulate the error $\epsilon_j^{(1)} \perfprobshortj + \epsilon_j^{(2)} \; \text{as} \; \epsilon_j^{(1)} (\perfprobshortj + \epsilon_j^{(2)})$. 

Now, for diffracted (non-ideal) probabilities, we rearrange how the threshold $t_j$ changes for a given $\epsilon_j^{(1)}, \epsilon_j^{(2)}$:
\begin{align*}
    t_j^{\epsilon_j} &= \mathop{\mathbb{E}}_{\bm{x} \in \bm{X}_{\tilde{y}=j}} \epsilon_j^{(1)} (\perfprobshortj + \epsilon_j^{(2)}) \\
    t_j^{\epsilon_j} &= \epsilon_j^{(1)} \left(\mathop{\mathbb{E}}_{\bm{x} \in \bm{X}_{\tilde{y}=j}}  \perfprobshortj + \mathop{\mathbb{E}}_{\bm{x} \in \bm{X}_{\tilde{y}=j}} \epsilon_j^{(2)} \right)\\
    t_j^{\epsilon_j} &= \epsilon_j^{(1)} \left(t^*_j + \epsilon_j^{(2)} \cdot \mathop{\mathbb{E}}_{\bm{x} \in \bm{X}_{\tilde{y}=j}} 1 \right)\\
    t_j^{\epsilon_j} &= \epsilon_j^{(1)} (t^*_j + \epsilon_j^{(2)})
\end{align*}

Thus, for per-class diffracted (non-ideal) probabilities, Eqn. (\ref{eqn_robustness_cj_in_proof}) becomes
\begin{align*} 
    {\hat{\bm{X}}_{ \tilde{y} \smalleq i,y^* \smalleq j}}^{\epsilon_j} &=  \{\bm{x} \in \bm{X}_{\tilde{y} = i} :  \epsilon_j^{(1)} (\perfprobshortj + \epsilon_j^{(2)}) \ge \epsilon_j^{(1)} (t^*_j + \epsilon_j^{(2)})\} \\
     &= \{\bm{x} \in \bm{X}_{\tilde{y} = i} :  \perfprobshortj  \ge t^*_j \}  \\
     &= \bm{X}_{ \tilde{y} \smalleq i,y^* \smalleq j} && \triangleright \text{by Theorem (\ref{thm:exact_label_errors})}
\end{align*}

In the second to last step, we see that the formulation of the label errors is the formulation of $\cj$ for $\emph{ideal}$ probabilities, which we proved yields exact label errors and consistent estimation of $\joint$ in Theorem \ref{thm:exact_label_errors}, which concludes the proof. Note that we eliminate the need for the assumption that each diagonal entry of $\nm$ maximizes its column because this assumption is only used in the proof of Theorem \ref{thm:exact_label_errors} when collisions occur, but here we only consider the case when there are no collisions. 

\end{proof}

\begin{theorem}[Per-Example Robustness] 
For a noisy dataset, $\bm{X} \coloneqq (\bm{x}, \tilde{y})^n \in {(\mathbb{R}^d, [m])}^n$ and model $\bm{\theta}{\scriptstyle:} \bm{x} \smallra \hat{p}(\tilde{y})$, 
if $\predprobshortj$ is \textbf{per-example diffracted} without label collisions and each diagonal entry of $\nm$ maximizes its row, then $\hat{\bm{X}}_{ \tilde{y} \smalleq i,y^* \smalleq j} \approxeq  \bm{X}_{\tilde{y} \smalleq i, y^* \smalleq j}$ and $\estjoint \approxeq \joint$.
\end{theorem}

\begin{proof} \label{proof:robustness}

We consider the nontrivial real-world setting when a learning model $\bm{\theta}{\scriptstyle:} \bm{x} \smallra \hat{p}(\tilde{y})$ outputs erroneous, non-ideal predicted probabilities with an error term added for every example, across every class, such that $ \forall \bm{x} \in \bm{X}, \forall j \in [m], \; \predprobshortj = \perfprobshortj + \errorxj$. As a notation reminder $\perfprobshortj$ is shorthand for the ideal probabilities $p^*(\tilde{y} = j \vert y^* = y_k^*) + \errorxj$ and $\predprobshortj$ is shorthand for the predicted probabilities $\hat{p}(\tilde{y}=j ; \bm{x}, \bm{\theta})$.

The predicted probability error $\errorxj$ is distributed uniformly with no other constraints. We use $\epsilon_j \in \mathcal{R}$ to represent the mean of $\errorxj$ per class, i.e. $\epsilon_j = \mathop{\mathbb{E}}_{\bm{x} \in \bm{X}} \errorxj$, which can be seen by looking at the form of the uniform distribution in Eqn. (\ref{eqn:piecewise_error}). If we wanted, we could add the constraint that $\epsilon_j = 0, \forall j \in [m]$ which would simplify the theorem and the proof, but is not as general and we prove exact label error and joint estimation without this constraint.

We re-iterate the form of the error in Eqn. (\ref{eqn:piecewise_error}) here ($\mathcal{U}$ denotes a uniform distribution):

\begin{equation*} 
    \errorxj \sim \begin{cases} 
      \mathcal{}{U}(\epsilon_j + t_j - \perfprobshortj \, , \, \epsilon_j - t_j + \perfprobshortj] & \perfprobshortj \geq t_j \\
      \mathcal{U}[\epsilon_j - t_j + \perfprobshortj \, , \, \epsilon_j + t_j - \perfprobshortj) & \perfprobshortj < t_j
   \end{cases}
\end{equation*}

When there are no label collisions, the bins created by the confident joint are:
\begin{equation} \label{eqn_robustness_cj_in_general_proof}
    \hat{\bm{X}}_{ \tilde{y} \smalleq i,y^* \smalleq j} \coloneqq  \{\bm{x} \in \bm{X}_{\tilde{y} = i} :  \predprobshortj  \ge t_j \}
\end{equation}
where 
\begin{equation*}
    t_j = \frac{1}{|\bm{}{X}_{\tilde{y}=j}|} \sum_{\bm{x} \in \bm{X}_{\tilde{y}=j}} \predprobshortj
\end{equation*}

Rewriting the threshold $t_j$ to include the error terms $\errorxj$ and $\epsilon_j$, we have
\begin{align*}
    t_j^{\epsilon_j} &= \frac{1}{|\bm{}{X}_{\tilde{y}=j}|} \sum_{\bm{x} \in \bm{X}_{\tilde{y}=j}}  \perfprobshortj + \errorxj \\
    t_j^{\epsilon_j} &= \mathop{\mathbb{E}}_{\bm{x} \in \bm{X}_{\tilde{y}=j}}   \perfprobshortj + \mathop{\mathbb{E}}_{\bm{x} \in \bm{X}_{\tilde{y}=j}} \errorxj\\
     &= t_j + \epsilon_j
\end{align*}
where the last step uses the fact that $\errorxj$ is uniformly distributed over $\bm{x} \in \bm{X}$ and $n \rightarrow \infty$ so that $\mathop{\mathbb{E}}_{\bm{x} \in \bm{X}_{\tilde{y}=j}} \errorxj = \mathop{\mathbb{E}}_{\bm{x} \in \bm{X}} \errorxj = \epsilon_j$. We now complete the proof by showing that
$$ \perfprobshortj + \errorxj \ge t_j + \epsilon_j \iff  \perfprobshortj  \ge t_j$$
If this statement is true then the subsets created by the confident joint in Eqn. \ref{eqn_robustness_cj_in_general_proof} are unaltered and therefore ${\hat{\bm{X}}_{ \tilde{y} \smalleq i,y^* \smalleq j}}^{\errorxj} =  {\hat{\bm{X}}_{ \tilde{y} \smalleq i,y^* \smalleq j}} \overset{Thm.\  \ref{thm:exact_label_errors}}{=} {\bm{X}_{ \tilde{y} \smalleq i,y^* \smalleq j}}$, where ${\hat{\bm{X}}_{ \tilde{y} \smalleq i,y^* \smalleq j}}^{\errorxj}$ denotes the confident joint subsets for $\errorxj$ predicted probabilities. 

Now we complete the proof. From the distribution for $\errorxj$ (Eqn. \ref{eqn:piecewise_error}) , we have that
\begin{align*}
     \perfprobshortj < t_j &\implies \errorxj < \epsilon_j + t_j -  \perfprobshortj \\
     \perfprobshortj \ge t_j &\implies \errorxj \ge \epsilon_j + t_j -  \perfprobshortj
\end{align*}
Re-arranging
\begin{align*}
     \perfprobshortj < t_j &\implies  \perfprobshortj + \errorxj < t_j + \epsilon_j \\
     \perfprobshortj \ge t_j &\implies  \perfprobshortj + \errorxj \ge t_j + \epsilon_j
\end{align*}
Using the contrapositive, we have
\begin{align*}
     \perfprobshortj + \errorxj \ge t_j + \epsilon_j &\implies  \perfprobshortj  \ge t_j \\
     \perfprobshortj  \ge t_j &\implies  \perfprobshortj + \errorxj \ge t_j + \epsilon_j
\end{align*}
Combining, we have
\begin{equation*}
    \perfprobshortj + \errorxj \ge t_j + \epsilon_j \iff  \perfprobshortj  \ge t_j
\end{equation*}
Therefore,
\begin{equation*}
    {\hat{\bm{X}}_{ \tilde{y} \smalleq i,y^* \smalleq j}}^{\errorxj} \overset{Thm.\ \ref{thm:exact_label_errors}}{=} {\bm{X}_{ \tilde{y} \smalleq i,y^* \smalleq j}}
\end{equation*}

The last line follows from the fact that we have reduced ${\hat{\bm{X}}_{ \tilde{y} \smalleq i,y^* \smalleq j}}^{\errorxj}$ to counting the same condition ($\perfprobshortj  \ge t_j$) as the confident joint counts under ideal probabilities in Thm (\ref{thm:exact_label_errors}). Thus, we maintain exact finding of label errors and exact estimation (Corollary \ref{cor:per_class_robustness}) holds under no label collisions. The proof applies for finite datasets because we ignore discretization error, however, for equality, the proof requires the assumption $n \rightarrow \infty$ which is used in this step: $\mathop{\mathbb{E}}_{\bm{x} \in \bm{X}_{\tilde{y}=j}} \errorxj \overset{n \rightarrow \infty}{=} \mathop{\mathbb{E}}_{\bm{x} \in \bm{X}} \errorxj = \epsilon_j$. Thus, we use approximately equals in the statement of the theorem. 


Note that while we use a uniform distribution in Eqn. \ref{eqn:piecewise_error}, any bounded symmetric distribution with mode $\epsilon_j = \mathop{\mathbb{E}}_{\bm{x} \in \bm{X}} \epsilon_{\bm{x}, j}$ is sufficient. Observe that the bounds of the distribution are non-vacuous (they do not collapse to a single value $e_j$) because $t_j \neq \perfprobshortj$ by Lemma \ref{lemma:ideal_threshold}.

\end{proof}

        \begin{algorithm}[H] 
        \caption{ (\textbf{Confident Joint}) for class-conditional label noise characterization. }
        \label{alg:cj}
        \begin{algorithmic} 
        \Statex \textbf{input} $\bm{\hat{P}}$ an $n \times m$ matrix of out-of-sample predicted probabilities $\bm{\hat{P}}[i][j] \coloneqq \hat{p}(\tilde{y} = j; x, \bm{\theta})$
        \Statex \textbf{input} $\bm{\tilde{y}} \in {\mathbb{N}_{\scriptscriptstyle \geq 0}}^n$, an $n \times 1$ array of noisy labels
        \Statex \textbf{procedure} {\sc ConfidentJoint}($\bm{\hat{P}}$, $\bm{\tilde{y}}$):
        
        \Statex {\sc \textbf{PART 1} (Compute thresholds)}
        \For {$j \gets 1,m$}  
        \For {$i \gets 1,n$}
        \State $l \gets$ new empty list []
        \If {$\bm{\tilde{y}}[i] = j$}
        \State append $\bm{\hat{P}}[i][j]$ to $l$
        \EndIf
        \EndFor
        \State $\bm{t}[j] \gets \text{average}(l)$     \Comment{May use percentile instead of average for more confidence}
        \EndFor
        
        
        \Statex {\sc \textbf{PART 2} (Compute confident joint)}
        \State $\bm{C} \gets$ $m \times m$ matrix of zeros
        \For {$i \gets 1,n$} 
        \State $cnt \gets 0$
        \For {$j \gets 1,m$}
        \If {$\bm{\hat{P}}[i][j] \geq \bm{t}[j] $}
        \State $cnt \gets cnt + 1$
        \State $y^* \gets j$   \Comment{guess of true label}
        \EndIf
        \EndFor
        \State $\tilde{y} \gets \bm{\tilde{y}}[i]$ 
        \If {$cnt > 1$}  \Comment{if label collision}
        \State $ y^* \gets \argmax{\bm{\hat{P}}[i] }$ 
        \EndIf
        \If {$cnt > 0$}
        \State $\bm{C}[\tilde{y}][y^*] \gets \bm{C}[\tilde{y}][y^*] + 1$
        \EndIf
        \Statex \textbf{output} $\bm{C}$, the {$\scriptstyle m \times m$} unnormalized counts matrix
        \EndFor
        \end{algorithmic}
        \end{algorithm}

\section{The confident joint and joint algorithms}

The confident joint is expressed succinctly in equation Eqn. \ref{eqn_paper_confident_joint} with the thresholds expressed in Eqn. \ref{eqn_paper_threshold}. For clarity, we provide these equations in algorithm form (See Alg. \ref{alg:cj} and Alg. \ref{alg:calibration}).

The confident joint algorithm (Alg. \ref{alg:cj}) is an $\mathcal{O}(m^2 + nm)$ step procedure to compute $\cj$. The algorithm takes two inputs: (1) $\bm{\hat{P}}$ an $n \times m$ matrix of out-of-sample predicted probabilities $\bm{\hat{P}}[i][j] \coloneqq \hat{p}(\tilde{y} = j; x_i, \bm{\theta})$ and (2) the associated array of noisy labels. We typically use cross-validation to compute $\bm{\hat{P}}$ for train sets and a model trained on the train set and fine-tuned with cross-validation on the test set to compute $\bm{\hat{P}}$ for a test set. Any method works as long  $\hat{p}(\tilde{y}=j; \bm{x}, \bm{\theta})$ are out-of-sample, holdout predicted probabilities.

\textbf{Computation time.}  Finding label errors in ImageNet takes 3 minutes on an i7 CPU. Results in all tables reproducible via open-sourced \texttt{cleanlab} package.

Note that Alg. \ref{alg:cj} embodies Eqn. \ref{eqn_paper_confident_joint}, and Alg. \ref{alg:calibration} realizes  Eqn. \ref{eqn_calibration}.

\begin{algorithm}[!htb]
\caption{ ( \textbf{Joint} ) calibrates the confident joint to estimate the latent, true distribution of class-conditional label noise}
\label{alg:calibration}
\begin{algorithmic} 
\Statex \textbf{input} $\cj[i][j]$, {$\scriptstyle m \times m$} unnormalized counts
\Statex \textbf{input} $\bm{\tilde{y}}$ an $n \times 1$ array of noisy integer labels
\Statex \textbf{procedure} {\sc JointEstimation}($\bm{C}$, $\bm{\tilde{y}}$):
\State $\tilde{\bm{C}}_{\tilde{y}=i, y^*=j} \gets \frac{\bm{C}_{\tilde{y}=i, y^*=j}}{\sum_{j \in [m]} \bm{C}_{\tilde{y}=i, y^*=j}} \cdot \lvert \bm{X}_{\tilde{y}=i} \rvert$ \Comment{calibrate marginals}
\State $\estjointlong \gets \frac{\tilde{\bm{C}}_{\tilde{y}=i, y^*=j}}{\sum\limits_{i \in [m], j \in [m]} \tilde{\bm{C}}_{\tilde{y}=i, y^*=j}}$ \Comment{joint sums to 1}
\Statex \textbf{output} $\estjoint$ joint dist. matrix $\sim p(\tilde{y}, y^*)$
\end{algorithmic}
\end{algorithm}

\section{Extended Comparison of Confident Learning Methods on CIFAR-10}

Fig. \ref{cifar10_abs_diff_ALL} shows the absolute difference of the true joint $\joint$ and the joint distribution estimated using confident learning $\estjoint$ on CIFAR-10, for 20\%, 40\%, and 70\% label noise, 20\%, 40\%, and 60\% sparsity, for all pairs of classes in the joint distribution of label noise. Observe that in moderate noise regimes between 20\% and 40\% noise, confident learning accurately estimates nearly every entry in the joint distribution of label noise. This figure serves to provide evidence for how confident learning identifies the label errors with high accuracy as shown in Table \ref{table:cifar10_benchmark} as well as support our theoretical contribution that confident learning exactly estimates the joint distribution of labels under reasonable assumptions (c.f., Thm. \ref{thm:robustness}).

Because we did not remove label errors from the validation set, when training on the data cleaned by CL in the train set, we may have induced a distributional shift, making the moderate increase accuracy a more satisfying result.

\begin{figure*}[!htb]
\centerline{\includegraphics[width=.99\textwidth]{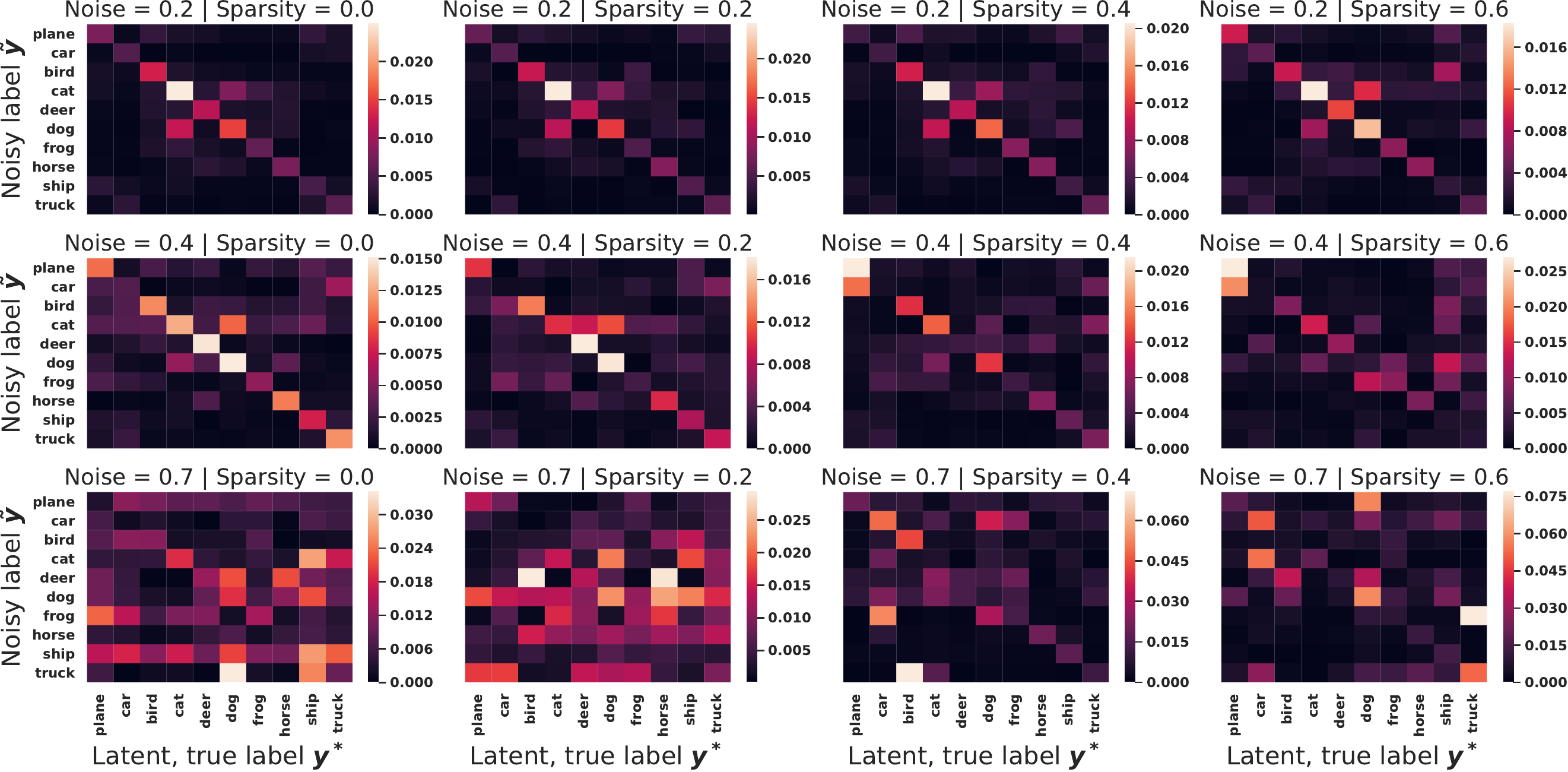}}
\caption{Absolute difference of the true joint $\joint$ and the joint distribution estimated using confident learning $\estjoint$ on CIFAR-10, for 20\%, 40\%, and 70\% label noise, 20\%, 40\%, and 60\% sparsity, for all pairs of classes in the joint distribution of label noise.} 
\label{cifar10_abs_diff_ALL}
\end{figure*}

\begin{table*}[t]

\setlength\tabcolsep{3pt} 
\renewcommand{\arraystretch}{0.9}
\caption{RMSE error of $\joint$ estimation on CIFAR-10 using $\cj$ to estimate $\estjoint$ compared with using the baseline approach $\bm{C}_{\text{confusion}}$ to estimate $\estjoint$.}
\vskip -0.1in
\label{table:cifar10_rmse}
\begin{center}
\begin{small}
\resizebox{\textwidth}{!}{ 

\begin{tabular}{l|cccc|cccc|cccc}
\toprule
Noise &      \multicolumn{4}{c}{0.2} & \multicolumn{4}{c}{0.4} & \multicolumn{4}{c}{0.7} \\
Sparsity &      0 &      0.2 &      0.4 &      0.6 &      0 &      0.2 &      0.4 &      0.6 &      0 &      0.2 &      0.4 &      0.6 \\
\midrule

$\|\estjoint$ - $\joint\|_2$        & \textbf{0.004} &  \textbf{0.004} &  \textbf{0.004} &  \textbf{0.004} &  \textbf{0.004} &  \textbf{0.004} &  \textbf{0.004} &  \textbf{0.005} &  0.011 &  \textbf{0.010} &  0.015 &  \textbf{0.017} \\
$\|\hat{\bm{Q}}_{confusion}$ - $\joint\|_2$  &  0.006 &  0.006 &  0.005 &  0.005 &  0.005 &  0.005 &  0.005 &  0.007 &  0.011 &  0.011 &  0.015 &  0.019 \\

\bottomrule
\end{tabular}
}

\end{small}
\end{center}
\vskip -0.2in
\end{table*}

In Table \ref{table:cifar10_rmse}, we estimate the $\joint$ using the confusion-matrix $\bm{C}_{\text{confusion}}$ approach normalized via Eqn. (\ref{eqn_calibration}) and compare this $\estjoint$, estimated by normalizing the CL approach with the confident joint $\cj$, for various amounts of noise and sparsity in $\joint$. Table \ref{table:cifar10_rmse} shows improvement using $\cj$ over $\bm{C}_{\text{confusion}}$, low RMSE scores, and robustness to sparsity in moderate-noise regimes.

\subsection{Benchmarking INCV}

We benchmarked INCV using the official Github code\footnote{\hyperlink{https://github.com/chenpf1025/noisy_label_understanding_utilizing}{https://github.com/chenpf1025/noisy\_label\_understanding\_utilizing}} on a machine with 128 GB of RAM and 4 RTX 2080 ti GPUs. Due to memory leak issues (as of the February 2020 open-source release, tested on a MacOS laptop with 16GB RAM and Ubuntu 18.04 LTS Linux server 128GB RAM) in the implementation, training frequently stopped due to out-of-memory errors. For fair comparison, we restarted INCV training until all models completed at least 90 training epochs. For each experiment,  Table \ref{table:incv_stuff} shows the total time required for training, epochs completed, and the associated accuracies. As shown in the table, the training time for INCV may take over 20 hours because the approach requires iterative retraining. For comparison, CL takes less than three hours on the same machine: an hour for cross-validation, less than a minute to find errors, an hour to retrain.

\begin{table*}[ht]
\vskip -0.1in
\setlength\tabcolsep{4pt} 
\renewcommand{\arraystretch}{0.9}
\caption{Information about INCV benchmarks including accuracy, time, and epochs trained for various noise and sparsity settings.}
\vskip -0.1in
\label{table:incv_stuff}
\begin{center}
\resizebox{\textwidth}{!}{ 

\begin{tabular}{l|cccc|cccc|cccc}
\toprule
Noise &      \multicolumn{4}{c}{0.2} & \multicolumn{4}{c}{0.4} & \multicolumn{4}{c}{0.7} \\
Sparsity &      0 &      0.2 &      0.4 &      0.6 &      0 &      0.2 &      0.4 &      0.6 &      0 &      0.2 &      0.4 &      0.6 \\
\midrule
\textbf{Accuracy}    &    0.878 &    0.886 &    0.896 &    0.892 &    0.844 &    0.766 &     0.854 &    0.736 &    0.283 &     0.253 &     0.348 &     0.297 \\
\textbf{Time (hours)}   &  9.120 &  11.350 &   10.420 &    7.220 &   7.580 &   11.720 &   20.420 &    6.180 &  16.230 &  17.250 &   16.880 &   18.300 \\
\textbf{Epochs trained} &   91 &   91 &  200 &  157 &   91 &  200 &   200 &  139 &   92 &    92 &   118 &   200 \\
\bottomrule
\end{tabular}\
}

\end{center}
\vskip -0.2in
\end{table*}

\section{Additional Figures}

In this section, we include additional figures that support the main manuscript. Fig. \ref{fig:imagenet_learning_with_noisy_labels} explores the benchmark accuracy of the individual confident learning approaches to support Fig. \ref{fig_imagenet_training} and Fig. \ref{fig_imagenet_resnet18_training} in the main text. The noise matrices shown in Fig. \ref{cifar10_ground_truth_noise_matrices} were used to generate the synthetic noisy labels for the results in Tables \ref{table:cifar10_label_error_measures} and \ref{table:cifar10_benchmark}.

Fig. \ref{fig:imagenet_learning_with_noisy_labels} shows the top-1 accuracy on the ILSVRC validation set when removing label errors estimated by CL methods versus removing random examples. For each CL method, we plot the accuracy of training with 20\%, 40\%,..., 100\% of the estimated label errors removed, omitting points beyond 200k.

\begin{figure*}[ht]
    \centering
    \begin{subfigure}[b]{0.48\textwidth}  
        \centering 
        \includegraphics[width=1\textwidth]{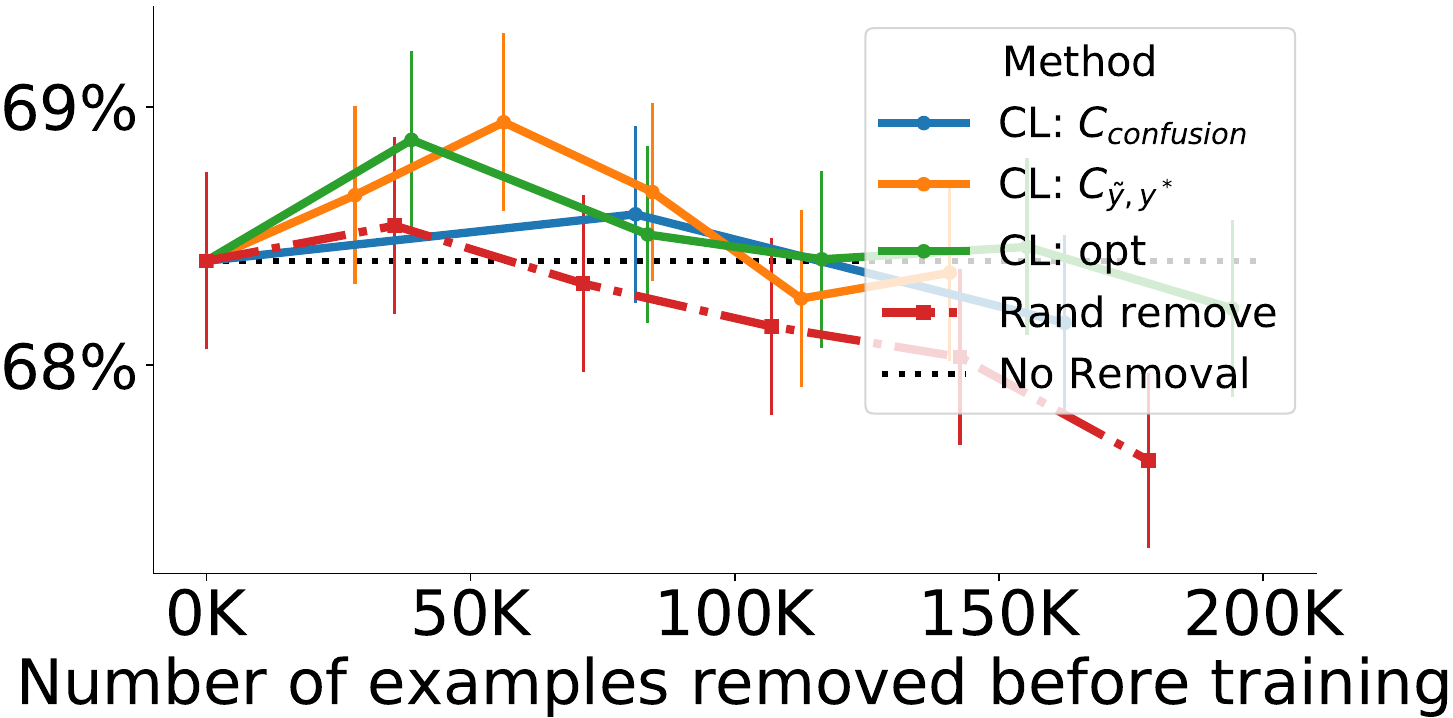}
        \caption[]%
        {ResNet18 Validation Accuracy}   
        \label{subfig:resnet18_imagenet_learning_with_noisy_labels}
    \end{subfigure}
    \begin{subfigure}[b]{0.48\textwidth}
        \centering
        \includegraphics[width=1\textwidth]{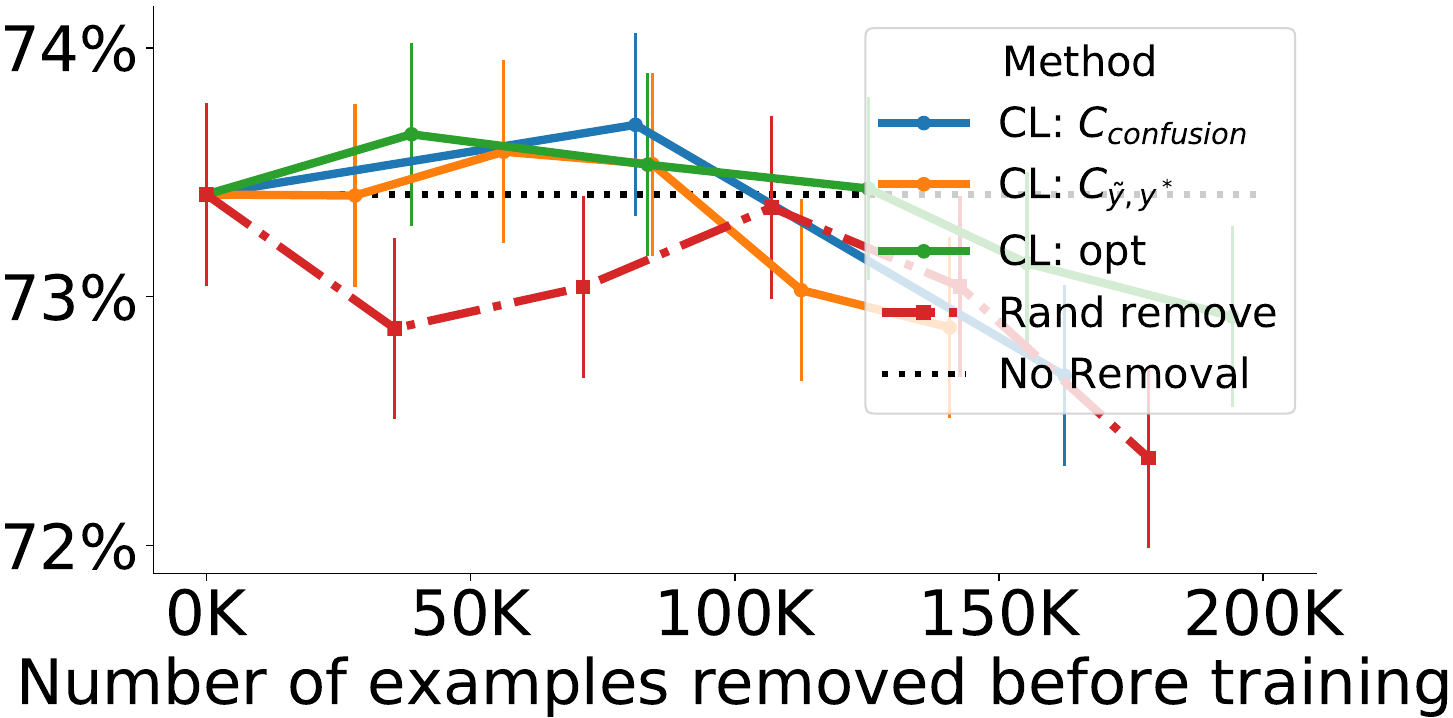}
        \caption[Network2]%
        {ResNet50 Validation Accuracy}    
        \label{subfig:resnet50_imagenet_learning_with_noisy_labels}
    \end{subfigure}
    \caption{Increased ResNet validation accuracy using CL methods on ImageNet with original labels (no synthetic noise added). Each point on the line for each method, from left to right, depicts the accuracy of training with 20\%, 40\%..., 100\% of estimated label errors removed. Error bars are estimated with Clopper-Pearson 95\% confidence intervals. The red dash-dotted baseline captures when examples are removed uniformly randomly. The black dotted line depicts accuracy when training with all examples.} 
    
    \label{fig:imagenet_learning_with_noisy_labels}
\end{figure*}

\begin{figure*}[ht]
\centerline{\includegraphics[width=\textwidth]{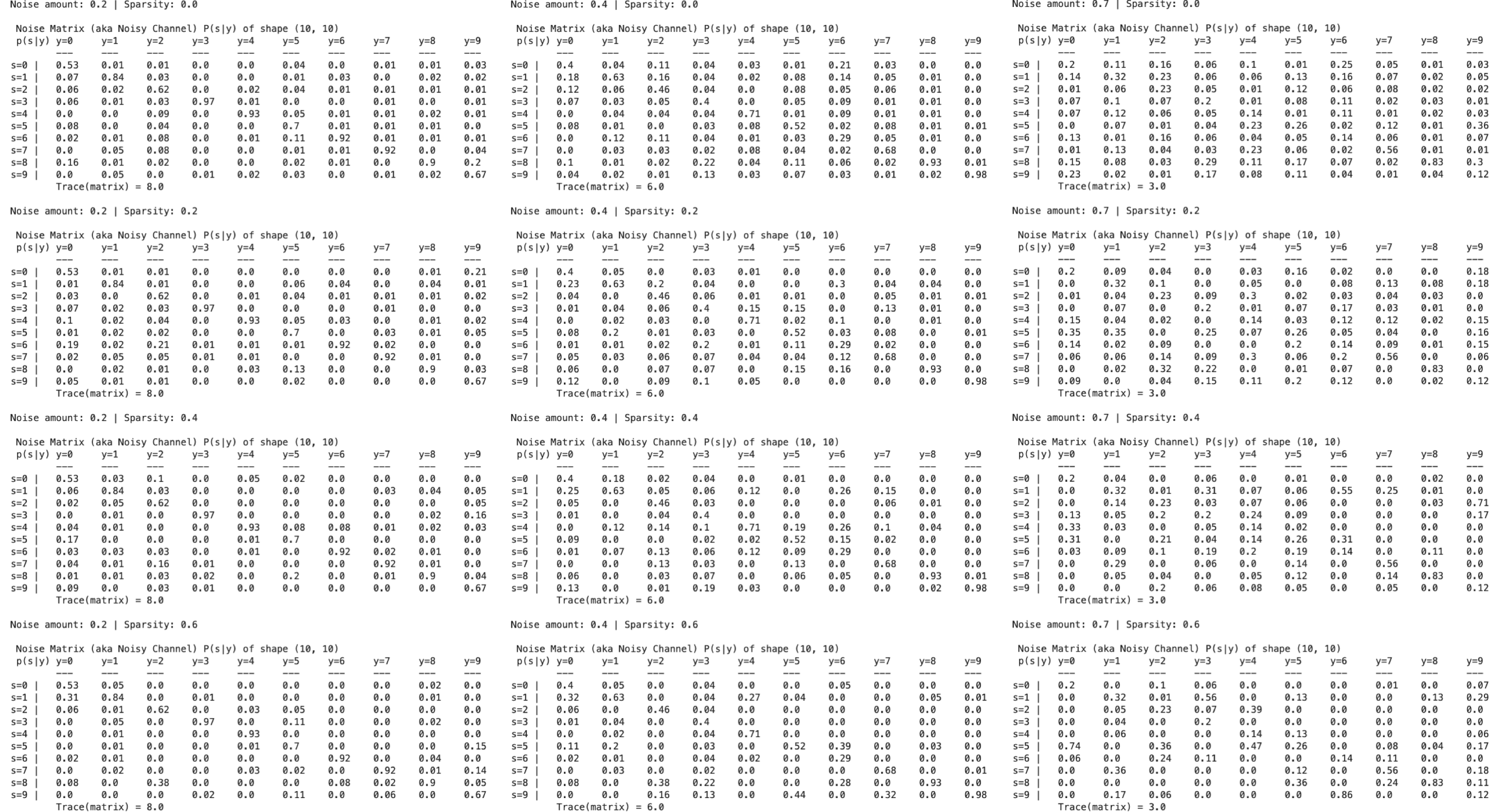}}
\caption{The CIFAR-10 noise transition matrices used to create the synthetic label errors. In the \href{https://github.com/cgnorthcutt/cleanlab}{\texttt{cleanlab}} code base, $s$ is used in place of $\tilde{y}$ to notate the noisy unobserved labels and $y$ is used in place of $y^*$ to notate the latent uncorrupted labels.}
\label{cifar10_ground_truth_noise_matrices}
\end{figure*}

\end{document}